\documentclass[acmsmall,screen]{acmart}

\setcopyright{rightsretained}
\acmJournal{PACMPL}
\acmYear{2020} \acmVolume{4} \acmNumber{OOPSLA} \acmArticle{134} \acmMonth{11} \acmPrice{}\acmDOI{10.1145/3428202}

\usepackage[utf8]{inputenc}
\usepackage[nomargin,inline,index,status=draft]{fixme}
\fxusetheme{colorsig}%
\FXRegisterAuthor{fxNY}{anfxNY}{\underline{\textbf{NY}}}%
\FXRegisterAuthor{fxRM}{anfxRM}{\underline{\textbf{RM}}}%
\FXRegisterAuthor{fxDZ}{anfxDZ}{\underline{\textbf{DZ}}}%

\usepackage{microtype}

\usepackage{xspace}
\usepackage{bbm}
\usepackage{color}
\usepackage[inline,shortlabels]{enumitem}%
\usepackage{yfonts}%
\usepackage{stmaryrd}
\usepackage{mathpartir}
\usepackage{nicefrac}%
\usepackage{thmtools}%
\usepackage{prooftree}
\usepackage{mathtools}
\usepackage{booktabs} 
\usepackage{xspace}
\usepackage{bold-extra}
\usepackage{euscript} [mathcal]

\usepackage{subcaption} %
\usepackage{multirow}                        
\usepackage{wrapfig}                        

\usepackage{xifthen}%
\usepackage[all]{xy}%

\usepackage{hyperref}%
\hypersetup{final}%

\usepackage{tikz}
\usetikzlibrary{patterns}

\usepackage{todonotes}

\newcommand\RMi[1]{\todo[color=white,inline]{\textcolor{red}{RM, #1}}}
\newcommand\NYi[1]{\todo[color=white,inline]{\textcolor{violet}{NY, #1}}}
\newcommand\DZi[1]{\todo[color=white,inline]{\textcolor{cyan}{DZ, #1}}}

\usepackage{etoolbox}
\newtoggle{full}

\newcommand{\ifempty}[3]{%
  \ifthenelse{\isempty{#1}}{#2}{#3}%
}%

\makeatletter
\newcommand\footnoteref[1]{\protected@xdef\@thefnmark{\ref{#1}}\@footnotemark}
\makeatother

\newcommand{\RedG}{\Longrightarrow}

\newcommand{\T}{T}

\renewcommand{\S}{S}

\newcommand{\rulename}[1]{\text{\textnormal{\small[\textsc{#1}]}}}

\newcommand{\PGCD}{\textsf{PGCD}\xspace}

\newcommand{\pp}{{\roleP}}
\newcommand{\pq}{{\roleQ}}
\newcommand{\pr}{{\roleR}}
\newcommand{\ps}{{\roleS}}
\newcommand{\q}{\pq}
\newcommand{\e}{\kf{e}}
\newcommand{\x}{x}

\newcommand{\val}{v} %
\newcommand{\valn}{\kf{n}}
\newcommand{\valr}{\kf{i}}

\newcommand{\SEP}{\ensuremath{~~\mathbf{|\!\!|}~~ }}
\newcommand{\kf}[1]{\ensuremath{\mathsf{#1}}}
\newcommand{\ma}{\ensuremath{\mathit{a}}}
\newcommand{\mb}{\ensuremath{\mathit{b}}}

\newcommand{\Gvti}[5]{\ensuremath{#1\to#2:\{#3_i({#4}_i). #5_i \}_{i \in I}}}

\newcommand{\GvtPre}[4]{\ensuremath{#1\to#2: #3({#4})}}

\newcommand{\ty}{\mathbf{t}}
\newcommand{\GvtPair}[3]{\ensuremath{#1\to#2: \{#3\}_{i\in I}}}

\newcommand{\N}{M}

\newcommand{\pa}[2]{#1 \triangleleft  #2}
\newcommand{\set}[1]{\{#1\}}

\newcommand{\eval}[2]{#1 \downarrow #2}
\newcommand{\evalState}[3]{#1 \downarrow_{#3} #2}
\newcommand{\true}{\kf{true}}
\newcommand{\false}{\kf{false}}

\newcommand{\der}[3]{ #1 \vdash   #2  :#3}

\newcommand{\participant}[1]{\mathtt{pt}\{#1\}}

\newcommand{\proj}[2]{ #1 \upharpoonright_{#2}}
\newcommand{\projt}[2]{{#1}{\upharpoonright}{#2}}
\newcommand{\sub}[2]{\set{#1/#2}}

\newcommand{\motion}[2]{\kf{#1}\ENCan{#2}}
\newcommand{\Gmotion}[3]{\kf{#1}\ENCan{#2}.{#3}}
\newcommand{\Lmotion}[3]{\kf{#1}\ENCan{#2}.{#3}}

\newcommand{\procin}[3]{#1 ?  #2 .#3}

\newcommand{\procout}[4]{#1 ! #2\ENCan{#3}. #4}

\newcommand{\PP}{\ensuremath{{P}}}
\newcommand{\PC}{\ensuremath{\Bbb P}}
\newcommand{\plays}{\ensuremath{\triangleleft{}}}
\newcommand{\proc}[3]{\ensuremath{\langle{#2,#3\mathsf{;}\, #1}\rangle}}
\newcommand{\geom}[1]{\ensuremath{\mathsf{geom}\left({#1}\right)}}
\newcommand{\parcomp}{\ensuremath{~|~}}

\newcommand{\Q}{\ensuremath{Q}}

\newcommand{\cond}[3]{\kf{if}~ #1 ~\kf{then} ~#2 ~\kf{else}~#3}

\newcommand{\tend}{\mathtt{end}}
\newcommand{\tunit}{\mathtt{unit}}
\newcommand{\tbool}{\mathtt{bool}}

\newcommand{\treal}{\mathtt{real}}

\newcommand{\tin}[3]{#1?#2(#3)}
\newcommand{\tout}[3]{#1!#2(#3)}

\renewcommand{\S}{S}

\newcommand{\lts}[1]{\xrightarrow{#1}}
\newcommand{\redG}[4]{#1\setminus#2\xrightarrow#3#4}
\newcommand{\redGM}[2]{#1\setminus#2}
\newcommand{\redGequiv}[2]{#1 \rightsquigarrow #2}
\newcommand{\redastGequiv}[2]{#1 \rightsquigarrow^\ast #2}

\newcommand{\subt}{\leqslant}

\newcommand{\subs}{\leq\vcentcolon}

\newcommand{\red}{\longrightarrow}

\newcommand{\cinferrule}[3][]{
  \mprset{fraction={===},
  fractionaboveskip=0.2ex,
  fractionbelowskip=0.4ex}
  \inferrule[#1]{#2}{#3}
}

\definecolor{ceca}{rgb}{1,0.5,0}

\newcommand{\Gvtir}[5]{\ensuremath{#1\to#2:\{#3_i({#4}_i). \redG{#5_i}\pp\ell\q \}_{i \in I}}}

\newcommand{\ENCan}[1]{\langle #1 \rangle}

\newcommand{\pre}{\ensuremath{\mathrm{Pre}\xspace}}
\newcommand{\post}{\ensuremath{\mathrm{Post}\xspace}}

\newcommand{\pred}{\mathcal{P}}

\newcommand{\dt}{\ensuremath{\mathsf{dt}}\xspace}
\newcommand{\clock}{\ensuremath{\mathsf{clock}}\xspace}

\definecolor{dtsessColor}{rgb}{0.8, 0.0, 0.0}%
\newcommand{\dtsessCol}[1]{{\color{dtsessColor}#1}}%
\newcommand{\dtsess}{\dtsessCol{\mathsf{dt}}}%

\def\set#1{\{#1 \}}
\def\tuple#1{\langle #1 \rangle }

\newcommand{\comp}[4]{#1 \lhd_{#2, #3} #4}      %
\newcommand{\sstore}[1]{\ensuremath{\vecFmt{#1}}}
\newcommand{\rbtstate}[4]{\ensuremath{\tuple{#2,#3,#4\mid\;#1}}} %
\newcommand{\judgement}[2]{\ensuremath{ #1 ~\vdash~(#2)}}
\newcommand{\pjudgement}[2]{\ensuremath{\vdash #1 ~\text{\textbf{psat}}~(#2)}}
\newcommand{\timestep}{{\ddagger}}
\newcommand{\instantaneous}{{\mathsf{\lightning}}}
\newcommand{\delayed}{{\mathsf{\rightsquigarrow}}}

\newcommand{\FP}{\ensuremath{\mathit{FP}}}

\newcommand{\myparagraph}[1]{\paragraph{\bf #1}}

\definecolor{vecColor}{rgb}{0.5, 0.0, 0.0}%
\newcommand{\conCol}[1]{{\color{blue}#1}}%
\newcommand{\vecFmt}[1]{\boldsymbol{\mathtt{#1}}}%

\newcommand{\pToc}{\mu^{-1}}
\newcommand{\cTop}{\mu}
\newcommand{\con}{{\boldsymbol{\conCol{\mathbbmtt{c}}}}}%

\newcommand{\freevars}{\ensuremath{\mathit{fv}}}
\newcommand{\refines}{\ensuremath{\preceq}}

\def\Cart{{\color{roleColor}{\mathsf{Cart}}}}
\def\Env{{\color{roleColor}{\mathsf{Env}}}}
\def\Crane{{\color{roleColor}{\mathsf{Crane}}}}
\def\Trolley{{\color{roleColor}{\mathsf{Trolley}}}}

\def\Producer{{\color{roleColor}{\mathsf{Prod}}}}
\def\Red{{\color{roleColor}{\mathsf{RRobot}}}}
\def\Green{{\color{roleColor}{\mathsf{GRobot}}}}

\def\START{\mathit{start}}
\def\GREEN{\mathit{green}}
\def\RED{\mathit{red}}
\def\OK{\mathit{ok}}
\def\FREE{\mathit{free}}

\def\ARRIVE{\mathit{arrive}}
\def\READY{\mathit{ready}}

\def\IDLEM{\mathsf{m\_idle}}
\def\MOVEM{\mathsf{m\_move}}
\def\IDLEMC{\mathsf{c\_idle}}

\def\WORKM{\mathsf{work}}
\def\PICKM{\mathsf{pick}}
\def\PLACEM{\mathsf{place}}

\definecolor{ruleColor}{rgb}{0.1, 0.3, 0.1}%
\definecolor{groundColor}{rgb}{0.38, 0.25, 0.32}%

\definecolor{roleColor}{rgb}{0.1, 0.3, 0.1}%
\newcommand{\roleCol}[1]{{\color{roleColor}#1}}%
\newcommand{\roleFmt}[1]{\boldsymbol{\roleCol{\mathtt{#1}}}}%
\newcommand{\roleP}[1][]{%
  \ifempty{#1}{{\color{roleColor}\roleFmt{p}}}{{\color{roleColor}\roleFmt{p}_{#1}}}%
}%
\newcommand{\pc}{{\roleP}}
\newcommand{\roleVar}[1][]{%
  \ifempty{#1}{{\color{roleColor}\roleFmt{x}}}{{\color{roleColor}\roleFmt{x}_{#1}}}%
}%
\newcommand{\procvar}{{\roleVar}}
\newcommand{\roleQ}[1][]{%
  \ifempty{#1}{{\color{roleColor}\roleFmt{q}}}{{\color{roleColor}\roleFmt{q}_{#1}}}%
}%
\newcommand{\roleR}[1][]{%
  \ifempty{#1}{{\color{roleColor}\roleFmt{r}}}{{\color{roleColor}\roleFmt{r}_{\!#1}}}%
}%
\newcommand{\roleS}[1][]{%
  \ifempty{#1}{{\color{roleColor}\roleFmt{s}}}{{\color{roleColor}\roleFmt{s}_{\!#1}}}%
}%

\newcommand{\GL}{\ensuremath{{\sf G}}}
\newcommand{\GLPre}{\ensuremath{g}}
\newcommand{\TLPre}{\ensuremath{T}}
\newcommand{\GLPAR}{\ensuremath{\ast}}
\newcommand{\SL}{\ensuremath{{\sf S}}}
\newcommand{\NUL}{\ensuremath{\epsilon}}

\newcommand{\Sec}{{Sec.}}
\newcommand\AT{\text{\small{\tt @}}}
\newcommand\ASENDPC{\vec{\pc}}

\usepackage{newunicodechar}
\newunicodechar{∃}{\ensuremath{\exists}}
\newunicodechar{∀}{\ensuremath{\forall}}
\newunicodechar{θ}{\ensuremath{\theta}}
\newunicodechar{τ}{\ensuremath{\tau}}
\newunicodechar{φ}{\ensuremath{\varphi}}
\newunicodechar{ξ}{\ensuremath{\xi}}
\newunicodechar{ζ}{\ensuremath{\zeta}}
\newunicodechar{ψ}{\ensuremath{\psi}}
\newunicodechar{π}{\ensuremath{\pi}}
\newunicodechar{α}{\ensuremath{\alpha}}
\newunicodechar{β}{\ensuremath{\beta}}
\newunicodechar{γ}{\ensuremath{\gamma}}
\newunicodechar{δ}{\ensuremath{\delta}}
\newunicodechar{ε}{\ensuremath{\epsilon}}
\newunicodechar{λ}{\ensuremath{\lambda}}
\newunicodechar{ρ}{\ensuremath{\rho}}
\newunicodechar{σ}{\ensuremath{\sigma}}
\newunicodechar{ω}{\ensuremath{\omega}}
\newunicodechar{Γ}{\ensuremath{\Gamma}}
\newunicodechar{Φ}{\ensuremath{\Phi}}
\newunicodechar{Δ}{\ensuremath{\Delta}}
\newunicodechar{Σ}{\ensuremath{\Sigma}}
\newunicodechar{Π}{\ensuremath{\Pi}}
\newunicodechar{∑}{\ensuremath{\Sigma}}
\newunicodechar{∏}{\ensuremath{\Pi}}
\newunicodechar{Θ}{\ensuremath{\Theta}}
\newunicodechar{Ω}{\ensuremath{\Omega}}
\newunicodechar{⇒}{\ensuremath{\Rightarrow}}
\newunicodechar{⇐}{\ensuremath{\Leftarrow}}
\newunicodechar{⇔}{\ensuremath{\Leftrightarrow}}
\newunicodechar{→}{\ensuremath{\rightarrow}}
\newunicodechar{←}{\ensuremath{\leftarrow}}
\newunicodechar{↔}{\ensuremath{\leftrightarrow}}
\newunicodechar{¬}{\ensuremath{\neg}}
\newunicodechar{∧}{\ensuremath{\land}}
\newunicodechar{∨}{\ensuremath{\lor}}
\newunicodechar{≠}{\ensuremath{\neq}}
\newunicodechar{≡}{\ensuremath{\equiv}}
\newunicodechar{∼}{\ensuremath{\sim}}
\newunicodechar{≈}{\ensuremath{\approx}}
\newunicodechar{≥}{\ensuremath{\geq}}
\newunicodechar{≤}{\ensuremath{\leq}}
\newunicodechar{≫}{\ensuremath{\gg}}
\newunicodechar{≪}{\ensuremath{\ll}}
\newunicodechar{∅}{\ensuremath{\emptyset}}
\newunicodechar{⊆}{\ensuremath{\subseteq}}
\newunicodechar{⊂}{\ensuremath{\subset}}
\newunicodechar{∩}{\ensuremath{\cap}}
\newunicodechar{⋂}{\ensuremath{\cap}}
\newunicodechar{∪}{\ensuremath{\cup}}
\newunicodechar{⋃}{\ensuremath{\cup}}
\newunicodechar{⊎}{\ensuremath{\uplus}}
\newunicodechar{∈}{\ensuremath{\in}}
\newunicodechar{∉}{\ensuremath{\not\in}}
\newunicodechar{⊤}{\ensuremath{\top}}
\newunicodechar{⊥}{\ensuremath{\bot}}
\newunicodechar{₀}{\ensuremath{_0}}
\newunicodechar{₁}{\ensuremath{_1}}
\newunicodechar{₂}{\ensuremath{_2}}
\newunicodechar{₃}{\ensuremath{_3}}
\newunicodechar{₄}{\ensuremath{_4}}
\newunicodechar{₅}{\ensuremath{_5}}
\newunicodechar{₆}{\ensuremath{_6}}
\newunicodechar{₇}{\ensuremath{_7}}
\newunicodechar{₈}{\ensuremath{_8}}
\newunicodechar{₉}{\ensuremath{_9}}
\newunicodechar{⁰}{\ensuremath{^0}}
\newunicodechar{¹}{\ensuremath{^1}}
\newunicodechar{²}{\ensuremath{^2}}
\newunicodechar{³}{\ensuremath{^3}}
\newunicodechar{⁴}{\ensuremath{^4}}
\newunicodechar{⁵}{\ensuremath{^5}}
\newunicodechar{⁶}{\ensuremath{^6}}
\newunicodechar{⁷}{\ensuremath{^7}}
\newunicodechar{⁸}{\ensuremath{^8}}
\newunicodechar{⁹}{\ensuremath{^9}}
\newunicodechar{𝔹}{\ensuremath{\mathbb{B}}}
\newunicodechar{ℝ}{\ensuremath{\mathbb{R}}}
\newunicodechar{ℕ}{\ensuremath{\mathbb{N}}}
\newunicodechar{ℂ}{\ensuremath{\mathbb{C}}}
\newunicodechar{ℚ}{\ensuremath{\mathbb{Q}}}
\newunicodechar{𝕋}{\ensuremath{\mathbb{T}}}
\newunicodechar{𝕏}{\ensuremath{\mathbb{X}}}
\newunicodechar{ℤ}{\ensuremath{\mathbb{Z}}}
\newunicodechar{✓}{\checkmark}
\newunicodechar{✗}{\ensuremath{\times}}
\newunicodechar{◊}{\ensuremath{\lozenge}}
\newunicodechar{□}{\ensuremath{\square}}
\newunicodechar{𝓐}{\ensuremath{\mathcal{A}}}
\newunicodechar{𝓑}{\ensuremath{\mathcal{B}}}
\newunicodechar{𝓒}{\ensuremath{\mathcal{C}}}
\newunicodechar{𝓓}{\ensuremath{\mathcal{D}}}
\newunicodechar{𝓔}{\ensuremath{\mathcal{E}}}
\newunicodechar{𝓕}{\ensuremath{\mathcal{F}}}
\newunicodechar{𝓖}{\ensuremath{\mathcal{G}}}
\newunicodechar{𝓗}{\ensuremath{\mathcal{H}}}
\newunicodechar{𝓘}{\ensuremath{\mathcal{I}}}
\newunicodechar{𝓙}{\ensuremath{\mathcal{J}}}
\newunicodechar{𝓚}{\ensuremath{\mathcal{K}}}
\newunicodechar{𝓛}{\ensuremath{\mathcal{L}}}
\newunicodechar{𝓜}{\ensuremath{\mathcal{M}}}
\newunicodechar{𝓝}{\ensuremath{\mathcal{N}}}
\newunicodechar{𝓞}{\ensuremath{\mathcal{O}}}
\newunicodechar{𝓟}{\ensuremath{\mathcal{P}}}
\newunicodechar{𝓠}{\ensuremath{\mathcal{Q}}}
\newunicodechar{𝓡}{\ensuremath{\mathcal{R}}}
\newunicodechar{𝓢}{\ensuremath{\mathcal{S}}}
\newunicodechar{𝓣}{\ensuremath{\mathcal{T}}}
\newunicodechar{𝓤}{\ensuremath{\mathcal{U}}}
\newunicodechar{𝓥}{\ensuremath{\mathcal{V}}}
\newunicodechar{𝓦}{\ensuremath{\mathcal{W}}}
\newunicodechar{𝓧}{\ensuremath{\mathcal{X}}}
\newunicodechar{𝓨}{\ensuremath{\mathcal{Y}}}
\newunicodechar{𝓩}{\ensuremath{\mathcal{Z}}}
\newunicodechar{…}{\ensuremath{\ldots}}
\newunicodechar{∗}{\ensuremath{\ast}}
\newunicodechar{⊢}{\ensuremath{\vdash}}
\newunicodechar{⊧}{\ensuremath{\models}}
\newunicodechar{′}{\ensuremath{'}}
\newunicodechar{″}{\ensuremath{''}}
\newunicodechar{‴}{\ensuremath{'''}}
\newunicodechar{∥}{\ensuremath{\|}}
\newunicodechar{⊕}{\ensuremath{\oplus}}
\newunicodechar{⁺}{\ensuremath{^+}}
\newunicodechar{⊇}{\ensuremath{\supseteq}}
\newunicodechar{∘}{\ensuremath{\circ}}
\newunicodechar{∙}{\ensuremath{\cdot}}
\newunicodechar{⋅}{\ensuremath{\cdot}}
\newunicodechar{≈}{\ensuremath{\approx}}
\newunicodechar{×}{\ensuremath{\times}}
\newunicodechar{∞}{\ensuremath{\infty}}
 
\begin{document}

\title[]{Multiparty Motion Coordination:\\From Choreographies to Robotics Programs}
\author{Rupak Majumdar}
\affiliation{
  \institution{Max Planck Institute for Software Systems}            %
  \city{Kaiserslautern}
  \country{Germany}                    %
}
\email{rupak@mpi-sws.org}          %

\author{Nobuko Yoshida}
\email{n.yoshida@imperial.ac.uk}
\orcid{0000−0002−3925−8557}
\affiliation{%
  \institution{Imperial College London}
  \department{Computing}
  \streetaddress{180 Queen's Gate}
  \city{London}
  \postcode{SW7 2AZ}
\country{United Kingdom}
}
\author{Damien Zufferey}
\orcid{0000-0002-3197-8736}             %
\affiliation{
  \institution{Max Planck Institute for Software Systems}            %
  \city{Kaiserslautern}
  \country{Germany}                    %
}
\email{zufferey@mpi-sws.org}          %

\toggletrue{full}

\begin{abstract}
We present a programming model and typing discipline for complex multi-robot coordination programming.
Our model encompasses both synchronisation through message passing and continuous-time 
dynamic motion primitives in physical space.
We specify \emph{continuous-time motion primitives} in an assume-guarantee 
logic that ensures compatibility of motion primitives as well as collision freedom.
We specify global behaviour of programs in a \emph{choreographic} type system that extends
multiparty session types with jointly executed motion primitives, predicated refinements, 
as well as a \emph{separating conjunction} that allows reasoning about subsets of interacting
robots.
We describe a notion of \emph{well-formedness} for global types 
that ensures motion and communication can be correctly synchronised and provide algorithms for
checking well-formedness, projecting a type, and local type checking.
A well-typed program is \emph{communication safe}, \emph{motion compatible}, and \emph{collision free}.
Our type system provides a compositional approach to ensuring these properties.

We have implemented our model on top of the ROS framework.
This allows us to program multi-robot coordination scenarios 
on top of commercial and custom robotics hardware platforms.
We show through case studies that we can model and statically verify quite
complex manoeuvres involving multiple manipulators and mobile robots---such
examples are beyond the scope of previous approaches.

 \end{abstract}


\maketitle

\section{Introduction}
\label{sec:intro}

Modern robotics applications are often deployed in safety- or business-critical applications
and formal specifications and reasoning about their correct behaviours is a difficult and challenging problem.
These applications tightly  
integrate computation, communication, control of physical dynamics,
and geometric reasoning.
Developing programming models and frameworks for such applications has been long noted as an important and challenging problem
in robotics \cite{Lozano-Perez}, 
and yet very little support exists today for compositional specification of behaviours or their static enforcement. 

In this paper, we present a programming model for concurrent robotic systems
and a type-based approach to statically analyse programs. 
Our programming model uses \emph{choreographies}---global protocols which allow implementing distributed and concurrent 
components without a central control---to compositionally specify and statically verify both 
message-based communications and jointly executed motion between robotics components in the physical world.
Our choreographies 
extend \emph{multiparty session types} \cite{HYC08}
with \emph{dynamic motion primitives}, which represent 
pre-verified motions that the robots can execute in the physical world.
We compile well-typed programs into robotics platforms.
We show through a number of nontrivial usecases 
how our programming model can be used to design and implement
correct-by-construction, complex, robotic applications on top 
of commercial and custom-build robotic hardware.

Our starting point is the theory of motion session types
\cite{ECOOP19}, which forms 
a type discipline based on global types with simple, discrete-time motion primitives. 
For many applications, we found this existing model 
too restrictive: it requires that all components agree on a pre-determined, global,
discrete clock and it forces \emph{global synchronisation} among all robots at the fixed interval
determined by the ``tick'' of the global clock.
For example, two robots engaged in independent
activities in different parts of a workspace must 
nevertheless stop their motion and synchronise every tick.
This leads to communication-heavy programs in which the programmer must either pick
a global clock that ensures every motion primitive can finish within one tick
(making the system very slow) or that every motion primitive is interruptible
(all robots stop every tick and coordinate).
Our new model enhances the scope and applicability of motion sesstion types to robots: 
we go from the global and discrete clock to \emph{continuous
behaviours} over time, allowing complex synchronisations as well as
\emph{frame separation} between independent subgroups of robots.
That is, the programmer does
not need to think about global ticks when writing the program. 
Instead, they think in terms of motion primitives and the type system enforces that concurrent composition of 
motions is well-formed and that trajectories exist in a global timeline.

Reasoning simultaneously about dynamics and message-based
synchronisation is difficult because time is global and can be used as
an implicit broadcast synchronisation mechanism.  
The complexity of our model arises from the need to ensure that every component is
simultaneously ready to let time progress (through dynamics) or 
ready to send or receive messages (property \emph{communication
  safety}).  At the same time, we must ensure that systems 
are able to correctly execute motion primitives
(property \emph{motion compatibility});
and jointly executed
trajectories are separated in space and time (property \emph{collision
  freedom}).  
Our verification technique is \emph{static}: if a program type checks, then
every execution of the system satisfies communication safety, motion compatibility, and collision freedom. 
We manage the complexity of reasoning about the interplay between dynamics and concurrency
through a separation of concerns.

First, we specify \emph{continuous-time trajectories} as \emph{motion primitives}.  
Since the dynamics of different components can
be coupled, the proof system uses 
an \emph{assume-guarantee proof system}
\cite{AbadiLamport93,ChandyMisra88,Jones83,DBLP:journals/pieee/NuzzoSBGV15} on continuous time
processes to relate an abstract motion primitive to the original
dynamics.  The assume guarantee contracts decouple the dynamics.  
Additionally, the proof system also checks that trajectories
ensure disjointness of the use of geometric space over time.

Second, we interpret choreographic specifications in continuous time, 
and extend the existing formalism in \cite{ECOOP19} with \emph{predicate refinements}---to reason about
permissions (i.e., what parts of the state space an individual robot can access without colliding into a different robot). 
This is required for motion compatibility and collision freedom. 
We also introduce a \emph{separating conjunction} operator to reason about disjoint frames.  
The combination of message passing and dynamics makes static verification challenging.
We introduce a notion of \emph{well-formedness} of choreographies 
that ensures motion and communication can
be synchronised and disallows, e.g., behaviours when a message is blocked because a motion
cannot be interrupted.
We give an algorithm for checking well-formedness using a dataflow
analysis on choreographies.
We show that well-formed types can be projected on to end-components
and provide a local type checking that allows refinement between
motion primitives. %

Checking compatibility and collision freedom reduces to validity queries in the underlying logic.
Interestingly, the separating conjunction allows subgroups of robots to
interact---through motion and messages---disjointly from other subgroups;
this reduces the need for global communication in our implementations.

We compile well-typed programs into programs in the PGCD \cite{PGCD} and ROS \cite{ROS} frameworks.
We have used our implementation to program and verify a number of complex robotic coordination and manipulation
tasks.
Our implementation uses both commercial platforms (e.g., the Panda 7DOF manipulator, the BCN3D MOVEO arm)
and custom-built components (mobile carts).
In our experience, our programming model and typing discipline are sufficient to specify 
quite complex manoeuvres between multiple robots and to obtain verified implementations.
Moreover, the use of choreographies and the separating conjunction are crucial in reducing verification effort: 
without the separating conjunction, verification times out on more complex examples. 
Our implementation uses SMT solvers to discharge validity queries arising out of the proof system;
our initial experience suggests that while the underlying theories (non-linear arithmetic) are 
complex, it is possible to semi-automatically prove quite involved specifications. 

\paragraph{\bf Contributions and Outline}
This paper provides a static compositional modelling,
verification, and implementation framework through behavioural 
specifications for
concurrent robotics applications 
that involve reasoning about message
passing, continuous control, and geometric separation.
We manage this complexity by decoupling dynamics and message
passing and enables us to specify and implement robotic applications on top of commercial and custom-build robotic hardware.
Our framework coherently integrates programming languages techniques,
session type theories, 
and static analysis techniques; 
this enables us to model \emph{continuous behaviours} over time in the presence of
complex synchronisations between independent subgroups of robots.
We extend the global types in \cite{ECOOP19}
  to \emph{choreographies} including 
  key constructs such as framing and predicate refinements,
  together with separation conjunctions,  
  and integrate the typing system with an assume-guarantee proof system. We
  implement our verification system using a new set of robots
  (extending \cite{ECOOP19,PGCD}) in order to 
program and verify complex coordinations and manipulation tasks.

The rest of the paper is organised as follows.
\begin{description}
\item[\underline{\sc\Sec~\ref{sec:example}}:]
  We first motivate our design for robotics specifications, 
and explain the correctness
  criteria -- \emph{communication safety},
  \emph{motion compatibility}, and
  \emph{collision freedom} with an example.
\item[\underline{\sc\Sec~\ref{sec:calculus}:}]
We introduce a \emph{multiparty motion session calculus} with 
\emph{dynamic abstract motion primitives};
we provide an assume guarantee proof
system to construct abstract motion primitives 
for maintaining geometric
separation in space and time for concurrently executed motion
primitives. We then
prove Theorem~\ref{th:motion-compat} (\emph{Motion Compatibility}).  
\item[\underline{\sc\Sec~\ref{sec:proof-system}:}]
  We provide \emph{choreographic specifications}
enriched with dynamic motion primitives and separation operator,  
and define their dataflow analysis. 
We propose a typing system and prove 
its main properties, 
Theorem~\ref{th:SR} (\emph{Subject Reduction}),
Theorem~\ref{th:progress} (\emph{Communication and Motion Progress} and
\emph{Collision-free Progress}). 
\item[\underline{\sc\Sec~\ref{sec:eval}:}]  
We describe our implementation, evaluation, and case studies. 
Our evaluation demonstrates that our framework allows specifying
complex interactions and verifying examples beyond the scope of previous work.
\item[\underline{\sc\Sec~\ref{sec:related} and \Sec~\ref{sec:discussion}:}]
  We discuss related work and conclude with 
  a number of open challenges in reasoning about robotics applications. 
\end{description}
Putting it all together, we obtain a compositional specification and 
implementation framework for concurrent robotics applications.
Our paper is a starting point rather than the final word on robot programming.
Indeed, we make many simplifying assumptions about world modeling,
sensing and filtering, and (distributed) feedback control and planning.
Even with these simplifications, the theoretical development is non-trivial.
With a firm understanding of the basic theory, we hope that future work
will lift many of the current limitations.

\begin{figure}[t]
\hspace*{-1.5cm}  
\begin{minipage}{0.45\linewidth}

\tikzset {_mwlgzwu4u/.code = {\pgfsetadditionalshadetransform{ \pgftransformshift{\pgfpoint{89.1 bp } { -108.9 bp }  }  \pgftransformscale{1.32 }  }}}
\pgfdeclareradialshading{_4euyx15e0}{\pgfpoint{-72bp}{88bp}}{rgb(0bp)=(1,1,1);
rgb(0bp)=(1,1,1);
rgb(25bp)=(0,0,0);
rgb(400bp)=(0,0,0)}

\tikzset {_7ai0opi17/.code = {\pgfsetadditionalshadetransform{ \pgftransformshift{\pgfpoint{89.1 bp } { -108.9 bp }  }  \pgftransformscale{1.32 }  }}}
\pgfdeclareradialshading{_co9gpcsuk}{\pgfpoint{-72bp}{88bp}}{rgb(0bp)=(1,1,1);
rgb(0bp)=(1,1,1);
rgb(25bp)=(0,0,0);
rgb(400bp)=(0,0,0)}

\tikzset {_eoe8gdy8h/.code = {\pgfsetadditionalshadetransform{ \pgftransformshift{\pgfpoint{89.1 bp } { -108.9 bp }  }  \pgftransformscale{1.32 }  }}}
\pgfdeclareradialshading{_dhd3hjkey}{\pgfpoint{-72bp}{88bp}}{rgb(0bp)=(1,1,1);
rgb(0bp)=(1,1,1);
rgb(25bp)=(0,0,0);
rgb(400bp)=(0,0,0)}

\tikzset {_cis49r535/.code = {\pgfsetadditionalshadetransform{ \pgftransformshift{\pgfpoint{89.1 bp } { -108.9 bp }  }  \pgftransformscale{1.32 }  }}}
\pgfdeclareradialshading{_9wx7ylssl}{\pgfpoint{-72bp}{88bp}}{rgb(0bp)=(1,1,1);
rgb(0bp)=(1,1,1);
rgb(25bp)=(0,0,0);
rgb(400bp)=(0,0,0)}
\tikzset{every picture/.style={line width=0.75pt}} %

\begin{tikzpicture}[x=0.75pt,y=0.75pt,yscale=-0.6,xscale=0.6]
\draw    (177,145) -- (400.5,146) ;
\draw   (256.5,128) -- (318.5,128) -- (318.5,138) -- (256.5,138) -- cycle ;
\draw  [fill={rgb, 255:red, 255; green, 255; blue, 255 }  ,fill opacity=1 ] (263.5,138.5) .. controls (263.5,134.08) and (267.08,130.5) .. (271.5,130.5) .. controls (275.92,130.5) and (279.5,134.08) .. (279.5,138.5) .. controls (279.5,142.92) and (275.92,146.5) .. (271.5,146.5) .. controls (267.08,146.5) and (263.5,142.92) .. (263.5,138.5) -- cycle ;
\draw  [fill={rgb, 255:red, 255; green, 255; blue, 255 }  ,fill opacity=1 ] (296.5,139) .. controls (296.5,134.44) and (300.19,130.75) .. (304.75,130.75) .. controls (309.31,130.75) and (313,134.44) .. (313,139) .. controls (313,143.56) and (309.31,147.25) .. (304.75,147.25) .. controls (300.19,147.25) and (296.5,143.56) .. (296.5,139) -- cycle ;
\draw   (434,146) -- (428,126) -- (403.74,126) -- (397.74,146) -- cycle ;
\path  [shading=_4euyx15e0,_mwlgzwu4u] (419.45,126.5) .. controls (419.45,124.22) and (417.85,122.37) .. (415.87,122.37) .. controls (413.89,122.37) and (412.29,124.22) .. (412.29,126.5) .. controls (412.29,128.78) and (413.89,130.63) .. (415.87,130.63) .. controls (417.85,130.63) and (419.45,128.78) .. (419.45,126.5)(424.81,126.5) .. controls (424.81,121.25) and (420.81,117) .. (415.87,117) .. controls (410.93,117) and (406.93,121.25) .. (406.93,126.5) .. controls (406.93,131.75) and (410.93,136) .. (415.87,136) .. controls (420.81,136) and (424.81,131.75) .. (424.81,126.5) ; %
 \draw   (419.45,126.5) .. controls (419.45,124.22) and (417.85,122.37) .. (415.87,122.37) .. controls (413.89,122.37) and (412.29,124.22) .. (412.29,126.5) .. controls (412.29,128.78) and (413.89,130.63) .. (415.87,130.63) .. controls (417.85,130.63) and (419.45,128.78) .. (419.45,126.5)(424.81,126.5) .. controls (424.81,121.25) and (420.81,117) .. (415.87,117) .. controls (410.93,117) and (406.93,121.25) .. (406.93,126.5) .. controls (406.93,131.75) and (410.93,136) .. (415.87,136) .. controls (420.81,136) and (424.81,131.75) .. (424.81,126.5) ; %

\path  [shading=_co9gpcsuk,_7ai0opi17] (394.31,101.5) .. controls (394.31,99.22) and (392.71,97.37) .. (390.73,97.37) .. controls (388.75,97.37) and (387.15,99.22) .. (387.15,101.5) .. controls (387.15,103.78) and (388.75,105.63) .. (390.73,105.63) .. controls (392.71,105.63) and (394.31,103.78) .. (394.31,101.5)(399.67,101.5) .. controls (399.67,96.25) and (395.67,92) .. (390.73,92) .. controls (385.79,92) and (381.78,96.25) .. (381.78,101.5) .. controls (381.78,106.75) and (385.79,111) .. (390.73,111) .. controls (395.67,111) and (399.67,106.75) .. (399.67,101.5) ; %
 \draw   (394.31,101.5) .. controls (394.31,99.22) and (392.71,97.37) .. (390.73,97.37) .. controls (388.75,97.37) and (387.15,99.22) .. (387.15,101.5) .. controls (387.15,103.78) and (388.75,105.63) .. (390.73,105.63) .. controls (392.71,105.63) and (394.31,103.78) .. (394.31,101.5)(399.67,101.5) .. controls (399.67,96.25) and (395.67,92) .. (390.73,92) .. controls (385.79,92) and (381.78,96.25) .. (381.78,101.5) .. controls (381.78,106.75) and (385.79,111) .. (390.73,111) .. controls (395.67,111) and (399.67,106.75) .. (399.67,101.5) ; %

\draw [line width=2.25]    (415.87,117) -- (399.67,101.5) ;
\draw [line width=2.25]    (406.93,126.5) -- (390.73,111) ;
\draw [line width=2.25]    (385.17,95) -- (366.8,102) ;
\draw [line width=2.25]    (383.24,108) -- (368.73,112) ;
\draw    (365.83,97) -- (370.67,117) ;
\draw  [fill={rgb, 255:red, 0; green, 0; blue, 0 }  ,fill opacity=1 ] (357.5,102.59) -- (365.53,98.71) -- (366.35,102.1) -- cycle ;
\draw  [fill={rgb, 255:red, 0; green, 0; blue, 0 }  ,fill opacity=1 ] (360.4,116.59) -- (368.43,112.71) -- (369.25,116.1) -- cycle ;

\draw   (139.5,145) -- (145.5,125) -- (171.1,125) -- (177.1,145) -- cycle ;
\path  [shading=_dhd3hjkey,_eoe8gdy8h] (154.59,125.5) .. controls (154.59,123.33) and (156.25,121.57) .. (158.3,121.57) .. controls (160.35,121.57) and (162.01,123.33) .. (162.01,125.5) .. controls (162.01,127.67) and (160.35,129.43) .. (158.3,129.43) .. controls (156.25,129.43) and (154.59,127.67) .. (154.59,125.5)(149.03,125.5) .. controls (149.03,120.25) and (153.18,116) .. (158.3,116) .. controls (163.42,116) and (167.58,120.25) .. (167.58,125.5) .. controls (167.58,130.75) and (163.42,135) .. (158.3,135) .. controls (153.18,135) and (149.03,130.75) .. (149.03,125.5) ; %
 \draw   (154.59,125.5) .. controls (154.59,123.33) and (156.25,121.57) .. (158.3,121.57) .. controls (160.35,121.57) and (162.01,123.33) .. (162.01,125.5) .. controls (162.01,127.67) and (160.35,129.43) .. (158.3,129.43) .. controls (156.25,129.43) and (154.59,127.67) .. (154.59,125.5)(149.03,125.5) .. controls (149.03,120.25) and (153.18,116) .. (158.3,116) .. controls (163.42,116) and (167.58,120.25) .. (167.58,125.5) .. controls (167.58,130.75) and (163.42,135) .. (158.3,135) .. controls (153.18,135) and (149.03,130.75) .. (149.03,125.5) ; %

\path  [shading=_9wx7ylssl,_cis49r535] (180.66,100.5) .. controls (180.66,98.33) and (182.32,96.57) .. (184.37,96.57) .. controls (186.42,96.57) and (188.08,98.33) .. (188.08,100.5) .. controls (188.08,102.67) and (186.42,104.43) .. (184.37,104.43) .. controls (182.32,104.43) and (180.66,102.67) .. (180.66,100.5)(175.1,100.5) .. controls (175.1,95.25) and (179.25,91) .. (184.37,91) .. controls (189.5,91) and (193.65,95.25) .. (193.65,100.5) .. controls (193.65,105.75) and (189.5,110) .. (184.37,110) .. controls (179.25,110) and (175.1,105.75) .. (175.1,100.5) ; %
 \draw   (180.66,100.5) .. controls (180.66,98.33) and (182.32,96.57) .. (184.37,96.57) .. controls (186.42,96.57) and (188.08,98.33) .. (188.08,100.5) .. controls (188.08,102.67) and (186.42,104.43) .. (184.37,104.43) .. controls (182.32,104.43) and (180.66,102.67) .. (180.66,100.5)(175.1,100.5) .. controls (175.1,95.25) and (179.25,91) .. (184.37,91) .. controls (189.5,91) and (193.65,95.25) .. (193.65,100.5) .. controls (193.65,105.75) and (189.5,110) .. (184.37,110) .. controls (179.25,110) and (175.1,105.75) .. (175.1,100.5) ; %

\draw [line width=2.25]    (158.3,116) -- (175.1,100.5) ;
\draw [line width=2.25]    (167.58,125.5) -- (184.37,110) ;
\draw [line width=2.25]    (190.14,94) -- (209.19,101) ;
\draw [line width=2.25]    (192.14,107) -- (207.19,111) ;
\draw    (210.19,96) -- (205.18,116) ;
\draw  [fill={rgb, 255:red, 0; green, 0; blue, 0 }  ,fill opacity=1 ] (218.83,101.59) -- (210.51,97.71) -- (209.65,101.11) -- cycle ;
\draw  [fill={rgb, 255:red, 0; green, 0; blue, 0 }  ,fill opacity=1 ] (215.82,115.59) -- (207.5,111.71) -- (206.64,115.11) -- cycle ;

\draw  [fill={rgb, 255:red, 126; green, 211; blue, 33 }  ,fill opacity=1 ] (446.74,133.54) -- (458.24,133.54) -- (458.24,143.54) -- (446.74,143.54) -- cycle ;
\draw  [fill={rgb, 255:red, 126; green, 211; blue, 33 }  ,fill opacity=1 ] (451.24,123.54) -- (462.74,123.54) -- (462.74,133.54) -- (451.24,133.54) -- cycle ;
\draw  [fill={rgb, 255:red, 126; green, 211; blue, 33 }  ,fill opacity=1 ] (458.24,133.54) -- (469.74,133.54) -- (469.74,143.54) -- (458.24,143.54) -- cycle ;
\draw  [fill={rgb, 255:red, 208; green, 2; blue, 27 }  ,fill opacity=1 ] (128.35,133) -- (116.43,133) -- (116.43,143) -- (128.35,143) -- cycle ;
\draw  [fill={rgb, 255:red, 208; green, 2; blue, 27 }  ,fill opacity=1 ] (116.43,133) -- (104.5,133) -- (104.5,143) -- (116.43,143) -- cycle ;
\draw  [fill={rgb, 255:red, 208; green, 2; blue, 27 }  ,fill opacity=1 ] (122.13,122) -- (110.2,122) -- (110.2,132) -- (122.13,132) -- cycle ;
\draw  [fill={rgb, 255:red, 126; green, 211; blue, 33 }  ,fill opacity=1 ] (281.74,117.54) -- (293.24,117.54) -- (293.24,127.54) -- (281.74,127.54) -- cycle ;
\draw   (247.5,107.5) -- (256.3,101) -- (256.3,104.25) -- (269.5,104.25) -- (269.5,110.75) -- (256.3,110.75) -- (256.3,114) -- cycle ;
\draw   (328.49,107.16) -- (319.88,113.94) -- (319.79,110.67) -- (306.59,111.04) -- (306.41,104.52) -- (319.61,104.15) -- (319.51,100.88) -- cycle ;

\draw (405,157) node   [align=left] {$\displaystyle \mathsf{GRobot}$};
\draw (292,157) node   [align=left] {$\displaystyle \mathsf{Cart}$};
\draw (158,157) node   [align=left] {$\displaystyle \mathsf{RRobot}$};

\end{tikzpicture}

 \includegraphics[width=0.9\linewidth]{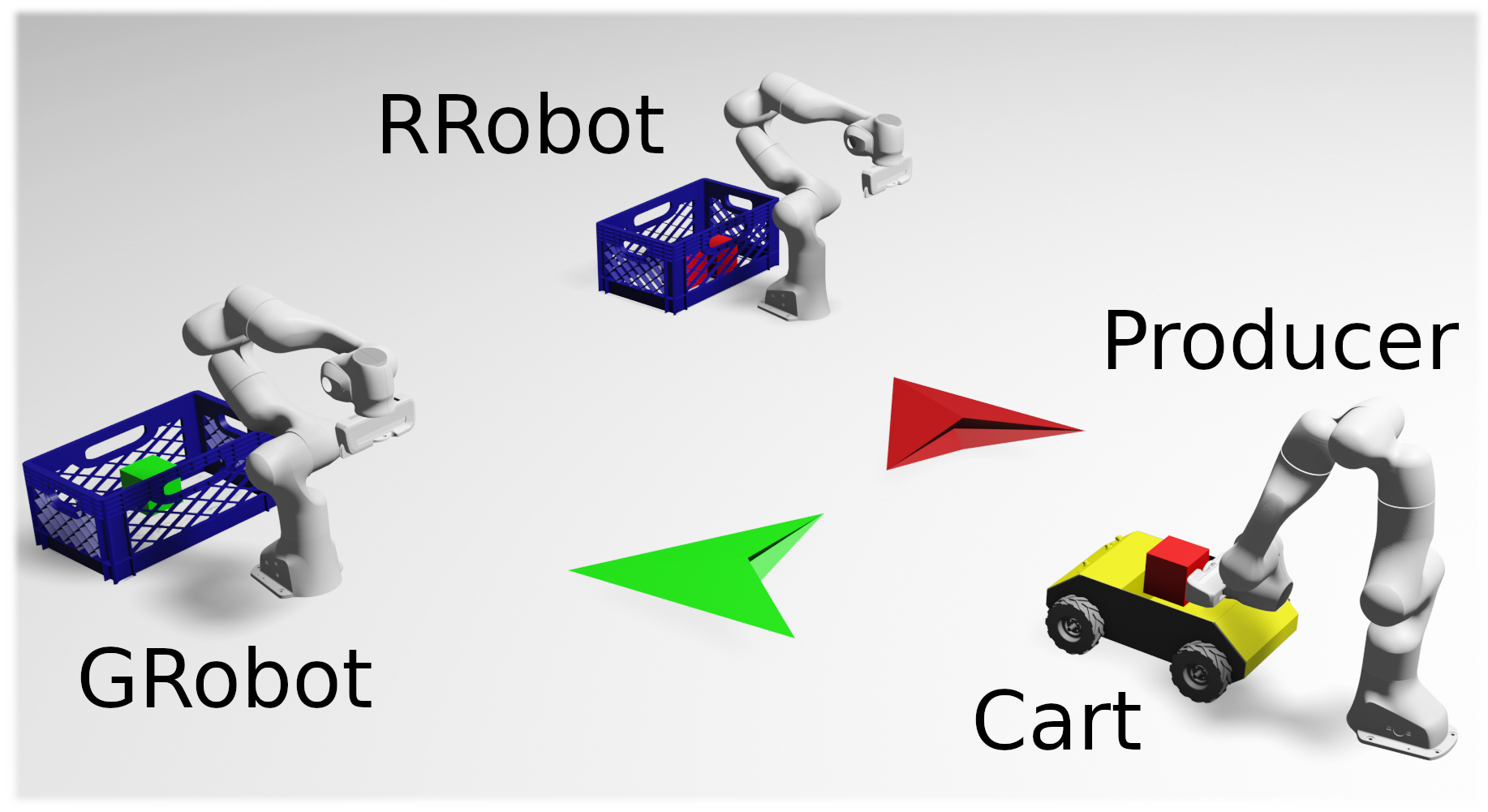}
\end{minipage}
\begin{minipage}[c]{0.40\linewidth}

\tikzset{every picture/.style={line width=0.75pt}} %

\begin{tikzpicture}[x=0.75pt,y=0.75pt,yscale=-0.6,xscale=1]
\draw [line width=2.25]    (200.97,24.47) -- (198.46,300) ;
\draw [shift={(198.42,304)}, rotate = 270.52] [color={rgb, 255:red, 0; green, 0; blue, 0 }  ][line width=2.25]    (17.49,-5.26) .. controls (11.12,-2.23) and (5.29,-0.48) .. (0,0) .. controls (5.29,0.48) and (11.12,2.23) .. (17.49,5.26)   ;
\draw [line width=2.25]    (310.21,25.97) -- (307.69,298.49) ;
\draw [shift={(307.66,302.49)}, rotate = 270.53] [color={rgb, 255:red, 0; green, 0; blue, 0 }  ][line width=2.25]    (17.49,-5.26) .. controls (11.12,-2.23) and (5.29,-0.48) .. (0,0) .. controls (5.29,0.48) and (11.12,2.23) .. (17.49,5.26)   ;
\draw  [fill={rgb, 255:red, 238; green, 10; blue, 10 }  ,fill opacity=1 ] (193.32,76.46) -- (205.57,76.46) -- (205.57,144.27) -- (193.32,144.27) -- cycle ;
\draw  [fill={rgb, 255:red, 184; green, 233; blue, 134 }  ,fill opacity=1 ] (205.57,154.82) -- (194.34,154.82) -- (194.34,206.8) -- (205.57,206.8) -- cycle ;
\draw [line width=2.25]    (403.11,25.22) -- (400.6,297.74) ;
\draw [shift={(400.56,301.74)}, rotate = 270.53] [color={rgb, 255:red, 0; green, 0; blue, 0 }  ][line width=2.25]    (17.49,-5.26) .. controls (11.12,-2.23) and (5.29,-0.48) .. (0,0) .. controls (5.29,0.48) and (11.12,2.23) .. (17.49,5.26)   ;
\draw  [fill={rgb, 255:red, 30; green, 211; blue, 231 }  ,fill opacity=1 ] (315.82,78.72) -- (303.57,78.72) -- (303.57,145.77) -- (315.82,145.77) -- cycle ;
\draw    (200.46,51.59) -- (307.7,59.73) ;
\draw [shift={(309.7,59.88)}, rotate = 184.34] [color={rgb, 255:red, 0; green, 0; blue, 0 }  ][line width=0.75]    (10.93,-3.29) .. controls (6.95,-1.4) and (3.31,-0.3) .. (0,0) .. controls (3.31,0.3) and (6.95,1.4) .. (10.93,3.29)   ;
\draw    (200.46,62.14) -- (402.65,77.06) ;
\draw [shift={(404.64,77.21)}, rotate = 184.22] [color={rgb, 255:red, 0; green, 0; blue, 0 }  ][line width=0.75]    (10.93,-3.29) .. controls (6.95,-1.4) and (3.31,-0.3) .. (0,0) .. controls (3.31,0.3) and (6.95,1.4) .. (10.93,3.29)   ;
\draw    (201.48,144.27) -- (308.73,152.4) ;
\draw [shift={(310.72,152.55)}, rotate = 184.34] [color={rgb, 255:red, 0; green, 0; blue, 0 }  ][line width=0.75]    (10.93,-3.29) .. controls (6.95,-1.4) and (3.31,-0.3) .. (0,0) .. controls (3.31,0.3) and (6.95,1.4) .. (10.93,3.29)   ;
\draw  [fill={rgb, 255:red, 65; green, 117; blue, 5 }  ,fill opacity=1 ] (314.8,155.57) -- (303.57,155.57) -- (303.57,207.56) -- (314.8,207.56) -- cycle ;
\draw    (309.7,209.06) -- (200.41,220.9) ;
\draw [shift={(198.42,221.12)}, rotate = 353.82] [color={rgb, 255:red, 0; green, 0; blue, 0 }  ][line width=0.75]    (10.93,-3.29) .. controls (6.95,-1.4) and (3.31,-0.3) .. (0,0) .. controls (3.31,0.3) and (6.95,1.4) .. (10.93,3.29)   ;
\draw  [fill={rgb, 255:red, 238; green, 10; blue, 10 }  ,fill opacity=1 ] (193.32,222.63) -- (205.57,222.63) -- (205.57,287.42) -- (193.32,287.42) -- cycle ;
\draw  [fill={rgb, 255:red, 30; green, 211; blue, 231 }  ,fill opacity=1 ] (312.76,221.87) -- (301.53,221.87) -- (301.53,286.67) -- (312.76,286.67) -- cycle ;
\draw  [fill={rgb, 255:red, 30; green, 211; blue, 231 }  ,fill opacity=1 ] (408.73,79.47) -- (395.45,79.47) -- (395.45,286.67) -- (408.73,286.67) -- cycle ;
\draw  [dash pattern={on 4.5pt off 4.5pt}]  (153.5,76.46) -- (446.5,78.72) ;
\draw  [dash pattern={on 4.5pt off 4.5pt}]  (153.5,144.27) -- (346.45,145.77) ;
\draw  [dash pattern={on 4.5pt off 4.5pt}]  (154.52,206.05) -- (347.47,207.56) ;
\draw  [dash pattern={on 4.5pt off 4.5pt}]  (159.63,286.67) -- (426.08,287.42) ;

\draw (201.99,12.41) node  [font=\large] [align=left] {Cart};
\draw (401.07,13.16) node  [font=\large] [align=left] {RRobot};
\draw (260.31,43.58) node  [font=\large] [align=left] {$\displaystyle \mathit{arrive}$};
\draw (311.23,12.41) node  [font=\large] [align=left] {GRobot};
\draw (355.28,61.15) node  [font=\large] [align=left] {$\displaystyle \mathit{free}$};
\draw (259.18,160.32) node  [font=\large] [align=left] {$\displaystyle \mathit{ready}$};
\draw (257.25,223.8) node  [font=\large] [align=left] {$\displaystyle \mathit{ok}$};
\draw (358.41,154.02) node [anchor=north west][inner sep=0.75pt]  [font=\Huge] [align=left] {*};
\draw (202,127) node [anchor=north west][inner sep=0.75pt]  [font=\large] [align=left] {$\displaystyle {\displaystyle \rightsquigarrow }$};
\draw (202,269) node [anchor=north west][inner sep=0.75pt]  [font=\large] [align=left] {$\displaystyle \rightsquigarrow $};
\draw (286,189) node [anchor=north west][inner sep=0.75pt]  [font=\large] [align=left] {$\displaystyle \leftsquigarrow $};

\end{tikzpicture}
 \end{minipage}
\caption{An example with a cart and two robot arms
  (left: schematic on top, actual
robots used in experiments in \Sec~\ref{sec:eval} bottom) and a partial message sequence
for one behavior of the system (right). 
Message exchanges take zero time but are shown with downward sloping arrows for readability.
The colored boxes denote motion primitives and their execution allows physical time to pass. 
The light blue box denotes the motion primitive ``$\WORKM$''
(in Figure~\ref{fig:basic-example}),
red denotes ``$\MOVEM$'', light green ``$\IDLEM$'', dark green ``$\PICKM$''.
The dotted horizontal lines show synchronisation
of global time across components. 
The use of the separating conjunction ``$\ast$'' ensures that time need not
be synchronised with the red robot
at intermediate synchronisations between the cart and the green robot.
At each point, at most one motion primitive is non-interruptible (shown with ``$\delayed$'') %
\label{fig:basic_example_fig}}
\end{figure}
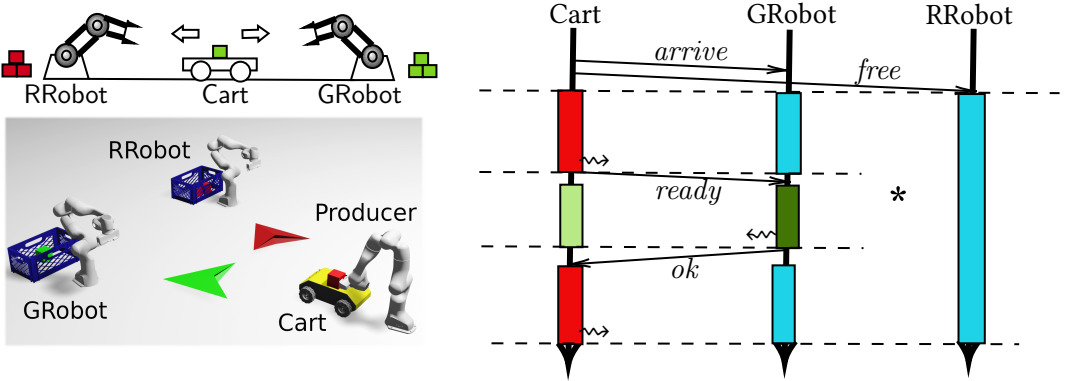

\begin{figure*}[t]
{
  \small
\begin{align*}
\mu & \ty. \!\!\!\!\!
 & 
  \left( \begin{array}{l}
     \left( 
          \begin{array}{l}  
                    \colorlet{oldcolor}{.} \color{blue}\boxed{\color{oldcolor}\Cart\to\Green:\ARRIVE.\Cart\to\Red:\FREE}.  \\[5px]
                    \left(\begin{array}{cc}
                         \left( \colorlet{oldcolor}{.} \color{green} \boxed{\color{oldcolor} \begin{array}{l}
                             \motion{\dtsess}{\Cart:\MOVEM^\delayed(\mathsf{point}),\Green:\WORKM^\instantaneous}.\\
                             \Cart\to\Green:\READY.\\
                              \motion{\dtsess}{\Cart:\IDLEM^\instantaneous,\Green:\PICKM^\delayed}.\\ 
                             \Green\to\Cart:\OK.\\
                             (\motion{\dtsess}{\Cart:\MOVEM^\delayed(\mathsf{point})} ~\GLPAR~ \motion{\dtsess}{\Green:\WORKM^\instantaneous})
                         \end{array}}  \color{oldcolor} \right)
                      & \!\!\GLPAR~ \motion{\dtsess}{\Red:\WORKM^\instantaneous}
                      \end{array} \right)
 \end{array}\right) \\
 + \\
 \big( \text{\quad as the above swapping $\Green$ with $\Red$\quad } \big) 
 \end{array} \right).\ty
\end{align*}
}
\vspace{-3mm}
\caption{Global choreography for the example\label{fig:basic-example}}
\vspace*{-4mm}
\end{figure*}
 
\section{Motivating Example}
\label{sec:example}
In this section, we show a coordination example and use it to motivate the different choices of our programming model and
choreography specifications.
We start with a simple protocol.
Consider a scenario where a cart shuttles (based on some unmodeled criterion) 
between two robotic arms, ``red'' and ``green'' (corresponding, e.g.,
to two different processing units).
In each round, the cart lets the robots know of its choice and then moves to the chosen robot.
On reaching its destination, the cart signals to the arm that it has arrived, and waits for the arm
to finish processing.
Meanwhile, the other arm can continue its own work independently.
When the arm finishes its processing, it signals the cart that processing is complete.
The cart moves back and then repeats the cycle.
Figure~\ref{fig:basic_example_fig} shows a schematic of the system as well as a sample
sequence of messages in each cycle.

\myparagraph{Multiparty Motion Choreographies}
We specify the global behaviour of a program using \emph{choreographies}, which describe 
the allowed sequence of global message exchanges and joint motion primitives as types for programs.
The choreography for our example is shown in Figure~\ref{fig:basic-example}.
It extends session types \cite{BLY2015,BYY2014,BMVY2019}
with joint motion primitives similarly to motion session types \cite{ECOOP19}
and a \emph{separating conjunction} $\GLPAR$. 

{\bf\em (1) Messages and motion primitives.} 
The choreography in Figure~\ref{fig:basic-example} is a
{\bf\emph{recursion}} ($\mu \ty.G$),
that makes a {\bf\emph{choice}} ($G_1 + G_2$) between two possibilities, 
corresponding to the cart interacting with the green or with the red robots (blue box). 
{\bf\em Message exchange} is specified by 
$\pp\to\pq:\ell$ 
which is a flow of a message labeled $\ell$
from process $\pp$ to process $\pq$.
{\bf\em Joint motion primitives} $\motion{\dtsess}{\pp:\ma, \pp':\ma', \cdots}$
specify that the processes $\pp$, $\pp'$, etc.\ jointly execute
motion primitives $\ma$, $\ma'$, etc.\ for the same amount of time.
Global time is a global synchroniser; the type system makes sure that
there is a consistent execution
that advances time in the same way for every process.

With just the construct for joint motion primitives, every component is forced to globally
synchronise in time, preventing the robots to work independently in physically isolated spaces. 
We would like to specify, for example, that when the cart and the green robot $\Green$ is
interacting, the red robot $\Red$ can independently perform its work without additional
synchronisations.

{\bf\em (2) Separating conjunctions.} 
To overcome this shortcoming,
we introduce the {\bf\emph{separating conjunction}} $\GLPAR$ (a
novel addition from previous session types \cite{ECOOP19,BLY2015,BYY2014,BMVY2019}, explained below),
which decomposes the participants into
\emph{disjoint} subgroups, in terms of communication, dynamics, and use of geometric space over time,
each of which can proceed independently until a merge point.

The core interaction between the cart and an arm happens within the green box.
Let us focus on the interaction when the green robot is picked.
In this case, the cart moves to the green robot (while the green arm performs internal work), 
then synchronises with a $\READY$ message, and then idles while the green arm picks an object from the cart.
When the arm is done, it sends an $\OK$ message, the cart moves back while the arm does internal work,
and the protocol round ends.
Meanwhile, the red robot can independently perform its work without synchronising with either the cart
or the green robot. 
The other branch of $+$ is identical, swapping the roles of the red and the green robots.

The separating conjunction makes it easy to specify physically and logically \emph{independent}
portions of a protocol.
Without it, we would be forced to send a message to the other (red) arm
every time the cart and the green arm synchronised, even if the red arm proceeded independently.
This would be required because time is global, and we would have to ensure that all motion primitives
execute for exactly the same duration.
Obviously, this leads to unwanted global synchronisations.

{\bf\em (3) Interruptible and non-interruptible motions.} 
Our type system ensures that messages and motions alternate in the correct way.
This requires differentiating between motion primitives
that are {\bf\emph{interruptible}} (marked by $\instantaneous$) 
from those that are not interruptible.
An interruptible motion primitive is one that can be interrupted by a message receipt.
For example, $\IDLEM$ or $\WORKM$ in our example are interruptible.
A {\bf\emph{non-interruptible}} motion primitive (marked by $\delayed$), on the other hand, should not be ``stopped'' by an external
message.
In order to ensure proper synchronisation between motion primitives and messages, 
we annotate motion primitives with $\instantaneous$ or $\delayed$, which specify whether a motion can be
interrupted by a message receipt or not.
The type system checks that there is, at any time, at most one non-interruptible motion primitive,
and moreover, the process executing that motion is the next one to send a synchronising message.
This ensures that motion is compatible with synchronisation.

The assumption that there is at most one non-interruptible motion primitive is a restriction; it rules
out complex multi-robot maneuvres.
However, there are already many ``loosely coupled'' systems that satisfy the assumption and it allows
us to develop the (already complicated) theory without introducing distributed feedback control strategies to the mix.
We leave lifting the assumption to future work.  

\myparagraph{Processes: Motion Primitives + Messages}
We model robotic systems as concurrent processes in a variant of 
the session $\pi$-calculus for multiparty interactions \cite{YG2020}. 
Our calculus (defined in \Sec~\ref{sec:calculus}) abstracts away from the sequential operations and only considers the \emph{synchronisation behaviour} and
the \emph{physical state changes}.
Synchronisation is implemented through synchronous message exchanges between 
individual physical components (participants). 
Between message exchanges, each process executes \emph{motion primitives}:
controller code that implements an actual robot motion.
Motion primitives affect state changes in the physical world.

\myparagraph{Correctness Criteria}
Our type system prevents a well-typed implementation from ``going wrong'' in the following ways. 

\textbf{\emph{Communication Safety.}} 
We provide a choreography type system to ensure that programs are \emph{communication safe} and \emph{deadlock free}:
a component does not get stuck sending a message with no recipient or waiting for a non-existent
message.
Ensuring communication safety is trickier for robotics programs because components may not be ready to receive messages 
because they are in the middle of executing a motion primitive that cannot be interrupted.

\textbf{\emph{Motion Compatibility.}}
We check that programs are able to \emph{jointly} execute motion primitives.
The controllers implementing motion primitives implement control actions for coupled dynamical systems
and we have to check that controllers for different components can be run together.

\textbf{\emph{Collision Freedom.}}
Components occupy 3D space; a correct implementation must ensure no geometric collision among objects.
For example, we must ensure that the robot arms do not hit the cart when the cart is moving in their workspace. 
Our motion compatibility check ensures the geometric footprint of each component remains disjoint from others
at all points in time.

\myparagraph{Type Checking} 
The type checking algorithm has three parts.
(1) First, we use an {\bf\emph{assume guarantee proof system}} to ensure jointly executed motion primitives are
{\bf\emph{compatible}}: they allow a global trajectory for all the robots in the system and ensure
there is no collision between components.
(2) Second, we introduce a notion of {\bf\emph{well-formedness}} for choreographies. 
Well-formed choreographies ensure motion and message exchanges can be globally synchronised and that
joint trajectories of the system are well-defined and collision free.
(3) Finally, as in standard session types, we {\bf\emph{project a global choreography to its end-points}}, yielding a local
type that can be checked for each process.
The challenge is to manage the delicate interplay between time, motion, message exchanges, choices and the separating
conjunction.
This makes the well-formedness check more sophisticated than other choreography type systems.

\section{Multiparty Motion Session Calculus}
\label{sec:calculus}
We now describe the motion session calculus, extending from \cite{ECOOP19}.

\subsection{Syntax} 
\myparagraph{Physical Components and Processes}
We assume a fixed static set $\PC$
of \emph{physical components} (ranged over by $\pp$, $\q$,...), which
are often called \emph{participants} or \emph{roles};
these are the different components that constitute the overall system.
Each physical component executes a software process, which 
takes care of communication, synchronisation and motion.
We describe the motion session calculus, which forms the core of the software process.

A value $\val$ can be a natural number $\valn$, an integer $\valr$, a boolean $\true$/$\false$, or a real number. %
An expression $\e$ can be a variable, a value, or a term built from expressions by applying (type-correct) computable operators.
A \emph{motion primitive} (\ma, \mb, $\ldots$) is an abstraction of underlying physical trajectories; for the moment,
consider a motion primitive simply as a name and describe its semantics later.
We use the notation $\motion{dt}{\ma}$ to represent that a motion primitive executes and time elapses.
We write the tuple $\motion{dt}{(\pp_i: \ma_i)}$ to denote a group of processes executing their respective motion primitives at the same time.
For the sake of simplicity, we sometimes use $\ma$ for both single or grouped motions.

The \emph{processes} ($\PP,Q,R,...$) of the multiparty motion session calculus are defined by:
\[
\begin{array}{lll}\PP   ::=  \procout {\pp}{\ell}{\e}{\PP} 
            \SEP \sum\limits_{i\in I}
            \procin{\pp}{\ell_i(x_i)}{\PP_i} \SEP \cond{\e}\PP \PP %
            \SEP
             \sum\limits_{i\in I}
            \procin{\pp}{\ell_i(x_i)}{\PP_i}+\motion{dt}{a}.{\PP}
            \SEP
            \motion{dt}{\ma}.\PP \SEP \mu\procvar.\PP \SEP \procvar %
\end{array}
\]
The output process $\procout{\pp}{\ell}{\e}{\Q}$ sends the value of expression $\e$ with label $\ell$ to participant $\pp$.
The sum of input processes (external choice) 
$\sum_{i\in I}\procin{\pp}{\ell_i(x_i)}{\PP_i}$ is a process that can accept a value
with label $ \ell_i $ from participant $ \pp$ for any $ i\in I$; 
$\sum_{i\in I}\procin{\pp}{\ell_i(x_i)}{\PP_i}+\motion{dt}{a}.{\PP}$
is an external choice with a \emph{default branch} with a motion action $\motion{dt}{a}.{\PP}$ which can always proceed when there is no message to receive.
According to the label $\ell_i$ of the received value, the variable $x_i$
is instantiated with the value in the continuation process $ \PP_i$. 
We assume that the set $I$ is always finite and non-empty. 
Motion primitives are indicated by $\motion{dt}{\ma}$; the $\kf{dt}$ denoting
that real time progresses when a motion primitive $\ma$ is executed.
The conditional process $\cond{\e}{\PP}{Q} $ represents the internal choice between processes $ \PP $ and $ \Q $. 
Which branch of the conditional process will be taken depends on the evaluation of the expression $ \e. $ 
The process $\mu \procvar.\PP$ is a recursive process.  
Note that our processes do not have a null process: this is because a physical component does not ``disappear'' when the program stops.
Instead, we model inaction with an iteration of some default motion primitive.

\begin{example}[Processes]\rm
The processes for the cart 
and the two robots from Figure~\ref{fig:basic_example_fig} are given as:
\[
\begin{array}{cc}
\begin{array}{rl}\small
\Cart:\!\!\!\!\!\!\!\!\!\!\!\!\\  
\mu \procvar.\!\!\!\!\!\!
& (\Green!\ARRIVE.\Red!\FREE.\\
& \quad \motion{dt}{\MOVEM(\text{co-ord of }\Green)}.\Green!\READY.\\
&  \quad \motion{dt}{\IDLEM}.\Green?\OK.\motion{dt}{\MOVEM(\text{base})}.\procvar\\
& + \textrm{ symmetrically for $\Red$ })
\end{array}
&
\!\!\!\!\!\!\!\!
\begin{array}{rl}
\small
\Red,\!\!\!\!\!\!\!\!\!\!\!\!\!\!\!\!\!\\\
\Green:\!\!\!\!\!\!\!\!\!\!\!\!\!\!\!\!\!\!\!\!\\
\mu \procvar.\!\!\!\!\!\!& 
(\Cart?\ARRIVE.\motion{dt}{\WORKM}.\Cart?\READY.\\
& \quad \ \motion{dt}{\PICKM}.\Cart!\OK.\motion{dt}{\WORKM}).\procvar\\
& \ + \ (\Cart?\FREE.\motion{dt}{\WORKM}).\procvar\\
\end{array}
\end{array}
\]
Note that for both processes, the recursion ensures that the program does not terminate.
Motion primitives for the cart involve moving to various locations $\MOVEM(\mathit{pos})$
and idling at a location ($\IDLEM$), and those for the robots involve doing some (unspecified)
work $\WORKM$ or picking items off the cart $\PICKM$
(internally, these motion primitives would involve motion planning and inverse kinematics for the robot arms).
Next, we describe the modelling and representation of motion primitives in more detail.
\end{example}

\myparagraph{Physical Variables and Footprint}
Motion primitives make the robot ``move'' in the physical world.
Each motion primitive represents an abstraction of an underlying controlled dynamical
system, such as the controller for the robot arm or the controller for a cart.
The dynamical system changes underlying physical state (such as the position and orientation of the arm
or the position and velocity of the cart).
The dynamics can be coupled: for example, if an arm is mounted on a cart, then the motion of the cart
is influenced by the mass and position of the arm.

We model the underlying physical system using \emph{physical variables}, and we partition these
into state variables $X$ (dynamical variables read and controlled by the component),
input variables $W$ (dynamical variables read by the component, whose values are provided by the environment).
The specification of a motion primitive will constrain the values of these variables over time.
We make the physical variables clear by writing $\pc \plays\proc{\PP}{{X}}{{W}}$ for a 
physical component $\pc\in \PC$ with state and input
variables $X$ and $W$, respectively, which executes the process $\PP$.

Each physical component $\pc\in \PC$ has a \emph{geometric footprint} $\geom{\pc}$ associated
with it.
This represents the physical space occupied by the component, and will be used to ensure that two
components do not collide.
The footprint is a function from valuations to variables in $X$ to a subset of $ℝ^3$.
It describes
an overapproximation of the space
occupied by the component as a function of the current state.
Note that the footprint can depend on the state.

\begin{example}\rm
\label{ex:mp-footprint}
Let $x$ and $v$ denote the position and the velocity of the cart, respectively, restricting
the discussion only on one axis.
Thus, $X = \set{x, v}$. If we assume there are no external influences on the cart, we can
take $W = \emptyset$.
The footprint provides a bounding box around the cart as it moves.
If the cart has dimensions $(l, w, h)$, the footprint at the location $(x, y, z)$ is 
\begin{align}
\set{(\mathsf{x}, \mathsf{y}, \mathsf{z}) \mid 
x-\frac{l}{2} \leq \mathsf{x} \leq x+\frac{l}{2} \wedge 
0\leq \mathsf{y} \leq h 
\wedge z -\frac{w}{2} \leq \mathsf{z} \leq z+\frac{w}{2}
}
\end{align} 
\end{example}

\myparagraph{Composition}
We now define parallel composition of physical components.
Composition of physical components ensure the following aspects.
First, like process calculi, parallel composition provides a locus for message
exchange: a physical component can send messages to, or receive messages from, another one.
Second, composition connects physical state variables of one component to the physical input variables
of another---this enables the coupling of the underlying dynamics.

To ease reasoning about connections between physical variables,
we assume that the components in a composition have disjoint sets of physical and
logical variables, and connections occur through syntactic naming of input variables.
Thus, a component with an input variable $x$ gets its value from the unique component that has a physical
state variable called $x$. 
Hence, there is no ambiguity in forming connections.
We refer to the variables of $\pc$ as $\pc.X$ and $\pc.W$.

We define a \emph{multiparty session} as a parallel composition of pairs of participants and processes:
\[ 
\begin{array}{lll}
   \N  &  ::=   & \pp\plays\proc{\PP}{X}{W}  \SEP \N \parcomp \N
 \end{array}
\]
with the intuition that process $\PP$ plays the role of participant $\pp$, %
and can interact with other processes playing other roles in $M$. %
A multiparty session is \emph{well-formed} if all its participants are different, 
and for each input variable of each participant there is a syntactically unique
state variable in a different participant so that connections between physical variables is well-defined.
We consider only well-formed processes. 

\begin{example}\rm
For the example from \Sec~\ref{sec:example}, the multiparty session is
\[
\Cart\plays\proc{\text{proc.\ for $\Cart$}}{\set{x,v}}{\emptyset} \parcomp
\Red \plays\proc{\text{proc.\ for $\Red$}}{\cdot}{\cdot} \parcomp
\Green \plays\proc{\text{proc.\ for $\Green$}}{\cdot}{\cdot} 
\]
(we have omitted the physical variables for the robot arms for simplicity).
\end{example}

\subsection{Motion Primitives}
\label{sec:mp-rules}

Before providing the operational semantics, we first
consider how motion primitives 
are specified. 
Recall that motion primitives of a component $\pc\plays\proc{\PP}{X}{W}$
abstract a trajectory arising out of the underlying dynamics. 
A first idea is to represent a motion primitive as a pair 
$\tuple{\pre, \post}$---this provides a precondition $\pre$ that specifies 
the condition on the state under which the 
motion primitive can be applied and 
a postcondition $\post$ that specifies the effect of applying the
motion primitive on the physical state. 
However, this is not sufficient: the motion of two components can be coupled in time and the  
trajectory of a component depends on the inputs it gets from other components 
and, in turn, its trajectory influences the trajectories of the other components. 
Therefore, a motion primitive also needs to specify \emph{assumptions} on its external inputs and 
\emph{guarantees} it provides to other processes over time.
These predicates are used to decouple the dynamics. 

Our motion primitive has three more ingredients. 
The first is the \emph{footprint}:
the geometric space used by a process over time while it executes its trajectory and is used 
to ensure geometric separation between components. 
The second is a pair $(D, \timestep)$ of a \emph{time interval and an annotation}: the time interval 
bounds the \emph{minimal} and \emph{maximal times} between which a motion primitive is ready to communicate 
via message passing, and the 
$\timestep$ annotation distinguishes between motion primitives
that are interruptible by external messages 
from those that cannot be interrupted. 

We assume preconditions and postconditions only depend on the state, while assumptions, guarantees,
and footprints may depend on both the state and the elapsed time. 

\begin{definition}[Motion primitive]
The specification for a motion primitive $\ma$ of a physical process $\pc\plays\proc{\PP}{X}{W}$,
written 
$\judgement{\ma}{\pre,A,G,\post,\FP,D, \timestep}$,
consists of 
two predicates \emph{precondition} $\pre$ and \emph{postcondition} $\post$ over $X\cup W$, 
an \emph{input assumption} $A$ which is a predicate over $W\cup\set{\clock}$, 
an \emph{output guarantee} $G$ which is a predicate over $X\cup\set{\clock}$, 
a \emph{footprint} $\FP$ which is a predicate over $X \cup \set{\mathsf{x},\mathsf{y},\mathsf{z},\clock}$,
a \emph{duration} $D\in (ℝ\cup\set{\infty})^2$ which is a time interval,
and 
$\timestep\in\set{\instantaneous,\delayed}$ which indicates if the motion primitive can be interrupted 
by an external message ($\instantaneous$) or not ($\delayed$). 

Given motion primitives $\ma$ and $\ma'$, we say $\ma$ \emph{refines} $\ma'$, denoted by 
$\ma \refines \ma'$,  
with $\judgement{\ma}{φ,A,G,ψ,\FP, D, \timestep}$  
and $\judgement{\ma'}{φ',A',G',ψ',\FP', D', \timestep'}$
iff 
(1) $φ ⇒ φ'$,
(2) $A ⇒ A'$,
(3) $G' ⇒ G$,
(4) $ψ' ⇒ ψ$, 
(5) $\FP' ⊆ \FP$, 
(6) $\timestep = \timestep'$, and
(7) either $\timestep =\instantaneous$ and $D⊆ D'$ or
$\timestep=\delayed$ and $D' ⊆ D$. 
\end{definition}

\subsection{Operational Semantics}
\label{sec:opsem}
The operational semantics is given as reduction rules relative to \emph{stores} $\sstore{X},\sstore{W}$
(Figure~\ref{fig:opsem})
that map physical variables to values. 
The semantics uses the standard structural rules defined
in Figure~\ref{tab:sync:congr}. 

We adopt some standard conventions regarding the syntax of processes and sessions. 
Namely, we  will use 
$\prod_{i\in I} {\pp_i}\plays\proc{\PP_i}{X_i}{W_i}$ for ${\pp_1}\plays\proc{\PP_1}{X_1}{W_1}\parcomp \ldots\parcomp {\pp_n}\plays\proc{\PP_n}{X_n}{W_n},$ 
where $I=\set{1,\ldots,n}$, or simply as $\prod_{i\in I} \pa{\pp_i}{\PP_i}$ when the physical variables are not important. %
We use infix notation for external choice process, 
e.g., instead of $\sum_{i\in \{1,2\}}
\procin{\pp}{\ell_i(\x)}{\PP_i}$, we write 
$\procin{\pp}{\ell_1(\x)}{\PP_1}+\procin{\pp}{\ell_2(\x)}{\PP_2}.$
{\em The value $\val$ of expression $\e$ with physical state $\sstore{X}$} (notation $\evalState{\e}{\val}{\sstore{X}}$) is computed as expected. 
We assume that $\evalState{\e}{\val}{\sstore{X}}$ is effectively computable and takes logical ``zero time.''

\begin{figure}[!t]
  \label{tab:sync:red}
\small
\[
\begin{array}[t]{@{}c@{}}
\rulename{recv}\inferrule%
{
j\in I 
}
{
\pp\plays\proc{\sum\limits_{i\in I} \procin{\q}{\ell_i(\x)}{\PP_i} + \motion{dt}{a}.{\PP} }{\sstore{X}}{\sstore{W}}
\xrightarrow{\q?\ell_j(v)}
\pp\plays\proc{\PP_j\set{v/x}}{\sstore{X}}{\sstore{W}}
}
\\[5mm]
\rulename{send}
\inferrule%
          {
            \evalState{\e}{\val}{\sstore{X},\sstore{W}}
}{
\pp\plays\proc{\procout{\q}{\ell_j}{\e}{Q}}{\sstore{X}}{\sstore{W}}
\xrightarrow{\q!\ell\tuple{\val}}
\pp\plays\proc{Q}{\sstore{X}}{\sstore{W}}
}
          \quad
          \rulename{comm}
          \inferrule%
{
\pa\pp{\PP} \xrightarrow{\q!\ell\tuple{\val}} \pa\pp{\PP'}\\
\pa\q{\Q} \xrightarrow{\pp?\ell(\val)} \pa\q{\Q'}\\
}
{
  \pa\pp{\PP} \parcomp \pa\q{\Q} %
  \red{} \pa\pp{\PP'} \parcomp \pa\q{\Q'} %
}
\\[6mm]
\inferrule[\rulename{default}]
{
\pp\plays \proc{\motion{\dtsess}{\ma}.\PP}{\sstore{X}}{\sstore{W}} \xrightarrow{\tau}
\pp\plays \proc{\motion{\dtsess}{\ma\AT 0}.\PP}{\sstore{X}}{\sstore{W}} \\
\pp\plays \proc{\motion{\dtsess}{\ma\AT 0}.\PP}{\sstore{X}}{\sstore{W}} \xrightarrow{\motion{\dtsess}{\ma},\xi,\nu, t} 
\pp\plays \proc{\motion{\dtsess}{\ma\AT t}.\PP}{\sstore{X}'}{\sstore{W}'}
} 
{
\pp\plays \proc{\sum\limits_{i\in I} \procin{\q}{\ell_i(\x)}{\PP_i} + \motion{\dtsess}{a}.{\PP}}{\sstore{X}}{\sstore{W}}
\xrightarrow{\motion{\dtsess}{\ma},\xi,\nu, t} 
\pp\plays \proc{\motion{\dtsess}{\ma\AT t}.\PP}{\sstore{X}'}{\sstore{W}'}
 }
\\[5mm]
\inferrule[\rulename{traj-base}]
          {
\pre\{\sstore{X}/X,\sstore{W}/W, 0/\clock\}\quad 
A\{\sstore{W}/W, 0/\clock\}\wedge G\{\sstore{X}/X, 0/\clock\} \\
\geom{\pp}(\sstore{X})\subseteq\FP\set{\sstore{X}/X,0/\clock}
}
{
\pp\plays \proc{\motion{\dtsess}{\ma}.\PP}{\sstore{X}}{\sstore{W}} \xrightarrow{\tau}
\pp\plays \proc{\motion{\dtsess}{\ma\AT 0}.\PP}{\sstore{X}}{\sstore{W}}
}
\\[5mm]
\rulename{traj-step}
\inferrule%
          {
  \xi: [t_1,t_2] \rightarrow ℝ^X, \nu: [t_1,t_2] \rightarrow ℝ^W\\
  \xi(t_1) = \sstore{X}, \xi(t_2) = \sstore{X}' \\
  \nu(t_1) = \sstore{W}, \nu(t_2) = \sstore{W}' \\
  \forall t\in [t_1,t_2].\ A\{\nu(t)/W, t/\clock\} \wedge G\{\xi(t)/X, t/\clock\}
\ \wedge \geom{\pp}(\xi(t))\subseteq\FP\set{\sstore{X}/X,t/\clock}
}
{
\pp\plays \proc{\motion{\dtsess}{\ma\AT t_1}.P}{\sstore{X}}{\sstore{W}} \xrightarrow{\motion{\dtsess}{\ma},\xi,\nu, t_2-t_1} 
\pp\plays \proc{\motion{\dtsess}{\ma\AT t_2}.P}{\sstore{X}'}{\sstore{W}'}
}
\\[7mm]
\inferrule[\rulename{t-conditional}]
{
  \evalState{\e}{\true}{\sstore{X},\sstore{W}}
}
{
\pp\plays\proc{\cond{\e}{\PP}{\Q}}{\sstore{X}}{\sstore{W}}
\xrightarrow{}
\pp\plays\proc{\PP}{\sstore{X}}{\sstore{W}}
}
\quad 
\inferrule[\rulename{non-interrupt}]
{
t \in D\\ 
\timestep = \delayed\\
\post\set{\sstore{X}/X,\sstore{W}/W,t/\clock}\\
}
{
\pp\plays \proc{\motion{\dtsess}{\ma\AT t}.\PP}{\sstore{X}}{\sstore{W}} \xrightarrow{\tau} \pp\plays \proc{\PP}{\sstore{X}}{\sstore{W}}
}
\\[5mm]
\rulename{interrupt}
\inferrule%
    {
t \in D\\ 
\timestep = \instantaneous\\
\post\set{\sstore{X}/X,\sstore{W}/W,t/\clock}\\
\pp\plays \proc{\PP}{\sstore{X}}{\sstore{W}} \xrightarrow{\q?\ell(\val)} \pp\plays \proc{\PP'}{\sstore{X}}{\sstore{W}}
}
{
\pp\plays \proc{\motion{\dtsess}{\ma\AT t}.\PP}{\sstore{X}}{\sstore{W}}
\xrightarrow{\q?\ell(\val)} 
\pp\plays \proc{\PP'}{\sstore{X}}{\sstore{W}}
}
\\[5mm]
\rulename{m-par}
\inferrule%
    {
\exists \xi, \nu.~ \forall i,j \in I.\\
  i\neq j \Rightarrow \mathsf{disjoint}(\ma_i,\ma_j) \\
  \pp_i\plays\proc{P_i}{\sstore{X}_i}{\sstore{W}_i}
    \xrightarrow{\motion{\dtsess}{\ma_i},\xi|_{X_i}, (\xi\cup\nu)|_{W_i}, t}
    \pp_i\plays\proc{P_i'}{\sstore{X}'_i}{\sstore{W}'_i}\\
}
{
  \prod_{i \in I} \pp_i\plays\proc{P_i}{\sstore{X}_i}{\sstore{W}_i} 
  \xrightarrow{\motion{\dtsess}{\ma_i}, t}
  \prod_{i \in I} \pp_i\plays\proc{P_i'}{\sstore{X}'_i}{\sstore{W}'_i}
}
\\[6mm]
\rulename{r-par$\tau$}
\inferrule%
    {
\N_1
\xrightarrow{\tau}
\N_2
}
{
\N_1
\parcomp \N \xrightarrow{\tau}
\N_2
\parcomp \N 
}
\qquad
\rulename{r-par}
\inferrule%
          {
\N_1
\red{}
\N_2
}
{
\N_1
\parcomp \N \red{}
\N_2
\parcomp \N 
}
\qquad
\rulename{r-struct}
\inferrule%
{
  \N'_1\equiv \N_1 \xrightarrow{\alpha} \N_2 \equiv \N'_2 
}
{ 
  \N'_1 \xrightarrow{\alpha} \N'_2 
}
\end{array}
\]
\noindent
We omit \rulename{f-conditional}. 
We use $\xrightarrow{\alpha}$ for 
any labelled transition relation or reduction ($\xrightarrow{}$). 
\caption{Operational semantics}
\label{fig:opsem}
\vspace*{-0.5cm}
\end{figure}
\begin{figure}{}
$
\small
\begin{array}[t]{@{}c@{}}
  \inferrule[\rulename{s-rec}]{}{ \mu \procvar.\PP \equiv \PP\sub{\mu \procvar.\PP}{\procvar} }
\quad
\inferrule[\rulename{s-multi}]{}{
	\PP \equiv \Q  \Rightarrow \pa\pp\PP \ | \  M\equiv\pa\pp\Q \ | \ M
} \quad 
\inferrule[\rulename{s-par 1}]{}{
	M \ | \ M' \equiv M' \ | \  M
}
\quad      
\inferrule[\rulename{s-par 2}]{}{
	(M \ | \ M') \ | \ M'' \equiv M \ | \ ( M' \ | \  M'')
} 
\end{array}
$
\caption{Structural congruence rules}
\label{tab:sync:congr}
\vspace*{-0.5cm}
\end{figure}

For reduction rules that do not change the physical state, we omit writing the physical state in the rule.
Time is global, and processes synchronize in time to make concurrent
motion steps of the same (but not pre-determined) duration.
Communication (\rulename{comm}) is synchronous and puts together sends and receives.
Rule \rulename{default} selects the default branch.
Rules for conditionals, communication without default motion  
and parallel composition 
are defined in a standard way \cite{ECOOP19,DGJPY2016,GHPSY19}.

To define the operational semantics for motions, we extend the process
syntax $P$ with time-annotated motion primitives, 
$\motion{dt}{\ma\AT t}\quad \text{ for }t\geq 0$. 
Let us fix a component $\pp$ and its motion primitive
$\judgement{\ma}{\pre,A,G,\post,\FP,D,\timestep}$.
Rules \rulename{traj-base} and \rulename{traj-step} set up trajectories for each motion primitive.
We distinguish between interruptible (\rulename{interrupt}) and non-interruptible (\rulename{non-interrupt})
motion primitives.
Non-interruptible motion primitives are consumed by the process.
Interruptible motion primitives consume the motion primitive on a message receipt.
The parallel composition rule \rulename{m-par} requires a consistent global trajectory and
ensures that when (physical) time elapses for one process, it elapses equally for all processes. 
Here, $\mathsf{disjoint}(\ma_i, \ma_j)$ states that the footprints along the trajectory are disjoint: 
    $\ma_i.\FP\set{\xi_i(t')/X,t'/\clock}\cap\ma_j.\FP\set{\xi_j(t')/X,t'/\clock}=\emptyset$ for all $t'\in[0,t]$.

We use $\red^*$ for the reflexive transitive closure of $\red$.
We say a program state $\prod_i \pp_i \plays\proc{\PP_i}{\sstore{X}_i}{\sstore{W}_i}$ is
\emph{collision free} if $\geom{\pp_i}(\sstore{X}_i)\cap\geom{\pp_j}(\sstore{X}_j) = \emptyset$
for every $i\neq j$.

\subsection{Joint Compatibility of Motion Primitives}
Two motion primitives are \emph{compatible} if they can be jointly executed.
To decide compatibility, we compose the specifications using the following
\emph{assume-guarantee proof rule}:\\
\centerline{
$
\small
\begin{array}{ll}
  \begin{array}{ll}
    \rulename{AGcomp}\\\\\\\\\
          \end{array}
    &
\inferrule{
\exists \FP_1, \FP_2.\
G_1 \wedge G_2 \Rightarrow \FP_1 \cap FP_2 = \emptyset\quad
G_1 \wedge G_2 \Rightarrow \FP_1 \cup FP_2 \subseteq \FP\\\\
 \timestep_1 =\instantaneous \vee \timestep_2 = \instantaneous \quad
 \timestep_1 =\delayed ⇒ D₁ ⊆ D₂ \quad
 \timestep_2 =\delayed ⇒ D₂ ⊆ D₁ \\\\
 \ma_1 \refines \tuple{φ₁,A∧ G₂,G₁,\FP_1,ψ₁,D_1,\timestep_1}\quad 
 \ma_2 \refines \tuple{φ₂,A∧ G₁,G₂,\FP_2,ψ₂,D_2,\timestep_2} 
}
          {\hspace*{-2mm}\judgement{\tuple{\pp_1{:}\ma_1, \pp_2{:}\ma_2}}{φ₁∧φ₂,A,G₁∧G₂,\FP,ψ₁∧ψ₂,D₁∩D₂,\timestep_1\oplus\timestep_2}}
          \end{array}
$}%

\noindent The $\oplus$ operator combines the interruptibility: 
$\instantaneous \oplus \instantaneous = \instantaneous$ and
$\instantaneous \oplus \delayed = \delayed \oplus \instantaneous = \delayed \oplus\delayed = \delayed$.

The \rulename{AGcomp} rule performs three checks.
First, the check on footprints ensures that the motion primitives are disjoint in space
(there is a way to find subsets $\FP_1$ and $\FP_2$ of the footprint $\FP$ so that the composed
motion primitives are in these disjoint portions).
Second, the guarantee of one motion primitive is used as the assumption of the other to check
that they are compatible.
Third, the checks on timing ensures that 
at most one process executes a non-interruptible motion primitive and 
the interruptible ones are ready before the non-interruptible one. 

We say 
$\motion{\dtsess}{(\pp_i {:} \ma_i)}$ is \emph{compatible} from $\pre$ if
there exist $G$, $\post$, $\FP$, and $D$ such that $\judgement{\prod_i{\pp_i:\ma_i}}{\pre, \true, G, \post, \FP, D, \timestep}$ is derivable
using \rulename{AGcomp} repeatedly.
Thus, compatibility checks that motion primitive specifications can be put together if 
processes start from their preconditions:
first, the assumptions and guarantees are compatible; 
second, there is no ``leftover'' assumption;
and
third, the footprints of the motion primitives do not intersect in space. 
Compatibility provides the ``converse'' condition that allows joint
trajectories in \rulename{M-par} to exist.

The next theorem formalizes what motion compatibility guarantees.
Intuitively, motion compatibility means it is sufficient to check the compatibility of contracts
to guarantee the existence of a joint trajectory, i.e., the execution is defined for all the components
and the joint trajectory satisfies the composition of the contracts.

\begin{restatable}[Motion Compatibility]{theorem}{theoremMC}
\label{th:motion-compat}
Suppose $\motion{\dtsess}{(\pp_1:\ma_1,\pp_2:\ma_2)}$ is compatible.
For every $t\in D$, if there exist trajectories $\xi_1, \xi_2,\nu_1,\nu_2$ such that
$\pa{\pp_i}{\PP_i} \xrightarrow{\motion{\dtsess}{\ma_i},\xi_i, \nu_i, t} \pa{\pp_i}{\PP_i'}$
for $i\in \set{1,2}$, then there exist trajectories
$\xi:[0,t]\rightarrow  ℝ^{X_1\cup X_2},
\nu:[0,t]\rightarrow ℝ^{W_1\cup W_2 \setminus (X_1\cup X_2)}$
such that
$\pa{\pp_i}{\PP_i} \xrightarrow{\motion{\dtsess}{\ma_i},\xi|_{X_i}, \nu|_{W_i}, t} \pa{\pp_i}{\PP_i'}$
and for all $0 \leq t' \leq t$, 
the footprints of $\pp_1$ and $\pp_2$ are disjoint:
$\geom{\pp_1}({\xi(t')|_{X_1}}) \cap \geom{\pp_2}({\xi(t')|_{X_2}}) = \emptyset$.
\end{restatable}

\begin{proof}
(Sketch.) %
This theorem is proved along the lines of the AG rule by \citet{DBLP:conf/hybrid/HenzingerMP01}.
Their proof relies on motion primitives having the following properties:
prefix closure,
deadlock freedom,
and input permissiveness.
Deadlock freedom is, in this setting \cite{DBLP:conf/hybrid/HenzingerMP01}, means that
    if a precond is satisfied then the trajectory must exists, and
    every execution that does not yet satisfies the postcondition can be prolonged be extended.
Input permissiveness states that a component cannot deadlock no matter how the environment decides to change the inputs.
This condition is needed to do induction over time.
Input permissiveness does not directly follow from the definition of our contracts.
However, it holds when the environment changes the inputs in a way that is allowed by the assumptions.
As \rulename{AGcomp} checks that this is true and rejects the composition otherwise, we can reuse the same proof strategy.
A process should not be allowed to block time making the system ``trivially safe'' but physically meaningless.
We also need to check the disjointness of footprints which is new in our model (needed by \rulename{M-par}).
This is also checked by \rulename{AGcomp}.
\end{proof}

\subsection{Examples for motion primitives, compatibilities and environments}
\myparagraph{Motion primitives from dynamics}
\label{ex:motion}
For the $\Cart$ in \Sec~\ref{sec:example}, we can derive motion primitives from
a simple dynamical model $\dot{x} = v, \dot{v} = u$.
Here, $X = \set{x,v}$ and $W= \emptyset$. 
For simplicity, assume the cart moves along a straight line (its $x$-axis)
and that the control $u$ can apply a fixed 
acceleration $a_{\max}$ or a fixed deceleration $-a_{\max}$.
Consider the motion primitive $\MOVEM$ which takes the cart from an initial position
and velocity $(x_i, v_i)$ to a final position $(x_f, v_f)$.
From high school physics, we can solve for $x$ and $v$, given initial values $x_i$ and $v_i$:\\
\centerline{$
x = x_i + v_it + \frac{1}{2}a_{\max} t^2 \mbox{ and } v = v_i + a_{\max} t
$}
from which, by eliminating $t$, we have $(v - v_i)^2 = 2 a (x - x_i)$.
Suppose that the cart starts from rest ($v_i=0$), accelerates until the midpoint $\mathit{mid} = \frac{1}{2}(x_i+x_f)$, 
and thereafter decelerates to reach $x_f$ again with velocity $v_f= 0$.
The precondition $\pre$ is $x_f > x\wedge v = 0$, the first conjunction saying we move right
(we can write another motion primitive for moving left).
The assumption $A$ is $\true$, and the guarantee is\\
\centerline{
$x_i \leq x \leq x_f \wedge \left(
\begin{array}{c}
(x_i \leq x \leq \mathit{mid} \Rightarrow v^2 = 2 a_{\max} (x - x_i))  \wedge\\
(\mathit{mid} \leq x \leq x_f \Rightarrow v^2 = 2 a_{\max} (x_f - x))
\end{array}
\right)
$}\\[1mm]
The postcondition is $x = x_f \wedge v = 0$.
The footprint provides a bounding box around the cart as it moves, and is given as in Example~\ref{ex:mp-footprint}.
We assume that the motion cannot be interrupted. Thus, we place the annotation $\delayed$.
The least time to destination is $t_m = 2\sqrt{(x_f-x_i)/a_{\max}}$ and the duration (the interval when it is ready to communicate) is $[t_m, \infty)$. 
A simpler primitive is $\IDLEM$: it keeps the cart stationary.
Its assumption is again $\true$, guarantee (and postcondition) is $x = x_i \wedge v = 0$, footprint is a bounding box around the fixed position.
It is interruptible by messages from other components (annotation $\instantaneous$) at any time: $D = [0,\infty)$. 

\myparagraph{Compatibility}
Now consider the constraints that ensure the joint trajectories involving the $\Cart$'s motion
and the $\Producer$'s $\WORKM$ primitive are compatible. 
First, note that the motion of the cart is non-interruptible ($\delayed$) but we assume the arm is interruptible,
satisfying the constraint in \rulename{AGcomp} that at most one motion is not interruptible.
Instead of the complexities of modeling the geometry and the dynamics of the arm, we approximate the footprint of the arm
(the geometric space it occupies) as a half-sphere centered at the base of the arm and extending upward.
Assume first that the motion primitive guarantees that the state is always within this half-sphere. 
In order for the cart and the robot arm to be compatible, we have to check that the footprints do not overlap.
Our guarantee is too weak to prove compatibility, as the footprint of the cart and the arm can intersect when the
cart is close to the arm.
Instead, we strengthen the guarantee to state that the arm can use the entire sphere when the cart is far and
as the cart comes closer the arm effector moves up to make space.

To realise this motion, the cart sends the arm the following information with $\ARRIVE$:\footnote{
	For clarity, the type in Figure~\ref{fig:basic-example} omits parameters to messages. Our type system
	uses predicated refinements to reason about parameters.
}
its current position $x_i$, its target position $x_f$, and a lower bound $t_{\mathit{ref}}$ on the time
it will take to arrive at $x_f$.
The footprint for the cart's motion can be specified by 
$\set{(\mathsf{x}, \mathsf{y}, \mathsf{z}) \mid \mathsf{x} ≤ x_0 + l/2 + (x_1 - x_0) * t / t_{\mathit{ref}} ∧ … }$.
For the producer, suppose $R_{\mathit{base}}$ is the radius of its base and $x_{\Producer}$ its $x$ coordinate. 
The footprint of the $\WORKM$ action is strengthened 
as $\set{(\mathsf{x}, \mathsf{y}, \mathsf{z}) \mid (\mathsf{z} >  \min(c t, h) \lor |\mathsf{x} - x_\Producer| ≤ R_{\mathit{base}})  ∧ … }$,
where $c ≥ h \cdot |x_0 -l - x_\Producer - R_{\mathit{base}}| / t_{\mathit{ref}}$.
Finally, for compatibility, we need to check that both $\Cart$ and $\Producer$ can adhere 
to their footprint and that the footprints are disjoint.
Disjointness is satisfied if $x_1 + l < x_\Producer - R_{\mathit{base}}$, set this as a precondition to $\MOVEM$.

\myparagraph{Environment Assumptions}
We can model a more complex cart moving in a dynamic environment
by adding a new component (participant) $\Env$ (for \emph{environment}). 
The physical variables of the environment process encode dynamic properties of the state, such as
obstacles in the workspace or external disturbances acting on the cart.
We can abstract the environment assumptions into a single motion primitive whose guarantees provide
assumptions about the environment behavior to the other components.
For example, the environment providing a bounded disturbance of magnitude to the cart's acceleration could be 
modeled as the guarantee $-1 \leq d \leq 1$ for a physical variable $d$ (for \emph{disturbance}) of $\Env$.
The cart can include this input variable in its dynamics: $\dot{v} = u + d$.
We can also model sensor and actuator errors in this way.
Likewise, we can model obstacles by a physical variable in the environment that denotes the portions
of the state space occupied by obstacles, exporting this information through the guarantees, and
using this information when the cart plans its trajectory in $\MOVEM$. 

\section{Motion Choreographies} 
\label{sec:proof-system}
In this section, we develop a multiparty session typing discipline to
prove communication safety and collision freedom.
Our session types extend usual multiparty session types by introducing
new operators and by reasoning
about joint motion in real-time. Moreover, as message exchanges can be a proxy to exchange permissions for motion primitives, 
our types have predicates as guards in order to model permissions for motion primitives.

\subsection{Global Types with Motions and Predicates}
\label{subsec:global}
We start with a choreography given as a \emph{global type}. 
The global type constrains the possible sequences of messages and controller actions that may occur in any execution of the system.
We extend \cite{ECOOP19} by 
framing and predicate refinements, together with 
separation conjunctions. 

{\em Sorts}, ranged over by $\SL,$ are used to define base types:
\[\SL \quad::=\quad \tunit \SEP \treal \SEP \text{\sf{point}}(ℝ³) \SEP \text{\sf{vector}} \SEP 
	\ldots \SEP \SL\times \SL\]
A \emph{predicate refinement} is of the form $\set{\nu : \SL\mid \pred}$, where 
$\nu$ is a \emph{value variable}
not appearing in the program, $\SL$ a sort, and 
$\pred$ a Boolean predicate.
Intuitively, a predicate refinement represents assumptions on the 
state of the sender and the communicated value to the recipient.
We write $\SL$ as abbreviation for $\set{\nu: \SL\mid\true}$ and
$\pred$ for $\set{\nu:\SL\mid \pred}$ if 
$\SL$ is not important. 

\begin{definition}[Global types]
  \label{def:global-types}%
  {\em Global types} ($\GL, \GL', ...$) are generated by the following grammar:
\[
\begin{array}{rcl}
\GL & ::=   & %
 \GLPre.\GL \;\SEP\; \ty \;\SEP\; \mu \ty.\GL \\[1mm]
  \GLPre & ::= & \motion{\dtsess}{(\pp_i{:}\ma_i)}\;\SEP\;
                 \GvtPre{\pp}{\q}{[\pred]\ell}{\set{\nu:\SL\mid\pred'}}
\; \SEP \; \GLPre.\GLPre \;\SEP\; \GLPre + \GLPre \;\SEP\; 
\GLPre \GLPAR \GLPre 
\end{array}
\]
where $\pp, \q$ range over $\PC$, $\ma_i$ range over 
abstract motion primitives, and
$\pred_i, \pred'$ range over predicates. 
$\GLPre$ corresponds to the prefix of global types where we do not allow
the recursion. 
We require that
    $\pp \neq \q$,
    $I \neq \emptyset$,
    $\freevars(\pred_i) ⊆ \pp.X ∪ \pp.W$,
    $\freevars(\pred'_i) ⊆ \pp.X ∪ \pp.W ∪ \set{\nu}$,
    and $\ell_i \neq \ell_j$ whenever $i \neq j,$ for all $i,j\in I$.
We postulate that recursion is guarded and recursive types with the same regular tree are considered equal.
We omit the predicate annotation if the predicate is $\true$ or not important.
\end{definition}

In Definition~\ref{def:global-types}, the main syntax 
follows the standard definition of global types in multiparty session types
\cite{KY13,KY2015,DGJPY15}, with the inclusion of more flexible syntax
of choreographies (\emph{separation} $\GLPre \GLPAR \GLPre$, 
\emph{sequencing} $\GLPre.\GLPre$ and summation $\GLPre + \GLPre$)
and \emph{motion} primitive $\motion{\dtsess}{(\pp_i{:}\ma_i)}$ 
extended from \cite{ECOOP19},  
and branching
$\GvtPre{\pp}{\q}{[\pred]\ell}{\set{\nu:\SL\mid\pred'}}$. 

The motion primitive explicitly declares 
a \emph{synchronisation} by AG contracts 
among all the participants $\pp_i$.
The {\em branching} type 
formalises a protocol where participant $\pp$ 
first tests if the guard $[\pred]$, then sends to $\pq$ one message with label $\ell$ 
and a value satisfying the predicate refinement $\pred'$. 
Recursion is modelled as $\mu\ty.\GL$, %
where  variable  $\ty$ is bound and guarded in $G$, e.g.,  
$\mu\ty.\ty$ is not a valid type. 
Following the standard session types, 
in $\GLPre_1 + \GLPre_2 + ... + \GLPre_n$, we assume: 
$g_i = \GvtPre{\pp}{\q}{[\pred_i]\ell_i}{\set{\nu:\SL\mid\pred'_i}}.g_i'$ 
and $\ell_i \not=\ell_j$; and similarly for $\GL$. 
By this rule, hereafter we write: 
$\Gvti{\pp}{\pq}{[\pred_i]\ell}{\pred'}{\GL}$,  
combining summations and branchings and putting $\GL$ in the tail into
each branching as the standard multiparty session types. 
$\participant{\GL}$ denotes a set of participants appeared in
$\GL$, inductively defined by $\participant{\motion{\dtsess}{(\pp_1{:}\ma_1,\ldots,\pp_k{:}\ma_k)}}
= \set{\pp_1,\ldots,\pp_k}$ and 
$\participant{\GvtPre{\pp}{\q}{\ell_i[\pred]}{\set{\nu:\SL\mid\pred'}}}=\{\pp,\q\}$
  with other homomorphic rules. 

A ``separating conjunction'' $\GLPAR$  allows us
to reason about subsets of participants.
It places two constraints on the processes and motion primitives on 
the two sides of $\GLPAR$: 
first, there should not be any communication that crosses the boundary and second, the motion
primitives executed on one side should not be coupled (through
physical inputs) with motion primitives on the other.
We call 
$\GLPre_1 \GLPAR \GLPre_2$
\textbf{\emph{fully-separated}} if $\participant{\GLPre_1} \cap \participant{\GLPre_2} = \emptyset$, there exist
$\FP_1$, $\FP_2$ with $\FP_1\cap \FP_2 = \emptyset$, and 
for every motion primitive $\judgement{\ma}{\pre,A, G,\post, \FP, D,\timestep}$ in $\GLPre_i$, we have
$A$ depends only on state variables in $\participant{\GLPre_i}$ and for all $t$, $\FP\{t/\clock\} \subseteq \FP_i\{t/\clock\}$,
for $i\in\set{1,2}$. 
$\GL$ is \textbf{\emph{fully-separated}} if each $\GLPre_1 \GLPAR \GLPre_2$ in $\GL$
is fully-separated. 

Our global type does not include an $\tend$ type.
An $\tend$ introduces an implicit global synchronisation that requires 
all components to finish exactly at the same time.
Informally, ``ending'' a program is interpreted as robots stopping their motion and staying idle forever (by forever
executing an ``idle'' primitive).
Our progress theorem will show that well-typed programs have 
infinite executions. 

\myparagraph{Data Flow Analysis on Choreographies}
Not every syntactically correct global type is meaningful.
We need to check that the AG contracts work together and that the
sends and receive work w.r.t the time taken by the different motions.
These checks can be performed as a ``dataflow analysis'' on the tree obtained
by unfolding the global type.
We start with a few definitions and properties.

\textbf{\emph{Well-scopedness.}}
Consider the unfolding of a global type $\GL$ as a tree. The leaves
of the tree are labeled with motions $\motion{\dtsess}{(\pp_i{:}\ma_i)}$ or message $\pp\to\q$, and internal nodes are labeled with the operators $.$, $\GLPAR$, or $+$.
By our assumption, we decorate each $+$ node with the sender and receiver $\pp\to\q$
below it.

We say the tree is \textbf{\emph{well-scoped}} if there is a way to label the nodes of the tree
following the rules below:
(1) The root is labeled with the set $\PC$ of all participants.
(2) If a ``.'' node labeled with $\PC'$, both its children are also labeled with $\PC'$, if possible.
(3) If a ``+'' node associated with a message exchange $\pp\to\q$ is labeled with $\PC'$, and both
$\pp,\q\in \PC'$, then each of its children are also labeled with $\PC'$, if possible.
(4) If a ``\GLPAR'' node is labeled with $\PC'$, then we label its children with $\PC_1$ and $\PC_2$
such that $\PC_1 \cap \PC_2 = \emptyset$ and $\PC_1 \cup \PC_2 = \PC'$.
(5) A leaf node $\motion{\dtsess}{\pc_i:\ma_i}$ is labeled with $\PC'$ if all $\pc_i$ are in $\PC'$, otherwise the labeling fails.
(6) A leaf node $\pp \to \q$ is labeled with $\PC'$ if both $\pp, \q$ are in $\PC'$, otherwise the labeling fails.

Such a scope labeling is unique if the global type is fully-separated and no participant is ``dropped,''
and so we can define the scope of a node in an unambiguous way.

\textbf{\emph{Unique minimal communication}.}\ 
We first define a notion of \emph{happens-before}
with \emph{events} as nodes labeled with motions $\motion{\dtsess}{(\pp_i{:}\ma_i)}$ or message exchanges $\pp\to \q$.
We define a \emph{happens before} relation on events as the smallest strict partial order such that: 
there is an edge $n:e \to n':e'$ if
there is an internal node $n^*$ in the tree labeled with $.$ and $n$ is in the left subtree of $n^*$ and
$n'$ is in the right subtree of $n^*$, and one of the following holds:
(1)~$e \equiv \pp\to\q$, $e' \equiv \pp'\to\q'$ and $\pp'$ is either
$\pp$ or $\q$; or
(2)~$\participant{e} \cap \participant{e'} \not = \emptyset$.  
We say there is an \emph{immediate happens before} edge $n\to n'$ if $n\to n'$ is in the happens before
relation but there is no $n''$ such that $n \to n''$ and $n''\to n'$.
A global type has \textbf{\emph{unique minimal communication}} after
motion 
iff every motion node $n:\motion{\dtsess}{\pp_i:\ma_i}$
has a \emph{unique} immediate happens before edge to some node $n':\pp\to\q$.

\textbf{\emph{Sender readiness}.}\ 
Consider now a type with unique minimal communication after motion and consider the unique immediate happens before edge 
$n:\motion{\dtsess}{(\pp_i{:}\ma_i)} \to n':(\pp\to\q)$.
Suppose $\pp$ is among the processes executing the motion primitives.
Since the motion primitives are assumed to be compatible, we know that at most one process is executing a non-interruptible motion,
and all the others are executing interruptible motions.
Since $\pp$ is the next process to send a message, we must ensure that it is the unique process executing the non-interruptible motion.
Compatibility ensures that the durations of all other processes are such that they are ready to receive the message from $\pp$.
We call this condition \textbf{\emph{sender readiness}}: whenever there is a communication after a motion, the sender of the communication
was the unique participant executing a non-interruptible motion; or every process in the motion was executing an interruptible motion.
On the other hand, if $\pp$ is not among the processes (this can happen when to parallel branches merge),
every participant in the motion must have been executing an interruptible motion.

\textbf{\emph{Total synchronisation}.}\ 
It is not enough that $\pp$ sends a message to only one other participant: if $\pp$ and $\q$ decide to switch to
a different motion primitive, every other participant in scope must also be informed.
This is ensured by the \textbf{\emph{total synchronisation}} between motions: 
we require that whenever there is a happens before edge between $n:\motion{\dtsess}{(\pp_i:\ma_i)}$ to $n':\motion{\dtsess}{(\pp'_i:\ma'_i)}$,
then for every $\pp_i$, there is a node where $\pp_i$ is a sender or a recipient that happens before $n'$. 

\textbf{\emph{Synchronisability}.} \ 
We call a global type is \textbf{\emph{synchronisable}} if it satisfies 
unique minimal communication, 
sender readiness, and 
total synchronisation. 
Synchronisability of a global type can be checked in time polynomial
in the size of the type by unfolding each recursive type once.  

\begin{example}[Synchronisability] \rm 
We illustrate the synchronisability condition. 
In general, a node can have multiple outgoing immediate happens before edges.
Consider, for distinct participants $\pp,\q,\pp'$, the type
${\pp}\rightarrow {\q}:{\ell}.
{\pp'}\rightarrow {\q}:{\ell'}.\GL$. 
The minimal senders are $\pp$ and $\pp'$, because these two sends cannot be uniquely ordered in an execution.
We avoid such a situation because in a process $\q$ we would not know whether to expect
a message from $\pp$ or from $\pp'$ or from both.

Note that synchronisability disallows a sequence of two motions without an intervening
message exchange.
Thus,
$\motion{\dtsess}{(\pp{:}\ma_1,\q{:}\ma_1')}.\motion{\dtsess}{(\pp{:}\ma_2,
  \q{:}\ma_2')}.\GL$ is not well-formed. 
This is because the implementation of the participants do not have any mechanism to synchronize 
when to shift from the first motion primitive to the second.

We require the total synchronisation between motions to notify every participant by some message between any change of motion primitives.
We disallow the type 
$\motion{\dtsess}{(\pp{:}\ma_{11},\pq{:}\ma_{12}, \pr{:}\ma_{13})}.{\pp}\rightarrow{\pq}:\ell.\motion{\dtsess}{(\pp{:}\ma_{21}, \pq{:}\ma_{22}, \pr{:}\ma_{23})}.\GL$.
because $\pr$ does not know when to shift from 
$\ma_{13}$ to $\ma_{23}$.

Note however, that the scoping introduced by $\GLPAR$ requires messages to be sent only to ``local'' subgroups.
The following type is fine, even though $\pp_3$ was not informed when $\pp_1$ and $\pp_2$ changed motion primitives, as it is in a different branch and there is no
happens before edge:
\begin{align*}
\left((
\begin{array}{l}
(\motion{\dtsess}{(\pp_1{:}\ma_1, \pp_2{:}\ma_2)}.
{\pp_1}\rightarrow{\pp_2}:{\ell_1}.
\motion{\dtsess}{(\pp_1{:}\ma_1',\pp_2{:}\ma_2')}) 
\end{array}  
) \GLPAR \motion{\dtsess}{\pp_3{:}\ma_3} \right)
.{\pp_1}\rightarrow{\pp_2}{:\ell_1}.{\pp_1}\rightarrow{\pp_3}{:\ell_3}.G
\end{align*}
\end{example}

\begin{example}\rm
By inspection, the motion type in Figure~\ref{fig:basic-example} is well-scoped and fully separated.
It is also synchronizable: to see this, note that for every joint motion primitive, there is at most
one motion which is non-interruptible and the participant corresponding to that motion primitive
is the unique minimal sender.
Moreover, the unique minimal sender sends messages to every participant in the scope of the separating conjunction,
thus the type is totally synchronized.
\end{example}

We are ready to define when global types are well-formed.

\begin{definition}[Well-formed global types]
  \label{def:well-formed}
A global type $\GL$ is \emph{well-formed} if it satisfies the following conditions:
\begin{enumerate}
\item \emph{Total choice:} for each branching type $\Gvti{\pp}{\q}{[\pred_i]\ell}{\pred'}{G}$, we have $\bigvee_i \pred_i$ is valid.

\item 
\emph{Well-scoped:} 
$\GL$ is well-scoped and fully-separated.

\item
\emph{Total and compatible motion:} 
For every motion type $\motion{\dtsess}{(\pp_i {:} \ma_i)}$, there exists a motion primitive
for each participant in scope 
and moreover the tuple of motion primitives $(\pp_i{:}\ma_i)$ is compatible.

\item \emph{Synchronisability:}
The global type is synchronisable.
\end{enumerate}
\end{definition}
\noindent Hereafter we assume global types are well-formed.

\begin{restatable}[Well-formedness]{proposition}{theoremWF}
\label{pro:WF}
Synchronisability is decidable in polynomial time.
Checking well-formedness of global types reduces to checking validity in the 
underlying logic of predicates in global types and motion primitives.
\end{restatable}
\begin{proof}
All causalities in global types can be checked 
when unfolding each recursive type once 
(cf.~\cite[\Sec.~3.5]{HYC2016}).
Hence synchronisability of a global type can be checked in time polynomial
in the size of the type by checking unique minimum communication,
sender readiness and total synchronisation simultaneously. 
By Definition \ref{def:well-formed}, checking well-formedness depends on 
checking validity in the 
underlying logic. 
\end{proof}

\subsection{Local Types and Projection}
\label{sec:local}

Next, we project global types to their end points against which each
end-point process is typed. 
The syntax of \emph{local types} extends local types from \cite{ECOOP19}.
Each local type represents a specification for each component. 

\begin{definition}[Local motion session types]
  \label{def:session-types}%
  The grammar of local types, ranged over by $T$, is:
  \begin{align*}
  T ::=  \ \Lmotion{\dtsess}{\ma}{T} \SEP \oplus\{[\pred_i]\tout\q{\ell_i}{\pred_i'}.T_i\}_{i\in I} 
          \SEP  \&\{\tin\pp{\ell_i}{\pred_i}.T_i\}_{i\in I} 
     \SEP  \&\{\tin\pp{\ell_i}{\pred_i}.T_i\}_{i\in I}\ \& \ \Lmotion{\dtsess}{\ma}{T} 
          \SEP \mu \ty.T \SEP \ty
  \end{align*}
  where $\ma$ ranges over motion primitives. 
  We require that $\ell_i \neq \ell_j$ whenever $i \neq j$, for all $i,j\in I$.
  We postulate that recursion is always guarded. 
  Unless otherwise noted, types are closed.
\end{definition}

Our goal is to \emph{project} a global type onto a participant to get a local type.
To define this notion formally, we first need the following definition that merges two
global types.

\begin{definition}[Merging] 
  \label{def:merge}%
We define a \emph{merging operator} $\bigsqcap$, which is a partial operation over global types, as:
\[
\small
  \textstyle%
  T_1 \mathbin{\bigsqcap} T_2 \;=\; 
  \left\{%
\text{%
    \begin{tabular}{@{\hskip 0mm}l}
      $T_1$ \hfill if $T_1 = T_2$ \hspace{2mm}
      \iftoggle{full}{$\rulename{mrg-id}$}{}  \\[2mm]%
      $\&\{\tin{\pp'}{\ell_k}{\pred_k}{.T_k}\}_{k \in I \cup J}$ \quad\quad
      \hfill if \(%
      \left\{%
      \begin{array}{l}
        T_1 = \&\{ \tin{\pp'}{\ell_i}{\pred_i}{.T_i}\}_{i \in I} \hfill \text{ and}\\%
        T_2 = \&\{ \tin{\pp'}{\ell_j}{\pred_j}{.T_j}\}_{j \in J}%
      \end{array}
      \right.%
      \)\\[3mm]
      \iftoggle{full}{\hfill $\rulename{mrg-bra1}$\\}{}
      $\&\{\tin{\pp'}{\ell_k}{\pred_k}{.T_k}\}_{k \in I \cup J} \& \ \Lmotion{\dtsess}{\ma \bigsqcap \ma'}{T'}$ \quad\quad
      \hfill if \(%
      \left\{%
      \begin{array}{l}
        T_1 = \&\{ \tin{\pp'}{\ell_i}{\pred_i}{.T_i}\}_{i \in I} \& \ \Lmotion{\dtsess}{\ma}{T'} \hfill \text{ and}\\%
        T_2 = \&\{ \tin{\pp'}{\ell_j}{\pred_j}{.T_j}\}_{j \in J} \& \ \Lmotion{\dtsess}{\ma'}{T'}%
      \end{array}
      \right.%
      \)\\[3mm]
      \iftoggle{full}{\hfill $\rulename{mrg-bra2}$\\}{}
      $\&\{\tin{\pp'}{\ell_k}{\pred_k}{.T_k}\}_{k \in I \cup J} \& \ \Lmotion{\dtsess}{\ma}{T'}$ 
      \hfill if \(%
      \left\{%
      \begin{array}{l}
        T_1 = \&\{ \tin{\pp'}{\ell_i}{\pred_i}{.T_i}\}_{i \in I} \hfill \text{ and} \\ 
        T_2 = \&\{ \tin{\pp'}{\ell_j}{\pred_j}{.T_j}\}_{j \in J} \& \ \Lmotion{\dtsess}{\ma}{T'}%
      \end{array}
      \right.%
      \) \\[3mm]
      \iftoggle{full}{\hfill $\rulename{mrg-bra3}$\\}{} 
      $\&\{\tin{\pp'}{\ell_i}{\pred_i}{.T_i}\}_{i \in I} \& \ \Lmotion{\dtsess}{\ma \bigsqcap \ma'}{T'}$ 
      \hfill if \(%
      \left\{%
      \begin{array}{l}
        T_1 = \Lmotion{\dtsess}{\ma}{T'} \hfill \text{ and} \\%
        T_2 = \&\{ \tin{\pp'}{\ell_i}{\pred_i}{.T_i}\}_{i \in I} \& \ \Lmotion{\dtsess}{\ma'}{T'}%
      \end{array}
      \right.%
      \) \\[3mm]
      \iftoggle{full}{\hfill $\rulename{mrg-bra4}$\\}{}
      $\&\{\tin{\pp'}{\ell_i}{\pred_i}{.T_i}\}_{i \in I} \& \ \Lmotion{\dtsess}{\ma}{T'}$ 
      \hfill if \(%
      \left\{%
      \begin{array}{l}
        T_1 = \Lmotion{\dtsess}{\ma}{T'} \hfill \text{ and} \\%
        T_2 = \&\{ \tin{\pp'}{\ell_j}{\pred_j}{.T_j}\}_{i \in I}%
      \end{array}
      \right. %
      \) \\[3mm]
      \if{\hfill $\rulename{mrg-bra5}$\\}\else{}\fi 
      $T_2 \mathbin{\bigsqcap} T_1$ \hfill if $T_2 \mathbin{\bigsqcap} T_1$ is defined, \\[2mm]
      undefined \hfill otherwise.
    \end{tabular}
  } %
  \right.%
\]
The merge operator for motion $\ma \bigsqcap \ma'$ returns a motion primitive $\ma''$ such that $\ma'' \refines \ma$ and $\ma'' \refines \ma'$.
We can build such $\ma''$ by taking the union of the assumptions and precondition and the intersection of the guarantees, postcondition, and footprint. 
\end{definition}

\begin{definition}[Projection] 
  \label{pro}%
  \label{def:projection}%
  The \emph{projection of a global type onto a participant $\pr$} is the largest relation $\proj{}{\pr}{}$ between global and session types such that, whenever
  $\proj{\GL}{\pr}{T}$:
\begin{enumerate}
    \item  if ${\GL}=\Gvti{\pp}{\pr}{[\pred_i]\ell}{\pred'}{\GL}$ then ${T}=\&\{ \tin\pp{\ell_i}{\pred'_i}.{T_i}\}_{i\in I}$ with $\proj{\GL_i} \pr T_i$;
    \item  if ${\GL}=\Gvti{\pr}{\q}{[\pred_i]\ell}{\pred'}{\GL}$ then ${T}= \oplus\{[\pred_i]\tout\pq{\ell_i}{\pred'_i}.{T_i}\}_{i\in I} $ and $\proj{\GL_i}\pr  T_i,$ $\forall i \!\in\! I$;
    \item  if ${\GL}=\Gvti{\pp}{\pq}{[\pred_i]\ell}{\pred'}{\GL}$ and $\pr \!\not\in\! \{\pp,\pq\}$ then there are $T_{\!i},$ $i \in I$ s.t.
    ${T} = \bigsqcap_{i \in I}\!\!{T_{\!i}}$, and
    $\proj{\GL_i\!}{\pr}{\!T_{\!i}},$ for every $i\in I$; 
  \item  if $\GL = \mu\ty. \GL'$ then $T=\mu\ty.T'$ with $\proj{\GL'}\pr T'$ if $\pr$ occurs in $\GL'$, otherwise undefined; and
\item if $\GL = \GLPre. \GL$ then $T=T' . T''$ 
with $\proj{\GLPre}{\pr}{T'}$ and 
$\proj{\GL}{\pr}{T''}$.\footnote{We abuse $T$ to denote \emph{a local type 
prefix} which is given replacing $\ty$ and $\mu\ty.T$ by $\NUL$ (as
defined for global type prefix $g$).}  

 \item 
$\proj{\motion{\dtsess}{(\pp_i {:} \ma_i)}}{\pr}{
\motion{\dtsess}{a_j}}$ with $\pr = \pp_j$;

\item  $\proj{(\GLPre_1 \GLPAR \GLPre_2)}{\pr}{\TLPre_i}$ 
and $\proj{\GLPre_i}{\pr}{\TLPre_i}$ if 
$\pr \in \participant {\GLPre_i}$ ($i\in \{1,2\}$)

\item  $\proj{\GLPre_1.\GLPre_2}{\pr}{\TLPre_1.\TLPre_2}$ 
with $\proj{\GLPre_i}{\pr}{\TLPre_i}$ ($i=1,2$)

\end{enumerate}
We omit the cases for recursions and selections 
(defined as \cite[Section~3]{SY2019}). The branching prefix 
is defined as the branching in (1-3). 
\end{definition}

\begin{example}[Projection of Figure~\ref{fig:basic-example}]\rm 
For the example from \Sec~\ref{sec:example}, 
the local type of the cart is:
\begin{align*}
\small
\mu \ty.\left(
\begin{array}{l}
	 \ (\Green!\ARRIVE(t).\Red!\FREE.\\
         \ \ \motion{\dtsess}{\MOVEM(\Green)}.\Green!\READY.
        \motion{\dtsess}{\IDLEM}.\Green?\OK.\motion{\dtsess}{\MOVEM}) \\
        \& \ (\Red!\ARRIVE(t).\Green!\FREE. \ldots \text{symmetric} \ldots)\ 
\end{array}
\right).\ty
\end{align*}
\end{example}

\subsection{Subtyping and Typing Processes}
The local types are a specification for the processes and there is some freedom to implement these specifications.
The subtyping relation helps bridge the gap between the specification and the implementation.

\begin{definition}[Subtyping]
\label{def:subt}%
\label{tab:sync:subt}%
{\em Subtyping} $\subt$ is the largest relation between session types coinductively defined by rules in Figure~\ref{fig:subtyping}.
\end{definition}

\begin{figure}[t]
\[
\small
\begin{array}{@{}c@{}}
\cinferrule[\rulename{sub-motion}]{
    \ma \refines \ma' \\
    T \subt T'
}{
  \Lmotion{\dtsess}{\ma}{T} \subt  \Lmotion{\dtsess}{\ma'}{T'}
}\quad
\cinferrule[\rulename{sub-out}]{
 \forall i\in I.\ \pred_i ⇒ \pred''_i ∧ \pred'_i ⇒ \pred'''_i  ∧ T_i \subt T'_i  
}{ \oplus\{[\pred_i]\tout\pp{\ell_i}{\pred'_i}.T_i\}_{i\in I} \subt
   \oplus\{[\pred''_i]\tout\pp{\ell_i}{\pred_i'''}.T_i'\}_{i\in I\cup J}
}
\\[5mm]

\cinferrule[\rulename{sub-in1}]{
    \forall i\in I.\ \pred_i' ⇒ \pred_i ∧ T_i \subt T_i' \\ 
    \Lmotion{\dtsess}{\ma}{T} \subt \Lmotion{\dtsess}{\ma'}{T'}
}{
  \&\{\tin\pp{\ell_i}{\pred_i}.T_i\}_{i\in I\cup J} \ \& \ \Lmotion{\dtsess}{\ma}{T}
    \subt 
   \&\{\tin\pp{\ell_i}{\pred_i'}.T_i'\}_{i\in I}\ \& \ \Lmotion{\dtsess}{\ma'}{T'}
}
\\[5mm]

\cinferrule[\rulename{sub-in2}]{
\forall i\in I.\ \pred_i' ⇒ \pred_i ∧ T_i \subt T_i'}{
  \&\{\tin\pp{\ell_i}{\pred_i}.T_i\}_{i\in I\cup J}\ \& \ \Lmotion{\dtsess}{\ma}{T}
    \subt 
   \&\{\tin\pp{\ell_i}{\pred_i'}.T_i'\}_{i\in I}
}
\quad 
\cinferrule[\rulename{sub-in3}]{
    \Lmotion{\dtsess}{\ma}{T} \subt \Lmotion{\dtsess}{\ma'}{T'}
}{
  \&\{\tin\pp{\ell_i}{\pred_i}.T_i\}_{i\in I}\ \& \ \Lmotion{\dtsess}{\ma}{T}
  \subt 
  \Lmotion{\dtsess}{\ma'}{T'}
}
\end{array}
\]
\caption{Subtyping rules\label{fig:subtyping}}
\end{figure}

A subtype has fewer requirements and provide stronger guarantees.
The subtyping of motion primitives (\rulename{sub-motion}) 
allows replacing an abstract motion primitive
by a concrete one if $\ma$ refines $\ma'$ and $T \subt T'$. 
For internal choice and sending (\rulename{sub-out}), the subtyping ensure the all the messages along with the associated predicate are allowed by the supertype.
Predicate refinements are converted to logical implication.

For the external choice and message reception (\rulename{sub-in1,2,3}), subtyping makes sure the process reacts properly to the messages expected by the super type.
A subtype can have cases for more messages but they don't matter.
The subtype guarantees that the process will never receive theses messages.
On the other hand, the sender of the messages has to be the same.
We enforce this directly in the syntax of the programming language and the type system.
It may seem that (\rulename{sub-in3}) introduces a new sender in the subtype 
but this rule is correct because, in our synchronous model, sending is blocking
and cannot be delayed.
Therefore, if the supertype only contains a motion, we know for a fact that no messages can
arrive unexpectedly in the subtype. 

Our subtyping conditions generalise those of motion session types of \citet{ECOOP19} in multiple ways.
First, we allow refinement of motion primitives (\rulename{sub-motion}) as a way to abstract trajectories.
In \cite{ECOOP19}, the actions are abstract symbols and are fixed statically. 
Second, the rules for communication must check predicate refinements---in \cite{ECOOP19}, global
types did not have refinements.

The final step in our workflow checks that the process of each
participant in a program implements its local type using the typing
rules from Table~\ref{tab:sync:typing}.  
Our typing rules additionally
maintain a logical context to deal with the predicate refinements
\cite{LiquidTypes}.

We write $\der{\Gamma,\Sigma}{S}{T}$ to indicate the statement $S$ has
type $T$ 
under the variable context $Γ$ and the logical context $Σ$.
$\Gamma$ is the usual typing context; $\Sigma$ is a formula characterising what is known about the state of the 
system when a process is about to execute. 
The typing rules are shown in Table~\ref{tab:sync:typing}. 
\rulename{t-motion} considers $\Sigma$ derives the pre-condition 
of motion $a$, while $Q$ guarantees its post-condition.
\rulename{t-out} assumes $\Sigma$ of $P$ derives a conjunction of the guard and
refined predicates declared in types. \rulename{t-choice1} is its
dual and \rulename{t-choice2} includes the default motion branch. 
\rulename{t-cond} is similar with the usual conditional proof
rule. 
\rulename{t-sess} combines well-typed participants each of which
follows a projection of some global type $\GL$ given assumptions over the initial state $\bigwedge_{i\in I}𝓟_i$ of each participant.
We requires the initial state to be collision free as the global types only maintains collision freedom.

\begin{table}[t]
\caption{\label{tab:sync:typing} Typing rules for motion processes.}
\centerline{$
\small
\begin{array}{c}
\inferrule[\rulename{t-sub}]{
    \der{\Gamma,\Sigma}{\PP}T \quad T\subt T'
}{
    \der{\Gamma,\Sigma}{\PP}T'
}
\qquad
\inferrule[\rulename{t-rec}]{\der{\Gamma\cup\set{\procvar:T},\true}{\PP}{T}}{\der{\Gamma,\Sigma}{\mu \procvar.\PP}T}
\qquad
\inferrule[\rulename{t-motion}]{
\der{\Gamma,a.\post}{{Q}}{T}\\
\Sigma \Rightarrow a.\pre\\
}{\der{\Gamma,\Sigma}{\motion{dt}{a}.{Q}}{\Lmotion{dt}{a}{T}}}
\\\\
\inferrule[\rulename{t-out}]{
    \Sigma \Rightarrow \pred\wedge \pred'\set{e/\nu}\\
    \der{\Gamma,\Sigma}{\e}{S}\\
    \der{\Gamma,\Sigma}{P}T
}{
    \der{\Gamma,\Sigma}{{\q}!{\ell(\e)}.{P}}{[\pred]\q!\ell(\set{\nu: S\mid\pred'}).T}
}
\quad 
\inferrule[\rulename{t-choice1}]{
    \forall i\in I \quad \der{\Gamma\cup\set{x_i:\S_i},\Sigma\wedge \pred_i\set{x_i/\nu}}{\PP_i}\T_i
}{
    \der{\Gamma,\Sigma}{\sum\limits_{i\in I}{\procin  \q   {\ell_i(x_i)} \PP_i}}{\&\{\q?\ell_i(\pred_i).\T_i}\}_{i\in I}
}
\\\\
\rulename{t-choice2}
\inferrule%
    {
    \forall i\in I \quad \der{\Gamma\cup \set{x_i:\S_i},\Sigma\wedge \pred_i\set{x_i/\nu}}{\PP_i}\T_i \quad
    \der{\Gamma,\Sigma}{\motion{dt}{a}.{Q}}{\T}
}{
    \der{\Gamma,\Sigma}{\sum\limits_{i\in I}{\procin  \q   {\ell_i(x_i)} \PP_i}+\motion{dt}{a}.{Q}}{\&\{\q?\ell_i(\set{\nu:S\mid\pred_i}).T_i}\}_{i\in I}\ \&\ T
}
\\\\
\rulename{t-cond}
\inferrule%
    {
    \der\Gamma\e\tbool \
    \exists k \in I \
    \Sigma \wedge \e \Rightarrow \pred_k\
    \der{\Gamma,\Sigma\wedge \e}{\PP_1}{T_k} \
    \der{\Gamma, \Sigma\wedge \lnot \e}{\PP_2}{\oplus\{[\pred_i]T_i\}_{i \in I \setminus \{k\}}}
}{
    \der{\Gamma,\Sigma}{{\cond{\e}{\PP_1}{\PP_2}}}{ \oplus\{[\pred_i]T_i\}_{i\in I}}
}
    \\\\
    \rulename{t-sess}
    \inferrule%
        {
    \forall i\in I\quad\der{∅,𝓟_i}{\PP_i}{\projt \GL {{\pp_i}}} \\
    \participant \GL = \PC \\
    \forall j\in I.\ j≠i ⇒ \geom{\pp_i}(𝓟_i)\cap\geom{\pp_j}(𝓟_j) = \emptyset
}{
    \der{\bigwedge\limits_{i\in I}𝓟_i}{\prod\limits_{i\in I}\pa {\pp_i}\PP_i}\GL
}
\end{array}
$}
\vspace*{-1em}
\end{table}

\subsection{Soundness}

The soundness of multiparty sessions is shown, using subject reduction (typed sessions reduce
to typed sessions) and progress.
In order to state soundness, 
we need to formalise how global types are reduced 
when local session types evolve. 
To define the consumption of global types, as done for processes in \Sec~\ref{sec:opsem}, 
we extend the syntax of global types to allow \emph{partial consumption} 
of motion types.
In addition, we extend $\GLPre$, annotating each motion type with its unique minimal sender,  
$\motion{\dtsess}{(\pp_i{:}\ma_i)\AT t}^{\pp}$ for $t\geq 0$.  

\begin{definition}[Global types consumption and reduction]
\label{def:global-type-consumption}%
\label{def:global-type-reduction}%
The {\em consumption} of the communication $\pp\lts{\ell}\q$ or motion $\motion{\dtsess}{a}[t_1,t_2]$ %
for the global type $\GL$ 
(notation $\redG \GL\pp\ell\q$ and 
$\redGM{\GL}{\motion{\dtsess}{a}[t_1,t_2]}$) %
is the global type %
defined (up to unfolding of recursive types) using the following rules: 
\begin{enumerate}
\item\label{Item:GC1}  $\redG{\big(\GvtPair \pp\q {{\ell}_i.{\GL_i}}\big)}\pp\ell\q
  = \GL_{k}$
            if there exists $k \in I$ with $\ell = \ell_k$;
\item\label{Item:GC1}
  $\redG{\big(\GvtPair \pr\ps {{\ell}_i.{\GL_i}}\big)}\pp\ell\q
   =\GvtPair \pr \ps {{\ell}_i.{(\redG{\GL_i}{\pp}{\ell}{\q})}}$
if $\{\pr, \ps\} \cap \{\pp, \q\} = \emptyset$;
\item\label{Item:GC3}
  $\redG{(\GLPre₁ \ast \GLPre₂)}\pp\ell\q  =  (\GLPre₁' \ast \GLPre₂)$ if $\set{\pp,\q}\subseteq \participant{\GLPre_1}$ and $\redG{\GLPre₁}\pp\ell\q = \GLPre₁'$ and 
  symmetrically if $\set{\pp,\q}\subseteq\participant{\GLPre_2}$;

\item\label{Item:GC41}  
  $\redGM{\motion{\dtsess}{(\pp_i{:}\ma_i)}}
     {\motion{\dtsess}{(\pp_i{:}\ma_i[t_i;t'_i])}} =
      \motion{\dtsess}{(\pp_i{:}\ma_i)\AT 0}^\pp$ 
      where $\pp$ is the unique minimal sender in $\GL$.
\item\label{Item:GC4}
  $\redGM{\motion{\dtsess}{(\pp_i{:}\ma_i)\AT t}^\pp}
     {\motion{\dtsess}{(\pp_i{:}\ma_i[t_i;t'_i])}} =
      \motion{\dtsess}{(\pp_i{:}\ma_i)\AT t'}^\pp$
            if $t' > t$ and for all $i$, $t'_i-t_i = t'-t$  and $t' \leq \lceil\ma_i.D\rceil$ where $\lceil D\rceil$ denotes the upper
	    bound of the interval $D$;
\item\label{Item:GC5}
  $\redGM{(\GLPre₁ \GLPAR \GLPre₂)}
{\motion{\dtsess}{(\pp_i{:}\ma_i[t_i;t'_i])_{\pp_i∈
      \participant{\GLPre_1} \uplus \participant{\GLPre_2}}}}  =  
(\GLPre₁' \GLPAR \GLPre₂')$ 
            if
            $\redGM{\GLPre_j}{\motion{\dtsess}{(\pp_i:\ma_i[t_i;t'_i])_{\pp_i∈\participant{\GLPre_j}}}} = \GLPre_j'$ ($j=1,2$)

\item\label{Item:GC6} 
$\redGM{\GLPre}{x}= \GLPre'$ then 
$\redGM{\GLPre.\GLPre_1}{x}= \GLPre'.\GLPre_1$ and 
$\redGM{\GLPre.\GL}{x}= \GLPre'.\GL$. 
\item\label{Item:GC7} 
$\redGM{\GLPre}{x} = \GLPre'$ 
if $∃\GLPre''.\exists \ASENDPC.\, 
\redastGequiv{\emptyset\rhd \GLPre}{\ASENDPC \rhd \GLPre''}$ and 
$\redGM{\GLPre''}{x} = \GLPre'$ 
where $C$ is a reduction 
context of the prefix defined as: $C = C.g \SEP   C\GLPAR g \SEP   g
\GLPAR C \SEP [\ ]$ and $\ASENDPC$ is a set of enabled senders
such that $\ASENDPC\subseteq \PC$ and  
\begin{enumerate}
\item\label{Item:GC7A} $\redGequiv{\ASENDPC\rhd C[\motion{\dtsess}{(\pp_i{:}\ma_i)\AT t}^\pp]}{
  \ASENDPC \cup\set{\pp} \rhd C }$ 
if $∃ i.~\pp_i = \pp \wedge t ∈ D_i \wedge \timestep_i =
\delayed$.

\item\label{Item:GC7B} $\redGequiv{\ASENDPC\rhd C[\motion{\dtsess}{(\pp_i{:}\ma_i)\AT
      t}^\pp]}{\ASENDPC \rhd C}$ with $\pp\in\ASENDPC$ and  
for all $i$, $\pp_i\not = \pp$ and $t\in D_i$.
\end{enumerate}

%
\end{enumerate}

%
%
%
%
%
%
%
%
%
%
%
%
%
In (6) and (7), we write $x$ as either
${\pp}\rightarrow{\q}$ or $\motion{\dtsess}{a}[t_1,t_2]$.
The \emph{reduction of global types} %
is the smallest pre-order relation closed under the rule:
\(
\GL\RedG\redG \GL\pp{\ell}\q \quad \text{and} \quad \GL\RedG\redGM{\GL}{\motion{\dtsess}{\pp_i{:}\ma_i[t_i,t'_i]}}
\).
\end{definition}

\noindent Note that the above reduction preserves 
well-formedness of global types.
We can now state the main results.
Our typing system 
is based on one in \cite{DBLP:journals/jlp/GhilezanJPSY19}
with motion primitives and refinements. 
Since the global type reduction in 
Definition~\ref{def:global-type-consumption}
is only related to communications or synchronisation
of motions but not refinements, 
(1) when the two processes $\pp$ and $\q$ synchronise by a
communication, its global type can always reduce; or
(2) when all processes synchronise by the same motion
action, their corresponding global type can always be
consumed.

Thus process behaviours correspond to reductions of global types, which
is formulated as in the following theorem.
Recall $\xrightarrow{\alpha}$ denotes any transition relation or reduction. 

\begin{restatable}[Subject Reduction]{theorem}{theoremSR}
\label{th:SR}
  Let  $\N$ be a multiparty session, $\GL$ be a well-formed global type, and a physical state $I$ such that $\der{I}{\N}{\GL}$.
  For all $\N'$, if $\N \xrightarrow{\alpha} \N'$, then
  $\der{I'}{\N'}{{\GL}'}$ for some ${\GL}'$, $I'$ such that ${\GL}\RedG{\GL}'$.
  Thus, if $\N\equiv \N'$ or $\N \xrightarrow{\alpha_1}\cdots\xrightarrow{\alpha_n} \N'$, then $\der{}{\N'}{{\GL}'}$ for some ${\GL}'$ such that ${\GL}\RedG{\GL}'$.
\end{restatable}

\iftoggle{full}{
See Appendix~\ref{app:sr} for the detailed proofs.
}{
See \cite{fullversion} for the detailed proofs.
}

Below we state the two progress properties as already explained
in \Sec~\ref{sec:example}. 
The first progress
is related to communications, while the second one gurantees
the typed multiparty session processes are always collision free.
As a consequence, if $\der{I}\N{\GL}$ for a well-formed type,
then $\N$ does not get stuck and can always reduce.

\begin{restatable}[Progress]{theorem}{theoremPROG}
\label{th:progress}
Let $\N$ be a multiparty session, $\GL$ be a well-formed type, and a physical state $I$. 
If $\der{I}\N{\GL}$ then:
\begin{enumerate}
\item 
{\sc (Communication and Motion Progress)}
there is $M'$ such that $M \red M'$ and
\item 
{\sc (Collision-free Progress)} If $M \red^\ast M'$, then 
$M'$ is also collision free.
\end{enumerate}
\end{restatable}

\begin{proof} ({\em Outline --
\iftoggle{full}{
see Appendix~\ref{app:sr} for the detailed proofs
}{
see \cite{fullversion} for the detailed proofs.
}
})

{\em (Communication Progress)} The proof is divided into two cases:
  (a) the guard appearing in the branching type followed by
  (b) a message exchange between two parties.
Case (a) follows from the fact that there is always at least one guard
  in a choice whose predicate is evaluated to true (Definition \ref{def:well-formed}(1)).
  The typing rules ensure that the predicates are satisfied when a
  message is sent.
Case (b) is proved using
  Theorem~\ref{th:SR} with
  Definition \ref{def:well-formed}(4) since the unique minimal
  sender in $G$ can always send a message.

{\em (Motion Progress)} 
For executing motions, Definition~\ref{def:well-formed}(3) and Theorem~\ref{th:motion-compat}
ensure local trajectories can be composed into global trajectories. 

{\em (Collision-free Progress)} 
By induction.
Under $I$ and by \rulename{t-sess}, $\N$ is initially collision free.
For the communication, note that it does not change the physical state, and therefore, does not impact the geometric footprint used by a process.
Hence the case $\redG \GL\pp\ell\q$ is straightforward.
For motion actions, Definition~\ref{def:well-formed}(3) and Theorem~\ref{th:motion-compat} ensures collision freedom through execution steps.
For a collision free program, collision freedom of the next state follows from compatibility of motion. 
\end{proof}
%

%
%
%
%
%
%
%
%
%
%
%
%
%
%
%
%
%
%
%
%
%
%
%

%
%
%

\subsection{A More Complex Coordination Example with Producer}
\label{subsec:large}

We now discuss an extended version of the coordination example from \Sec~\ref{sec:example}
that we shall implement on physical robots in \Sec~\ref{sec:eval}.
In addition to cart and two arms, we add a producer robot.
The $\Producer$ generates green or red parts and places them on to the cart.
The cart $\Cart$ carries the part to the two robot arms as before.
We shall use this extended version as the basis for one of our case studies.

The coordination protocol is as follows.
First, the protocol starts by syncing all participants.
Then, the cart moves to the producer after announcing a message $\ARRIVE$.
While the cart moves to the producer, the two consumer robots can work independently.
When the cart is at the producer, it synchronises through a message $\READY$, and idles
while the producer places an object on to it.
When the producer is done, it synchronises with the cart, and tells it whether the object is green or red.
Based on this information, the cart tells one of the consumers that it is arriving with a part and
tells the other consumer that it is free to work.
Subsequently, the cart moves to the appropriate consumer to deposit the part, while the other consumer
as well as the producer is free to continue their work.
When the robot is at the consumer, it syncs through a message, and idles until the part is taken off.
After this, the protocol starts again as the cart makes its journey back to the producer, while the consumers
independently continue their work.

\begin{figure*}[t]
{\small
\begin{align*}
\text{Initial Phase} \ \equiv\ 
\begin{array}{l}
\Cart \to \Producer: \ARRIVE. \Cart\to \Green: \START.\Cart\to\Red:\START.\\
  \left( 
 \begin{array}{ccc}
 \left(
 \begin{array}{l}
 (\motion{\dtsess}{\tuple{\Cart:\MOVEM^\delayed(\Producer), \Producer:\WORKM^\instantaneous}}^\Cart.\\
  \Cart\to\Producer:\READY.\motion{\dtsess}{\Cart:\IDLEM^\instantaneous,\Producer:\PLACEM^\delayed})^\Producer.
 \end{array} \right) \\
 \quad\quad ~\GLPAR~ \motion{\dtsess}{\Red: \WORKM^\instantaneous}^\Producer 
 ~\GLPAR~ \motion{\dtsess}{\Green: \WORKM^\instantaneous}^\Producer 
 \end{array} 
 \right)
\end{array}
\end{align*}
\begin{align*}
\begin{array}{l}
\text{Process Green Item} \ \equiv\ \\
  \quad \quad \left( \begin{array}{cc}
               \motion{\dtsess}{\Producer:\WORKM^\instantaneous}^\Cart) ~\GLPAR~ & 
               \left( 
               \begin{array}{l}  
                    \Cart\to\Green:\ARRIVE.\Cart\to\Red:\FREE.  \\
                    \left(\begin{array}{cc}
     \text{as in the green box in Figure~\ref{fig:basic-example}}
                      & ~\GLPAR~ \motion{\dtsess}{\Red:\WORKM^\instantaneous}^\Cart
                      \end{array}\right)
               \end{array} \right)
          \end{array} \right) \\
\end{array}
\end{align*}
\begin{align*}
\text{Global Type} \ \equiv \ 
\mu \ty. & \langle \text{Initial Phase}\rangle.
 \left(
 \begin{array}{l}
 (\Producer  \to\Cart:\GREEN. \langle \text{Process Green Item} \rangle)\\
 + 
  (\Producer\to\Cart:\RED. \langle \text{Process Red Item} \rangle)
 \end{array} 
\right) .\ty
\end{align*}
}
\vspace{-3mm}
\caption{Motion session type annotated with minimal senders. For readability, we have broken the type into sub-parts:
the initial phase, and processing items (the type for processing red items is symmetric and omitted)\label{fig:annotated-global-session-type}}
\end{figure*}

\myparagraph{Global Type}
The global type for the example extends one in Figure~\ref{fig:basic-example}, and is shown
in 
Figure~\ref{fig:annotated-global-session-type}. 
It can be seen that the global type has total choice (trivially), and is well-scoped and synchronisable. 
The motion primitive specifications are omitted;
we ensure in our evaluation that the motion primitives are compatible and the type is fully separated using calls
to the SMT solver dReal.

\myparagraph{Processes}
The processes in this example extend the processes from \Sec~\ref{sec:example} but the $\Red$ and $\Green$
processes are as before:
\[
\begin{array}{ll}
\Cart:
  &
\begin{array}{l}\small
\mu \procvar. \Producer!\ARRIVE.\Red!\START.\Green!\START.\MOVEM(\text{co-ord of }\Producer).\Producer!\READY.\IDLEM.\\
 \ (\Producer?\RED.\Red!\ARRIVE.\Green!\FREE.\MOVEM(\text{co-ord of }\Red).\\
 \ \ \ \quad\quad \Red!\READY.\IDLEM.\Red?\OK.\procvar\\
 \ + \Producer?\GREEN.\textrm{ symmetrically for $\Green$ })
\end{array}
\\[3mm]
\Producer: &
\begin{array}{l}\small
\mu \procvar.\Cart?\ARRIVE.\WORKM.\Cart?\READY.\PLACEM. (\Cart!\GREEN + \Cart!\RED).\WORKM.\procvar
\end{array}    
\end{array}  
\]
%

\myparagraph{Local Types}
The local types for the components are:
\begin{align*}
\small
\begin{array}{c|c}
\small
\begin{array}{rl}
\Cart:\!\!\!\!\!\!\\
\mu \ty. & \Producer!\ARRIVE(t).\Red!\START.\Green!\START.\\
        & \motion{\dtsess}{\MOVEM(\Producer)}.\Producer!\READY.\motion{\dtsess}{\IDLEM}.\\
	& (\ (\Producer?\GREEN.\Green!\ARRIVE(t).\Red!\FREE.\\
        & \ \ \motion{\dtsess}{\MOVEM(\Green)}.\Green!\READY.\\
        & \ \ \motion{\dtsess}{\IDLEM}.\Green?\OK) \\
        & \& \ (\Producer?\RED. \ldots \text{symmetric} \ldots)\ ).\ty
\end{array}
&
\begin{array}{l}
\Producer:\!\!\!\!\!\!\!\!\!\!\!\!\!\!\!\!\!\!\!\!\!\!\\
\quad\mu \ty.  \Cart?\ARRIVE(t).\motion{\dtsess}{\WORKM}.\Cart?\READY.\motion{\dtsess}{\PLACEM}.\\
	 \quad\quad (\Cart!\GREEN.\motion{\dtsess}{\WORKM}.\ty\; \oplus \; \Cart!\RED.\motion{\dtsess}{\WORKM}.\ty)\\
\Red,\Green:\!\!\!\!\!\!\!\!\!\!\!\!\!\!\!\!\!\!\!\!\!\!\!\!\!\!\\
\quad\mu \ty.  (\Cart?\START.\motion{\dtsess}{\WORKM}.\\
          \quad\quad(\ (\Cart?\ARRIVE(t).\motion{\dtsess}{\WORKM}.\\
          \quad\quad      \Cart?\READY.\motion{\dtsess}{\PICKM}.\Cart!\OK) \\
          \quad\quad\& \; (\Cart?\FREE.\motion{\dtsess}{\WORKM}) \ ).\ty
\end{array}
\end{array}
\end{align*}
We can show that all the processes type check.
From the soundness theorems, we conclude that the example 
satisfies communication safety, motion compatibility, and collision freedom.

\section{Implementation and Case Study}
\label{sec:eval}
\myparagraph{Implementation}
Our implementation has two parts.
The first part takes a program, a specification for each motion primitive, and a global type and checks
that the type is well-formed and that each process satisfies its local type.
The second part implements the program on top of
the Robotic Operating System (ROS)~\cite{ROS}, a software ecosystem for robots.
We reuse the infrastructure of motion session types \cite{ECOOP19}; for example, we write programs in \PGCD syntax \cite{PGCD}.
The verification infrastructure is about 4000 lines of Python code, excluding the solver.
The code is available at \url{https://github.com/MPI-SWS/pgcd} and instructions to run these experiments are located in the \texttt{oopsla20\_artifact} branch.

Internally, we represent programs and global types as state machines \cite{DBLP:conf/esop/DenielouY12}, and implement the dataflow analysis
on this representation.
Additionally, we specify motion primitives in the local co-ordinate system for each robot and automatically
perform frame transformations between two robots.
The core of ROS is a publish-subscribe messaging system; we extend the messaging layer to implement a synchronous message-passing layer.
Our specifications contain predicates with nonlinear arithmetic, for example, to represent footprints of components as spheres.
Obstacle are represented as passive components, i.e., components which have a physical footprint but do not execute program.
Such components can interact with normal components through input and state variables and their are considered when checking the absence of collision.
On the other hand, the obstacles are excluded from the checks related to communication, e.g., synchronisability.
We use the dReal4 SMT solver \cite{dReal} to discharge validity queries.
The running times are obtained on a Intel i7-7700K CPU at $4.2$GHz and dReal4 running in parallel on 6 cores.

\myparagraph{Tests}
We evaluate our system on two benchmarks and a more complex case study:
(1) We test scalability of the verification using micro-benchmarks. 
(2) We compare our approach with previous approaches
on a set of robotic coordination scenarios from the literature \cite{PGCD,ECOOP19}, and
(3) As a large case study, we verify and implement 
a complex choreography example based on a variation of the example from \Sec~\ref{subsec:large}.

\begin{wrapfigure}{r}{5cm}
\centering
\includegraphics[width=45mm]{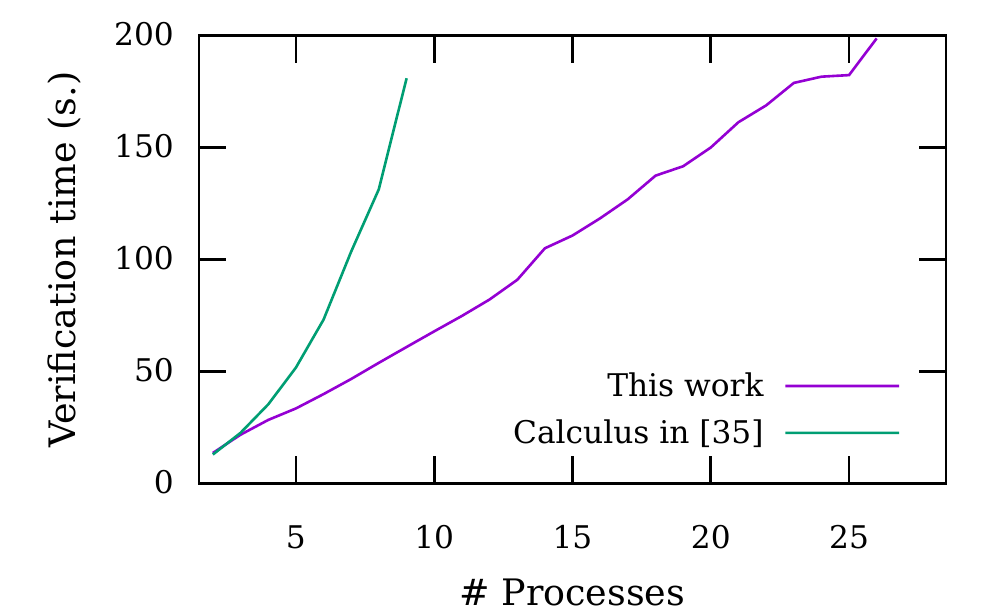}
\vspace{-2ex}
\caption{Parallel lanes micro-benchmarks}
\label{fig:lanes}
\vspace{-2ex}
\end{wrapfigure}

\myparagraph{(1) Micro-benchmarks}
The micro-benchmarks comprise a parametric family of examples that highlight the advantage 
of our motion session calculus, specifically the separating conjunction, over previous typing
approaches \cite{ECOOP19}.
The scenario consists of carts moving back and forth along parallel trajectories, and is
parameterised by the number of carts.
Fig.~\ref{fig:lanes} shows the verification times for this example in the two systems.
We start with 2 robots and increase the number of robots until we reach a $200$ s. timeout for the verification.
The previous calculus is discrete-time and does not separate independent components.
Thus, the type system synchronises every process and, therefore, generates complex collision checks which quickly overwhelms the verifier.
Our global types use the separating conjunction to split the specifications into independent pieces
and the verifcation time increases linearly in the number of processes.

\myparagraph{Setup for (2) and (3)}
For examples used previously in related works \cite{PGCD,ECOOP19}, we write specifications
using our motion session types, taking advantage of the modularity of the type system. 
For the new case study (part (3)), we implement a variation of the example presented in \Sec~\ref{sec:example},
where we decouple the producer placing an object and a sensor that determines if the object is green or red.
Thus, after the producer places an object on the cart, the cart first moves to the sensor, communicates with
the sensor to get the color, and then delivers the object as described in the protocol.
We filmed our experiments and a short video can be seen in the \textbf{\emph{supplementary materials}}.
We use three robots: a custom-built robot arm modified from an open-source arm, a commercial manipulator,
and a mobile cart, as shown in Figure~\ref{fig:robots-xp}.
We use placeholders for the producer arm and color sensor, as we do not have additional hardware (and arms are expensive).
The robots are built with a mix of off-the-self parts and 3D printed parts as described below.

\textbf{\emph{Arm}.}\ 
    The arm is a modified BCN3D MOVEO ({\url{https://github.com/BCN3D/BCN3D-Moveo}}).
    The upper arm section is shortened to make it lighter and easier to mount on a mobile cart. 
    It has three degrees of freedom.
    The motion primitives consist of moving between poses and opening/closing the gripper.
    The motion is a straight line in the configuration space (angles of the joints) which corresponds to curves in physical space.

\textbf{\emph{Panda Arm.}} \ 
    The Panda arm by Franka Emika ({\url{https://www.franka.de/technology}}) is a seven degrees of freedom commercial manipulator platform.
    It is controlled similarly to the MOVEO arm but with a more complex configuration space.
	The Panda arm has a closed-loop controller and can get to a pose with an error less than 0.1mm.
	The controller also comes with collision detection using feedback from the motors.
	We command this arm using a library of motion primitives provided by the manufacturer.

\textbf{\emph{Mobile Carts}.} \ 
    There are two carts.
    Both are omnidirectional driving platforms.
    One uses mecanum wheels and the other omniwheels to get three degrees of freedom 
    (two in translation, one in rotation) and can move between any two positions on a flat ground.
    The advantage of using such wheels is that all the three degrees of freedom are controllable and movement does not require complex planning.
    Its basic motion primitives are moving in the direction of the wheels, moving perpendicularly to the direction of the wheels, and rotating around its center.

The arm and cart are equipped with RaspberryPi 3 model B to run their
processes and use stepper motors.
This enables a control of the joints and wheels with less than 1.8 degree of error.
However, these robots do not have global feedback on their position and keep track of their state using \emph{dead reckoning}.
This can be a challenge for the cart, which keeps accummulating error over time.
As our example is a loop, we manually reset the cart's position when it gets back to the producer after delivering the object.
(A more realistic implementation would use feedback control, but we omit this because control algorithms are a somewhat orthogonal
concern.)
For the producer and the color sorter, we run the processes but have placeholders in the physical world and realise the corresponding action manually.
All computers run Ubuntu 18.04 with ROS 2 Dashing Diademata.

Table~\ref{tbl:xp1} shows the specification for the robots and their motions.
As the two carts share most of their specification, we group them together.
The specification includes the geometrical description of the robots and the motion primitives.

\myparagraph{(2) Revisiting the Examples from \PGCD}
As we build on top of \PGCD, we can use the examples used to evaluate that system.
We take 4 examples, for which we compare the specification effort 
in the previous calculus \cite{ECOOP19} to our new calculus.
As the two calculi have different models for the time and synchronisation, we 
made some minor adaptations such that the same programs can be described by the two calculi.
The old calculus requires motion primitives to have fixed duration and does not support interruptible motions.
Furthermore, there is no ``$\GLPAR$'' in the old calculus so all the motions are always synchronised.
In our new specification, we take advantage of our richer calculus to better decompose the protocol.
We use the following scenarios.
\begin{description}
\item[Fetch.]
    in this experiment, the Moveo arm is attached on top of a cart.
    The goal is to get an object.
    The cart moves toward the object until the object if within the arm's reach.
    The arm grabs the object and the cart goes back to its initial position.
\item[Handover.]
    This experiment is a variation of the previous one.
    There are two carts instead of one and the object to fetch is on top of the new cart.
    The two carts meet before the arm takes an object placed on top of the second cart and then, both go back to their initial positions.
\item[Twist and Turn.]
    In this experiment, the two carts start in front of the each other.
    The arm takes an object from the small cart.
    Then all three robots move simultaneously.

    The cart carrying the arm rotates in place, the other cart describes a curve around the first cart, and arm moves 
    from one side of the cart to the other side.
    At the end, the arm puts the object on the carrier.
\item[Underpass.]
    In this experiment, the arm and the cart cooperate to go under an obstacle.
    First, the cart goes toward the arm, which takes the object from the cart. 
    The cart goes around the arm passing under an obstacle. 
    Finally, the arm puts the object back on the cart on the other side of the obstacle.
\end{description}

Table~\ref{tbl:xp-revisit} reports the size of the examples and their specifications, as well as the verification time.
The verification time is broken down into syntactic well-formedness checks on the choreography, 
the motion compatibility checks for the trajectories, and typing the processes w.r.t.\ local types.
The motion compatibility checks dominate the verification time.
When we compare our new calculus to the previous approach, we can 
observe that the global types are slightly larger but the verification time can be significantly smaller.
The increase in specification size comes with the addition of new Assume-Guarantee contracts when $\GLPAR$ is used.

The verification time from the existing specification are higher than the time reported in the previously published results.
When implementing our new calculus, we found and fixed some bugs in the verifier code from \PGCD.
Those bugs resulted in incomplete collision checks and fixing them increased the burden on the SMT solver.\footnote{\label{note:dreal}
	Two out of 217 queries in the Fetch example are particularly hard for the solver and could not be solved with 1 hour.
	We suspect a bug in the solver as all the other queries are solved in $134$s.
}

\newcommand{\centercell}[1]{\multicolumn{2}{c}{#1}}

\begin{table}[t]
\centering
\caption{Programs, specification, and verification time for the \PGCD examples. ``Prev'' refers to \cite{ECOOP19}.}
\vspace{-1ex}
\scriptsize
\begin{tabular}{lrrr|cc|cc|cc|cc} \toprule
Scenario        & \multicolumn{3}{c}{Program (Loc)}     & \centercell{Global Type}  & \centercell{Syntactic}    &  \centercell{Motion}         & \centercell{Typing (s.)}  \\
                &           &         &                 & \centercell{(LoC)}        & \centercell{Checks (s.)}  &  \centercell{Compatibility (s.)}     &                      \\ 
                &   Arm     & Cart 1  & Cart 2          & Prev & This work&  Prev & This work  & Prev &  This work   &  Prev & This work   \\ \midrule
Fetch           &  $14$     &  $19$   &  $-$            &    $22$ & $42$            &       $0.1$  & $0.1$      & $> 3600$ \footnoteref{note:dreal} &   $\bf 7$           &  $0.1$ & $0.1$       \\
Handover        &   $9$     &   $8$   &  $6$            &    $14$ & $26$            &       $0.1$  & $0.1$      &   $2224$  &   $\bf 70$           &  $0.1$ & $0.1$       \\
Twist and turn  &   $12$    &  $16$   &  $6$            &    $15$ & $17$            &       $0.1$  & $0.1$      &    $599$  &  $\bf 203$           &  $0.1$ & $0.1$       \\ 
Underpass       &   $18$    &   $3$   & $14$            &    $21$ & $65$            &       $0.1$  & $0.1$      &    $187$  &  $\bf 114$           &  $0.1$ & $0.1$       \\ \bottomrule
\end{tabular}
\label{tbl:xp-revisit}
\end{table}

\begin{figure}[t]
\centering
    \begin{subfigure}[b]{0.4\textwidth}
        \includegraphics[width=0.965\textwidth]{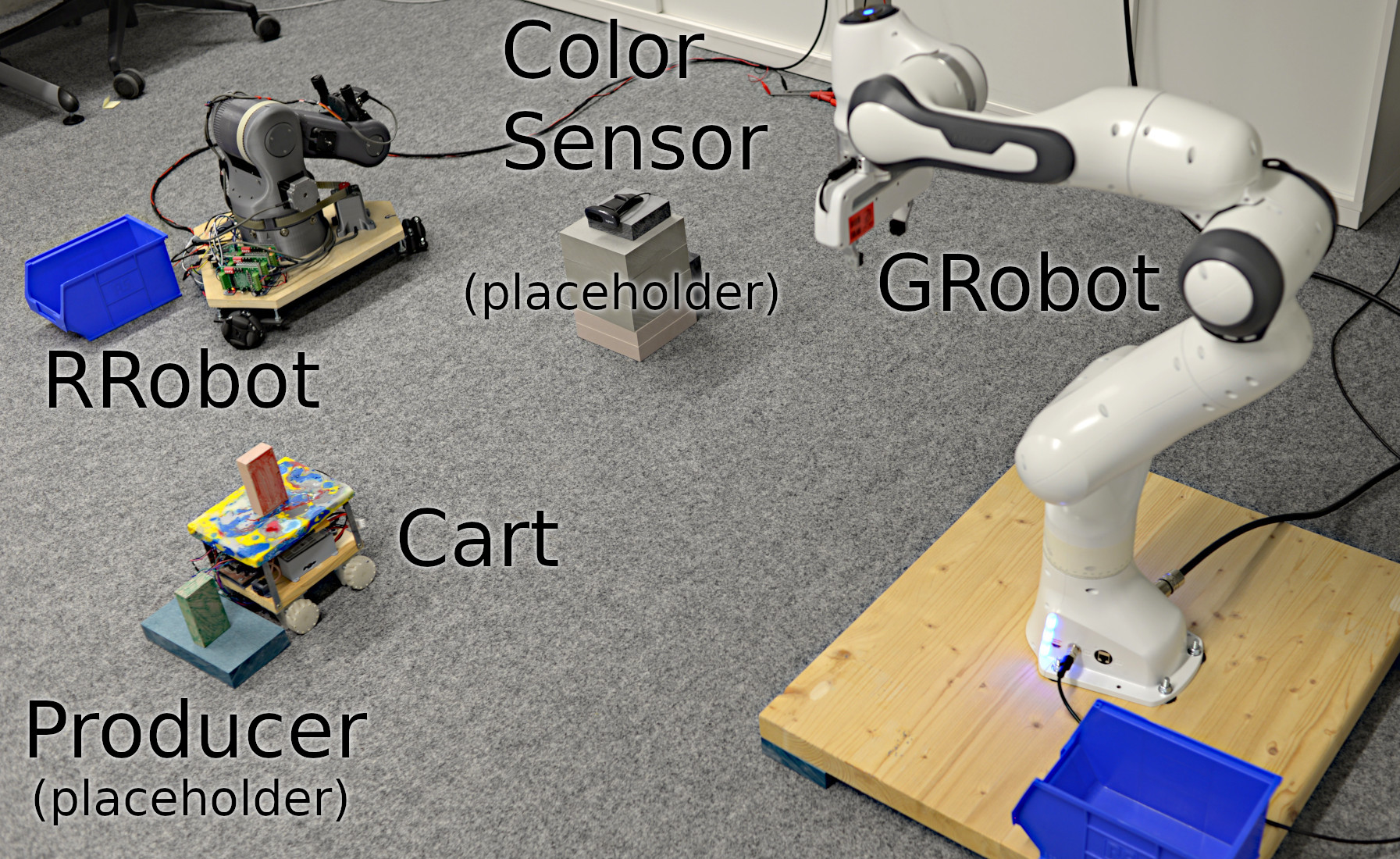}
        \caption{Setup}
        \label{fig:robots-xp-setup}
    \end{subfigure}
    \begin{subfigure}[b]{0.4\textwidth}
        \includegraphics[width=\textwidth]{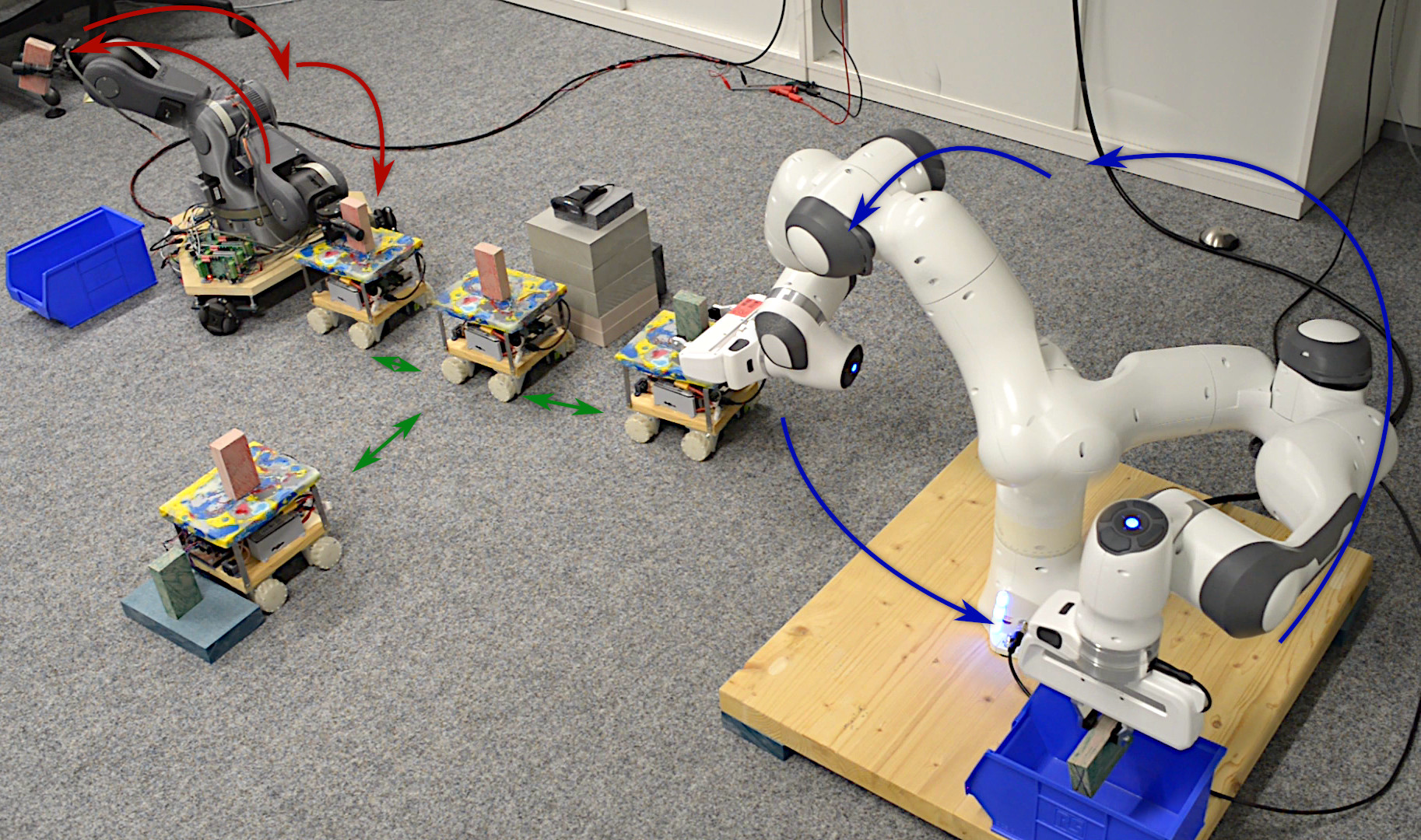}
        \caption{Sequence of motions}
        \label{fig:robots-xp}
    \end{subfigure}
\vspace{-1ex}
    \caption{Experimental setup for the sorting example}
\vspace{-1ex}
\end{figure}
\myparagraph{(3) A Complex Case Study}
The purpose of the case study is to show that we can semi-automatically verify systems beyond the scope of previous work. 
Table~\ref{tbl:xp1} shows the sizes of the processes (in lines of code, in the syntax of \PGCD)
and statistics related to the global specification and verification time. 
The global specification consists of the global type, the environment description, 
and (manually provided) annotations for the verification.
The annotations are mostly related to the footprints and the parallel composition.
Each time we use the parallel composition operator we specify a partition of 
the current footprint (the ``$\exists \FP_1,\FP_2$'' in rule \rulename{AGcomp}).

In Table~\ref{tbl:xp1}, we split the running times into the well-formedness checks which can be done syntactically,
the verification conditions for the execution of the motion primitives, and the typing.
We did not try to encode this example using the older \PGCD specifications for two reasons.
First, the stronger requirement on discrete time steps and global synchronisation in time in the previous calculus would subtantially change this example.
Remember that motions in \cite{ECOOP19} are specified simultaneously for all the robots and must have the same duration.
When multiple motion primitives have different durations, a motion step 
can only be as long as the shortest motion. 
Thus, we would need to chunk the longer motions into a sequence of smaller steps. 
This would, in turn, make the (sub)typing much more difficult as one motion step in the code would correspond to 
a sequence of motion steps in the specification.
Second, as the previous calculus does not include the separating conjunction, 
there is little hope that the SMT solver can cope with the compelxity of this example: it contains more robots and more complex ones.
Because our global types include separating conjunctions, we can more efficiently generate and discharge the 
verification conditions.
At the modest cost of specifying the footprints for each thread, the 
collision checks only need to consider the robots within a thread and not all the robots in the system.
This brings a substantial reduction in the complexity of the verification conditions.

In conclusion, the case study demonstrates the power of our method
in specifying and verifying non-trivial coordination tasks between multiple robots.
It requires the expressiveness of our motion primitive specifications
and the modularity of our separating conjunction.

\begin{table}[t]
\centering
\caption{Programs, motion specification, type, and verification time}
\vspace{-1ex}
\small
\begin{tabular}{lcc|lrl} \toprule
Robot     &   Motion Spec (LoC)   &   Program (LoC)   & Specification      \\ \midrule
Moveo arm &   306                 &   28              & Global type          &  96               & LoC          \\
Panda arm &   348                 &   31              & Syntactic   checks   &  0.2              & s.           \\
Carts     &   404                 &   45              & Compatibility checks &  2410             & s.           \\
Sensor/Producer    &   105        &   17 / 10             & Typing               &  3.1              & s.           \\ \bottomrule
\end{tabular}
\label{tbl:xp1}
\vspace{-1ex}
\end{table}

\section{Related Work}
\label{sec:related}

The field of concurrent robotics has made enormous progress---from robot soccer to
self-driving systems and to industrial manufacturing.
However, to the best of our knowledge, none of these impressive systems come with formal
guarantees of correctness.
Our motivation, like many other similar projects in the area of high-confidence robotics and cyber-physical systems design, is to 
be able to reason formally about such systems.

A spectrum of computer-assisted formal methods techniques can be, and have been, applied to reasoning about concurrent robotics,
ranging from interactive theorem proving to automated analysis via model checking.
As is well known, these techniques provide a tradeoff between manual effort and the expressiveness of
specifications and the strength of guarantees.
Our choice using manually specified choreographies and motion primitive specifications and automatically
checked type correctness attempts to explore a point in the design space that requires manual abstraction
of motion and geometry but provides automated checks for the interaction.
We build on a type-based foundation rather than global model checking to again reflect the tradeoff:
our use of choreographies and projections from global to local types restricts the structure of programs
that can be type checked but enables a more scalable \emph{local} check for each process; 
in contrast, a model checking approach could lead to state-space explosion already at the level of concurrency.

There are a vast number of extensions and applications of session
types and choreographies~\cite{Huttel:2016:FST:2911992.2873052,ABBCCDGGGHJMMMNNPVY2016,BETTY2017},
but little work bringing types to practical programming in the domain of robotics or cyber-physical systems.
We discuss the most related work. 
Our starting point was the theory of motion session types and PGCD \cite{PGCD,ECOOP19}, whose goals are
similar to ours.
We considerably extend the scope and expressiveness of motion session types:
through \emph{continuous-time} motion primitives and through separating conjunction for modular synchronisation.

Syntactic extensions  
of original global types \cite{HYC08} to represent more expressive
communication structures have been studied, 
e.g.,~in \cite{DBLP:conf/popl/LangeTY15}, 
in the context of synthesis from communicating
automata; in~\cite{DBLP:journals/corr/abs-1203-0780} to include 
parallel and choices; in \cite{DH2012} 
to represent nested global types.
Our aim is to include the minimum syntax extension to 
\cite{HYC08} to represent separation and synchronisations 
best suited to implementing robotics applications. 
In turn, a combination of motion primitives and predicates 
required us to develop a novel and non-trivial data flow analysis of 
global choreographies for the well-formedness check.

Hybrid extensions to process algebras \cite{RoundsS03,DBLP:journals/iandc/LynchSV03,BERGSTRA2005215,DBLP:conf/aplas/LiuLQZZZZ10}
extend process algebras with hybrid behaviour and
study classical concurrency issues such as process equivalences.
Extensions to timed (but not hybrid) specifications in the $\pi$-calculus are studied in  
\citet{BYY2014,BMVY2019} to  
express properties on times on top of \emph{binary} session typed processes. 

The theory of hybrid automata \cite{ACHH,Henzinger96} provides a foundation for verifying
models of hybrid systems.
The main emphasis in hybrid automaton research has been in defining semantics and designing
model checking algorithms---too many to enumerate, see \cite{Henzinger96,Platzer18}---and 
not on programmability. 
Assume guarantee reasoning for hybrid systems
has been studied, e.g.~\cite{Nuzzo:EECS-2015-189,DBLP:journals/fteda/BenvenisteCNPRR18,DBLP:journals/pieee/NuzzoSBGV15,DBLP:journals/pieee/Tripakis16}.
These works do not consider programmability or choreography aspects.

Deductive verification for hybrid systems attempts to define logics and invariant-based reasoning to hybrid systems.
Differential dynamic logic (dL) \cite{Platzer18,PlatzerT18} is a general logical framework to deductively reason about hybrid systems.
It extends \emph{dynamic logic} with differential operators and shows sound and (relatively) complete axiomatisations for the logic.
Keymaera \cite{DBLP:conf/cade/FultonMQVP15} and HHL Prover \cite{HHL} are interactive theorem provers based on hybrid progam logics. 
These tools can verify complex properties of systems at the cost of intensive manual effort.
In contrast, we explore a point in the design space with more automation but less expressiveness.
The ``trusted base'' in an interactive theorem prover is the core logic; optimizing the trusted base was not one of our goals.

A well-studied workflow in high-confidence cyber-physical systems is model-based design (see, e.g., \cite{HenzingerSifakis07} for an overview), 
where a system is constructed by successive refinement of an abstract model down to an implementation.
An important problem in model-based design is to ensure property-preserving refinement: one verifies properties of the system
at higher levels of abstraction, and ensures that properties are preserved through refinement.
In the presence of continuous dynamics, defining an appropriate notion of refinement and proving property-preserving
compilation formally are hard problems \cite{DBLP:journals/tosem/YanJWWZ20,DBLP:conf/pldi/BohrerTMMP18}.

In our implementation, we use \emph{automated} tools based on the dReal SMT solver \cite{dReal}.
Our verification is \emph{semi-automatic}, as we require user annotations for footprints or for
the motion primitive specifications, but discharge verification conditions through dReal.
Of course, there are programs that go past the capability of the solver---this
is already true for systems that \emph{only} have concurrency or only deal with dynamics. 
Our proof rules is to allow sound reasoning, potentially inside an interactive prover. 
Our implementation shows that---at least some---nontrivial examples do allow automation.
Non-linear arithmetic is difficult to scale.
Our observation is that a combination of manual specification of \emph{abstract} footprints
that replace complex geometrical shapes with simpler over-approximations along with the power of
state-of-the-art SMT solvers is reasonably effective even for complicated examples. 

Our choice of dReal (as opposed to a different SMT solver) is dReal's ``off the shelf'' support
for non-linear arithmetic and trigonometric functions. 
Trigonometry shows up in our handling of frame shifts. 
Note that dReal only considers $\delta$-decidability and can be incomplete in theory.
We expect that most implementations will be collision free in a ``robust'' way 
(that is, two robots will not just not collide, they would be separated by a minimum distance).
Therefore we believe the potential incompleteness is less of a concern.
Indeed, the main limiting factor in our experiments was scalability for larger verification conditions.

Formal methods have been applied to multi-robot planning \cite{Drona,Antlab}.
These systems ignore programming or geometry concerns and view robots as sequences of motion primitives.

\section{Discussion and Future Directions}
\label{sec:discussion}

In this paper, we have shown how choreographies can be extended with dynamic motion primitives to enable compositional
reasoning in the presence of continuous-time dynamics. 
We have developed an automated verification tool and a compiler from type-correct programs to distributed
robotics applications using ROS and commercial and custom-made robotics hardware.
Our goal is to \emph{integrate} types and static analysis techniques into existing robotics frameworks,
rather than provide fully verified stacks (see, e.g., \cite{DBLP:conf/pldi/BohrerTMMP18}). 
We explain our language features in terms of a calculus but session processes 
can be easily embedded into existing frameworks for robot programming. %
We have demonstrated that the language and type system are expressive enough to statically verify 
distributed manoeuvres on top of existing hardware and software. 

We view our paper as a first step in verifying robotics programs.
There are many other important but yet unmodelled aspects.
We outline several directions not addressed in our work.

For example, we omit \emph{probabilistic robotics} aspects, including the perception stack (vision, LIDAR, etc.),
and aspects such as filtering, localisation, and mapping
\cite{ProbabilisticRobotics,LaValle2012}.
These will require a probabilistic extension to our theory. 
Such a theory requires a nontrivial extension of program logics and analyses for probabilistic programs \cite{McIverMorgan} with
the verification and synthesis for stochastic continuous-state systems \cite{ZamaniEMAL14}.
We also omit any modeling of the perception stack or dynamic techniques, often based on machine learning, of learning the environment.
Instead, our models \emph{assume} worst-case disturbance bounds on the sensing or dynamics.
We believe an integration of learning techniques with formal methods is an interesting challenge
but goes beyond the scope of this paper.

Our framework statically verifies properties of a program.
In practice, robots work in dynamic, often unknown, environments \cite{LaValle,LaValle2012,Siegwart,ProbabilisticRobotics}.
When confronted with a formal method, domain experts often expect that
a verification methodology should be able to verify correctness of behaviors in an \emph{arbitrary} dynamic environment
and any failure to do so simply shows the inadequacy of verification techniques.
Formal methods cannot prove correctness in an \emph{arbitrary} dynamic environment.
When we verify a system, it is---as true in any formal methods---\emph{relative} to an environment assumption;
such assumptions are usually implicit in robotics implementations.
Thus, we can model moving obstacles, etc.\ in the environment through assumptions on the behaviour
of such obstacles (e.g., limits on their speed or trajectories); these assumptions are propagated
by our assume-guarantee proof system, and show up as a premise in the eventual correctness proof.

We focus on communication safety and collision freedom as the basic correctness conditions any
system has to satisfy.
An interesting next step is to extend the reasoning to more expressive specifications.
For safety specifications, such as invariants, one could reduce the problem to checking
communication safety.
For liveness specifications, the proof system would need to be extended with ranking arguments.

While all the above problems are interesting in their own right, they are orthogonal to our main contribution that 
one can reason about concurrency and dynamics in continuous time in a type-based setting. 
Our future work will look at more expressive scenarios, but the setting in our paper already required complex proofs and it 
was important for us to get the core correct.
We believe a verification system that can faithfully model and uniformly reason about more complex interactions and that scales
to larger implementations remains a grand challenge in computer science (see, e.g., \cite{Lozano-Perez} for an articulation of these challenges).

\begin{acks}                            %
We thank the OOPSLA reviewers for useful comments and suggestions; and
Julia Gabet for testing the artifact.
Majumdar and Zufferey are supported in part by 
the Deutsche Forschungsgemeinschaft project 389792660 TRR 248--CPEC
and by the European Research Council under the
Grant Agreement 610150 (http://www.impact-erc.eu/) (ERC Synergy Grant ImPACT).
Yoshida is supported in part by
  EPSRC EP/T006544/1, EP/K011715/1, EP/K034413/1, EP/L00058X/1, EP/N027833/1,
  EP/N028201/1, EP/T006544/1, EP/T014709/1 and EP/V000462/1, and NCSS/EPSRC VeTSS. 
\end{acks}

\bibliography{biblio,jlamp}

\iftoggle{full}{
\appendix 
\section{Appendix}
\label{sec:appendix}

\subsection{Proofs for Subject Reduction Theorem and Progress Theorem}
\label{app:sr}
We first list several basic lemmas. In the proofs, we often omit $I$ or $\mathcal{P}$ unless required. 

\begin{lemma}[Inversion lemma]
  \label{lem:Inv_S}
 \begin{enumerate}[label=(\arabic*)]
 \item\label{lem:Inv_S1}%
   Let $\der{\Gamma}{\PP}{\T}$.
   \begin{enumerate}[label=(\alph*)]
   \item\label{lem:Inv_S1a} 
     If $\PP = \sum\limits_{i\in I}\procin  {\pp}   {\ell_i(\x)} {\Q_i}$, then
     $ \&_{i\in I}\procin {\pp} {\ell_i{(\S_i)}}{\T_i}\subt\T $ and $ \der{\Gamma, x:\S_i}{ \Q_i}{\T_i} ,$ for every $i\in I.$
   \item\label{lem:Inv_S1b} If $\PP = \procout  {\pp}{\ell}{\e}{\Q}$, then
     $ \der{\Gamma}{\e}{\S}$  and  $\der{\Gamma}{ \Q}{\T'} $ and $\S'\subs \S$ and $\procout{\pp}{\ell}{\S'}{\T'}\subt\T.$
   \item\label{lem:Inv_S1c}
     If $\PP = \cond{\e}{\PP_1}{\PP_2} $, then
     $\der{\Gamma}{\PP_1}{\T_1}$ and $\der{\Gamma}{\PP_2}{\T_2}$  and    $\T_1\subt \T$ and $\T_2\subt\T.$ 
   \item\label{lem:Inv_S1d}  If $\PP = \mu \procvar.\Q $, then $ \der{\Gamma, \procvar:\T}{\Q}{\T} $.
   \item\label{lem:Inv_S1e}
     If $\PP = \procvar$, then $\Gamma = \Gamma', \procvar:\T'$ and $\T'\subt\T$.

\item\label{lem:Inv_M1}%
If $P=\motion{dt}{a}.{Q}$, then 
$\der{\Gamma}{Q}{\T'}$ and 
$\T'\subt\T$.

\item\label{lem:Inv_M2}%
If $P=\sum\limits_{i\in I}{\procin  \q   {\ell_i(x_i)}
  \PP_i}+\motion{dt}{a}.{Q}$, then 
$\&_{i\in I}\procin {\pp} {\ell_i{(\S_i)}}{\T_i} \& \T'\subt\T $ 
and $ \der{\Gamma, x:\S_i}{ \PP_i}{\T_i} ,$ for every $i\in I$
and $\der{\Gamma}{\motion{dt}{a}.{Q}}{\T'}$.

   \end{enumerate}
 \item\label{lem:Inv_S2}
   If $\der{} {\prod\limits_{i\in I}\pa{\pp_i}{\PP_i}}\GL$, then
   $\participant \GL\subseteq\set{\pp_i\mid i \in I}$ and
   $\der{}{\PP_i}{\projt \GL {{\pp_i}}}$  for every $i\in I$.

 \end{enumerate}
\end{lemma}
\begin{proof}
By induction on type derivations.
\end{proof}

\begin{lemma}[Substitution lemma]\label{lem:Subst_S}
\label{lem:subs_X}
\begin{enumerate}
\item If $\der{\Gamma, x:\S}{\PP}{\T} $ and $\der{\Gamma}{\val}{\S}$, then 
$\der{\Gamma}{\PP \sub{\val}{\x}}{\T}$.
\item 
If 
$\der{\Gamma, \procvar:\T}{\PP}{\T} $ and $\der{\Gamma}{\Q}{\T}$, then 
$\der{\Gamma}{\PP \sub{\Q}{\procvar}}{\T}$.
\end{enumerate}
\end{lemma}
\begin{proof}
By structural induction on $\PP$.
\end{proof}

\begin{lemma}{Subject Congruence}\label{lem:cong}$ \ $
	\begin{enumerate}[label=(\arabic*)]
	\item \label{lem:congP}
		Let $ \der{\Gamma}{\PP}{\T} $ and $ \PP\equiv\Q$. 
		Then $ \der{\Gamma}{\Q}{\T} $.
	\item \label{lem:congN}
		Let $ \der{}{M}{\GL}$ and $M\equiv M'$. 
		Then $ \der{}{M}{\GL}$.
	\end{enumerate}
\end{lemma}
\begin{proof}
By case analysis on the derivation of $ \PP\equiv \Q $ and $ M\equiv
M'$
\begin{enumerate}
\item The only interesting case is \rulename{t-rec}.
		So, in this case, $ \PP = \mu \procvar.\PP' $ and $ \Q = \PP'\sub {\mu \procvar.\PP'}{\procvar}.$
		By Lemma~\ref{lem:Inv_S}.\ref{lem:Inv_S1}.\ref{lem:Inv_S1d} we get $ \der{\Gamma,\procvar:\T}{\PP'}{\T}. $
		From that, $ \der{\Gamma}{\PP}{\T} $ 
and Lemma~\ref{lem:subs_X}, we conclude.
\item Straightforward.  
\end{enumerate}
\end{proof}

\begin{lemma}
  \label{lem:erase}
  If\; $\tout\q\ell{\S}.{\T}\leq\projt\GL\pp$ %
  \;and\;
  $\tin\pp\ell{\S'}.{\T'} \leq \projt\GL\q$, %
then 
  \begin{enumerate}[label=(\arabic*)]
  \item\label{item:erase:subp}%
    $\T \subt \projt{(\redG\GL\pp\ell\pq)}\pp$
  \item\label{item:erase:subq}%
    $\T' \subt \projt{(\redG\GL\pp\ell\pq)}\pq$; and 
  \item\label{item:erase:g-consume-proj-r}%
    $\projt\GL\pr \subt \projt{(\redG\GL\pp\ell\q)}\pr$ %
    \;for\; $\pr\not=\pp$, $\pr\not=\q$.%
  \end{enumerate}
\end{lemma}
\begin{proof}
By induction on $\GL$.  
\end{proof}

\theoremSR*

\begin{proof}
The second close is a corollary from the first close and
Lemma~\ref{lem:cong}.

The proof is by induction on the derivation of $\N \red \N'.$

\begin{itemize}
\item[] $\rulename{comm}$:

In this case, we can let:
$\N=\pa\pp{\sum\limits_{i\in I} \procin{\q}{\ell_i(\x)}{\PP_i}}\; \ | \ 
\;\pa\q{\procout \pp {\ell_j}{\e} \Q}$ 
and 
$\N' = \pa\pp{\PP_j}\sub{\val}{\x}\; | \;\pa\q\Q$ with
$\eval{\e}{\val}$. 

From $\der{}\N\GL,$ by Lemma \ref{lem:Inv_S}\ref{lem:Inv_S2},

 \begin{align}
    &  \participant\GL\subseteq \{\pp,\pq\} \cup \{\pp_l: l\in L\} \\
    &  \der{}{\sum\limits_{i\in I} \procin{\q}{\ell_i(\x)}{\PP_i}}{\projt{\GL}{\pp}} \label{eqv:proof28}\\
    & \der{}{\procout{\pp}{\ell_j}{\e}{\Q}}{\projt{\GL}{\pq}} \label{eqv:proof29} \\%
    &  \der{}{\PP_l}{\projt{\GL}{\pp_l}}  \text{ for every }l \in L \label{eqv:proof30}
\end{align}
By Lemma \ref{lem:Inv_S}\ref{lem:Inv_S1}\ref{lem:Inv_S1a}, from \eqref{eqv:proof28},
\begin{align}
    &   \bigwedge\limits_{i\in I}\tin \pq {\ell_i}{\S_i}.{\T_i}  \subt \projt{\GL}{\pp} \label{eqv:proof31}\\
    &   \der{\x:\S_i}{\PP_i}{\T_i} \text{ for every }i\in I   \label{eqv:proof32}%
\end{align}
By Lemma \ref{lem:Inv_S}\ref{lem:Inv_S1}\ref{lem:Inv_S1b}, from \eqref{eqv:proof29},
\begin{align}
    &   \procout {\pp}{\ell_{j}}{\S_j'}{\T_2'''}  \subt \projt{\GL}{\pq} \label{eqv:proof33}\\
    &   \der{}{\Q}{\T_2'''} \label{eqv:proof34} \\
    & \der{}{e}{\S_j} \label{eqv:proof35} \\
    & \S_j' \subs \S_j
  \end{align}
  By Lemma~\ref{lem:Subst_S}, from \eqref{eqv:proof32}, \eqref{eqv:proof35} and $\eval{\e}{\val},$
  \begin{align}
    \der{}{\PP_j\sub{v}{\x}}{\T_j} \label{eqv:proof37}
  \end{align}
  By Lemma~\ref{lem:erase}, from \eqref{eqv:proof31} and \eqref{eqv:proof33}
  \begin{align}
    \T_2''' \subt \projt{(\redG\GL\q{{\ell_j}}\pp)}{\pq}{}  \quad
    \T_{j}  \subt \projt{(\redG\GL\pp{{\ell_j}}\q)}{\pp}{} \quad
    \projt{\GL}{\pp_l} \subt \projt{(\redG\GL\q{{\ell_j}}\pp)}{\pp_l}{}, \;l\in L  \label{eqv:proof38}
 \end{align}
 By \rulename{t-sub} and \rulename{t-sess}, from \eqref{eqv:proof30}, \eqref{eqv:proof34}, \eqref{eqv:proof37} and \eqref{eqv:proof38},
 \begin{align*}
      \der{}{\N'}{\redG\GL\q{{\ell_j}}\pp}.
  \end{align*}

\item[] \rulename{default}: 

$\N = \pa\pp{\sum\limits_{i\in I} \procin{\q}{\ell_i(\x)}{\PP_i} +
  \motion{\dtsess}{a}{\PP}}$ and  $\N'=
  \pp\plays \proc{\motion{\dtsess}{\ma\AT t}.\PP}{\sstore{X}'}{\sstore{W}'}$ with
  \begin{enumerate}
\item\label{thm:Def_E1}%
$\pp\plays \proc{\motion{\dtsess}{\ma}.\PP}{\sstore{X}}{\sstore{W}} \xrightarrow{\tau} 
\pp\plays \proc{\motion{\dtsess}{\ma\AT 0}.\PP}{\sstore{X}}{\sstore{W}}$ and 
\item\label{thm:Def_E2}%
$\pp\plays \proc{\motion{\dtsess}{\ma\AT 0}.\PP}{\sstore{X}}{\sstore{W}} \xrightarrow{\motion{\dtsess}{\ma},\xi,\nu, t} 
\pp\plays \proc{\motion{\dtsess}{\ma\AT t}.\PP}{\sstore{X}'}{\sstore{W}'}$ 
\end{enumerate}
  From Item (\ref{thm:Def_E1}), by (IH) of rule \rulename{traj-base}, we have
  $\der{}{\N_0}{\GL_0}$
  such that $\N_0=\pp\plays \proc{\motion{\dtsess}{\ma\AT 0}.\PP}{\sstore{X}}{\sstore{W}}$ and ${\GL}\RedG{\GL_0}$.
  Similarly
  from Item (\ref{thm:Def_E2}), by (IH) of rule \rulename{traj-step},
  we have $\der{}{\N'}{\GL'}$ such that ${\GL_0}\RedG{\GL'}$, 
  hence ${\GL}\RedG{\GL'}$. 
  
\item[] \rulename{traj-base}:  By Definition~\ref{def:global-type-consumption}(\ref{Item:GC41},\ref{Item:GC6}).

\item[] \rulename{traj-step}:  By Definition~\ref{def:global-type-consumption}(\ref{Item:GC4},\ref{Item:GC6}).   

\item[] \rulename{non-interrupt}: By Definition~\ref{def:global-type-consumption}(\ref{Item:GC7}.\ref{Item:GC7A},\ref{Item:GC6}).   

\item[] \rulename{interrupt}: By Definition~\ref{def:global-type-consumption}(\ref{Item:GC7}.\ref{Item:GC7B},\ref{Item:GC6}).     

\item[]  $\rulename{t-conditional}$: 
     $\N=    \pa\pp{\cond{\e}{\PP}{\Q}} \; | \;  \prod\limits_{l\in L} \pa{\pp_l}{\PP_l},\;
       \N'=  \pa\pp\PP\; | \;   \prod\limits_{l\in L} \pa{\pp_l}{\PP_l},\; \eval{e}{\true}.$
\\
From $\der{}\N\GL,$ by Lemma \ref{lem:Inv_S}\ref{lem:Inv_S2},
\begin{align}
  & \der{}{{\cond{\e}{\PP}{\Q}}}{\projt{\GL}{\pp}}  \label{eqv:proof39}  \\
  & \der{}{\PP_l}{\projt{\GL}{\pp_l}},\;l\in L \label{eqv:proof40}
\end{align}
By Lemma \ref{lem:Inv_S}\ref{lem:Inv_S1}\ref{lem:Inv_S1c}, from \eqref{eqv:proof39},
\begin{align}
  & \der{}{\PP}{\T_1}  \qquad \der{}{\Q}{\T_2} \label{eqv:proof41} \\
  & \T_1\subt \projt{\GL}{\pp} \qquad  \T_2\subt \projt{\GL}{\pp} \label{eqv:proof42}
\end{align}
By  \rulename{t-sub} and \rulename{t-sess}, from \eqref{eqv:proof40},
\eqref{eqv:proof41} and  \eqref{eqv:proof42}, we derive
$\der{}{\N'}{\GL}.$

\item[]  $\rulename{f-conditional}$: Similar with
  $\rulename{t-conditional}$.

\item[] \rulename{m-par} 

By Definition~\ref{def:global-type-consumption}(\ref{Item:GC7}.\ref{Item:GC7B},\ref{Item:GC6}) and by (IH) from \rulename{traj-step}. 

\item[] \rulename{r-par} By (IH) of $\N_1 \xrightarrow{\alpha}\N_2$.

\item[] \rulename{r-struct}

  Assume that the last applied rule was \rulename{r-struct}.
  Then, $\N_1' \xrightarrow{\alpha} \N_2'$ was derived from $\N_1 \xrightarrow{\alpha} \N_2$  where $\N_1 \equiv \N_1'$ and $\N_2 \equiv    
  \N_2'.$ By Lemma~\ref{lem:cong}\ref{lem:congN}, from the assumption $\der{}{\N_1'}{\GL},$ we deduce $\der{}{\N_1}{\GL}.$
  By induction hypothesis,  $\der{}{\N_2}{\GL'} \text{ for some }\GL' \text{ such that }\GL\RedG\GL'.$
  By Lemma \ref{lem:cong}\ref{lem:congN}, $\der{}{\N'_2}{\GL'} \text{ for some }\GL' \text{ such that }\GL\RedG\GL'.$
\end{itemize}
\end{proof}

\theoremPROG*

\begin{proof} 
{\bf\em (Communication Progress)} The proof is divided into two cases:
  {\bf (a)} the guard appearing in the branching type followed by
  {\bf (b)} a message exchange between two parties; or 
reduces by conditional. 

{\bf Case (a)} follows from the fact that there is always at least one guard
  in a choice whose predicate is evaluated to true (Definition
  \ref{def:well-formed}(1)). 
The typing rules ensure that the predicates are satisfied when a
  message is sent.

{\bf Case (b)} The case when $\N$ contains the conditional 
is trivial since always $\N \red \N'$. 
There are three main cases:
\begin{itemize}
\item[]
$\GL=\GvtPair \pp\pq{\ell_i(S_i).\GL_i}$. 
Then $\N \equiv  \pa\pp\PP \;| \pa\q\Q \;|\;\N_1$ and
$\der{}\PP{\oplus_{i\in I}\tout \pq {\ell_i}{\S_i}.{(\projt{\GL_i}{\pp})} }$ 
 and 
$\der{}\Q{\&_{i\in I} \tin \pp
  {\ell_i}{\S_i}.{(\projt{\GL_i}{\pq})}}.$ 
Then we have 
     $\PP \equiv  \procout \pq {\ell_{j}}{\e} \PP_{j}^{\pp},$ $j\in I,$  $\eval\e v,$ and  $\Q\equiv \sum\limits_{i\in I'} \procin\pp{\ell_i(\x_i)}{\PP_i^\pq} \text{ and }I\subseteq I':$ 
      $\N \red \pa\pp \ \PP_{j}^{\pp} \;| \; \pa\q \PP_j^\pq \sub{v}{\x} \;| \; \N_1,$ by  \rulename{comm}.

\item[] Suppose $\GL= g.\GL'$.
  \begin{itemize}
  \item If $g = \GvtPre{\pp}{\q}{[\pred_i]\ell_i}{\set{\nu:\SL\mid\pred'_i}}.g'$, 
 it is proved as the first case above.
 \item The case $g$ is a motion primitive is proved in
 ({\bf\em Motion Progress}) below.
 \item The case $g=g_1 + g_2$  is reduced to the first case since
 all branching types can be written in the form of
 $\GvtPair \pp\pq{\ell_i(S_i).\GL_i}$ (see the paragraph below
 Definition~\ref{def:global-types}).
\item Suppose $\GL= (g_1 \ast g_2).\GL'$.

  Then by Definition~\ref{def:projection} and the fully separated condition of the well-formedness, if $\pp \in g_1$, then
  $\pp \not\in g_1$. Suppose
  $g_1 = \GvtPre{\pp}{\q}{[\pred_i]\ell_i}{\set{\nu:\SL\mid\pred'_i}}.g_i'$. 
  Then
  $\N \equiv  \pa\pp\PP \;| \pa\q\Q \;|\;\N_1$ and
$\der{}\PP{\oplus_{i\in I}\tout \pq {\ell_i}{\S_i}.{(\projt{\GL_i}{\pp})} }$ 
 and 
 $\der{}\Q{\&_{i\in I} \tin \pp {\ell_i}{\S_i}.{(\projt{\GL_i}{\pq})}}.$
 Thus this case is the same as $\GL=\GvtPair \pp\pq{\ell_i(S_i).\GL_i}$ above.

 When $g_i$ is a motion primitive is proved in ({\bf\em Motion Progress}) below.
  \end{itemize}
\item The case $\GL= \mu \ty.\GL'$. Since 
  $\GL=\GL'\sub{\mu \ty.\GL'}{\ty}$, the case reduces to one of the above. 
   \end{itemize}
{\bf\em (Motion Progress)} 
For executing motions, Definition~\ref{def:well-formed}(3) and Theorem~\ref{th:motion-compat}
ensure local trajectories can be composed into global trajectories.
Suppose $\GL=g.\GL'$.
\begin{itemize}
\item Suppose $g = \motion{\dtsess}{(\pp_i{:}\ma_i)}$. 
  In this case, we have:
  $\N=\prod_{i \in I} \pp_i\plays\proc{P_i}{\sstore{X}_i}{\sstore{W}_i}$.  
  and we can write $P_i=\motion{\dtsess}{\ma_i\AT t_i}.P_i'$.   
  By Definition~\ref{def:well-formed}(3) and Theorem~\ref{th:motion-compat},
  we know $i\neq j \Rightarrow \mathsf{disjoint}(\ma_i,\ma_j)$. Hence
  by \rulename{M-Par}, we obtain $\N \xrightarrow{\motion{\dtsess}{\ma_i}, t} \N'$, as required. 
\item The case
  $g = g_1\ast g_2$. If one of them is a branching, it is reduced to
  {\bf\em (Communication Progress)} proved above. Hence w.o.l.g., we can
  assume $g = g_1\ast g_2 \ast\cdots \ast g_n$ and 
  $g_i = \motion{\dtsess}{\pp_i:\ma_i}$. By the fully separation condition
  and Definition~\ref{def:well-formed}(3), 
  this case is similar with the case above. 
\end{itemize}  
{\bf\em (Collision-free Progress)} 
By induction, following the same structure as {\bf\em (Motion Progress)}.
\begin{itemize}
\item The base case is given by $I$ and by \rulename{t-sess} which guarantees that $\N$ is initially collision free.
\item For the communication ($\GvtPair \pp\pq{\ell_i(S_i).\GL_i}$), note that it does not change the physical state, and therefore, does not impact the geometric footprint used by a process.
      Hence, the case $\redG \GL\pp\ell\q$ is straightforward.
\item For motion actions ($\motion{\dtsess}{(\pp_i{:}\ma_i)}$), Definition~\ref{def:well-formed}(3) and Theorem~\ref{th:motion-compat} ensures collision freedom through execution steps.
      Therefore, after the corresponding reduction $\N \xrightarrow{\motion{\dtsess}{\ma_i}, t} \N'$, $\N'$ is also collision free.
\item For the separation ($g_1\ast g_2$), Definition~\ref{def:well-formed}(2) enforces that the two parts are fully separated and the induction hypothesis means each subpart is collision-free.
      Therefore, the overall system is collision-free.
\item The remaining cases are bookkeeping rules, e.g., inserting $@t$ or removing fully consumed motions, similar to the communication these rules do not change the physical state, and therefore, preserve collision freedom.
\end{itemize}
\end{proof}

 }{}

\end{document}